\documentclass[twoside]{article}

\usepackage{color-edits}
\usepackage{booktabs}
\usepackage{natbib}
\usepackage[colorlinks=true, citecolor=blue,linkcolor=green!60!black]{hyperref}
\usepackage[margin=1in]{geometry}

\usepackage{amsmath}
\usepackage{amssymb}
\usepackage{mathtools}
\usepackage{amsthm}
\usepackage{nicefrac}

\usepackage{thmtools,thm-restate}
\newtheorem{theorem}{Theorem}[section]
\newtheorem{corollary}[theorem]{Corollary}
\newtheorem{lemma}[theorem]{Lemma}
\newtheorem{proposition}[theorem]{Proposition}
\newtheorem{definition}{Definition}

\newtheorem{example}{Example}
\usepackage{amsfonts}
\usepackage{dsfont}
\usepackage{graphicx}
\usepackage[dvipsnames]{xcolor}
\usepackage{subcaption}
\usepackage[algo2e, ruled, noend, linesnumbered]{algorithm2e} 

\usepackage{etoolbox}
\makeatletter
\patchcmd{\@algocf@start}
  {-1.5em}
  {0pt}
  {}{}
\makeatother
\makeatletter
\def\thmt@innercounters{equation,algocf}
\makeatother

\usepackage{bbm}
\newcommand{\E}{\mathbb{E}}

\newcommand{\1}{\mathbbm{1}}
\newcommand{\hu}{\widehat{u}}
\newcommand{\ist}{i^\star}

\newcommand\numberthis{\addtocounter{equation}{1}\tag{\theequation}}
\newcommand{\calA}{\mathcal{A}}
\newcommand{\calF}{\mathcal{F}}
\newcommand{\calI}{\mathcal{I}}
\newcommand{\calL}{\mathcal{L}}
\newcommand{\ALG}{\textsc{Alg}}

\newcommand{\BB}{\text{BB}}
\newcommand{\BBDivide}{\text{\upshape BB}_{\text{\upshape Divide}}}
\newcommand{\BBDivideAdjusted}{\text{\upshape BB}_{\text{\upshape DA}}}
\newcommand{\BBPull}{\text{\upshape BB}_{\text{\upshape Pull}}}

\newcommand{\fst}{f^\star}

\newcommand{\AAE}{\textsc{AAE}}
\newcommand{\UCB}{\textsc{UCB}}
\newcommand{\Bern}{\texttt{Bern}}
\newcommand{\calH}{\mathcal{H}}

\newcommand{\xhdr}[1]{\vspace{1mm} \noindent{\bf #1}}

\usepackage[capitalize,noabbrev]{cleveref}
\usepackage{parskip}
\usepackage{authblk}
\usepackage{comment}
\newcommand{\calO}{\mathcal{O}}
\newcommand{\Phase}{\texttt{Phase}}
\newcommand{\Var}{\texttt{Var}}

\newcommand{\term}[1]{\ensuremath{\mathtt{#1}}\xspace}
\newcommand{\FOC}{\term{FOC}}
\newcommand{\APC}{\term{APC}}
\newcommand{\tcalI}{\tilde{\calI}}

\newcommand{\bell}{\bar{\ell}}

\DeclarePairedDelimiter\ceil{\lceil}{\rceil}
\DeclarePairedDelimiter\floor{\lfloor}{\rfloor}

\title{\Large Can Probabilistic Feedback Drive User Impacts in Online Platforms?\footnote{
Authors listed in alphabetical order.}}

\author[1]{Jessica Dai}

\author[2]{Bailey Flanigan}
\author[1]{Nika Haghtalab}
\author[1]{Meena Jagadeesan}

\author[3]{Chara Podimata}

\affil[1]{\textit{University of California, Berkeley }}
\affil[2]{\textit{Carnegie Mellon University}}
\affil[3]{\textit{MIT \& Archimedes / Athena RC}}

\begin{document}

\maketitle

\thispagestyle{empty}
\begin{abstract}
A common explanation for negative user impacts of content recommender systems is misalignment between the platform's objective and user welfare. In this work, we show that misalignment in the platform's objective is not the only potential cause of unintended impacts on users: even when the platform's objective is fully aligned with user welfare, the platform's learning algorithm can induce negative downstream impacts on users.  The source of these user impacts is that different pieces of content may generate observable user reactions (feedback information) at different rates; these feedback rates may correlate with content properties, such as controversiality or demographic similarity of the creator, that affect the user experience. Since differences in feedback rates can impact how often the learning algorithm engages with different content, the learning algorithm may inadvertently promote content with certain such properties.  Using the multi-armed bandit framework with probabilistic feedback, we examine the relationship between feedback rates and a learning algorithm's engagement with individual arms for different no-regret algorithms. We prove that no-regret algorithms can exhibit a wide range of dependencies: if the feedback rate of an arm increases, some no-regret algorithms engage with the arm more, some no-regret algorithms engage with the arm less, and other no-regret algorithms engage with the arm approximately the same number of times. From a platform design perspective, our results highlight the importance of looking beyond regret when measuring an algorithm's performance, and assessing the nature of a learning algorithm's engagement with different types of content as well as their resulting downstream impacts.
\end{abstract}
\thispagestyle{empty}

\setcounter{page}{1}

\section{Introduction}\label{sec:intro}

Recommendation platforms---which facilitate our consumption of news, music, social media, and many other forms of digital content---can harm users in unintended ways, as documented by researchers~\citep{allcott2020welfare}, journalists~\citep{wsj2021}, and regulators~\citep{eureg2022}.
One prevailing explanation for these impacts 
has been \textit{misalignment} between the platform's objective (e.g., platform profit or user engagement)
and user welfare \citep{SVNAH21}.  
This raises the question: \textit{Is aligning the 
platform's objective with user utility sufficient to avoid negative impacts on users?} 

 In this work, we show that even if the platform's objective \emph{perfectly optimizes user utility}, the process by which the platform continually \emph{learns} user preferences can induce unintended impacts on users. In this learning process, the platform's learning algorithm relies on observing users' reactions to content, such as whether a user clicked on a piece of content, pressed the like button, or retweeted it. Whether users react to content in observable ways can depend on the specifics of the content---e.g., the content could be controversial, provoking users to comment, or broadly relatable, prompting users to share it---in ways which are not captured by user utilities.\footnote{We expect that user feedback rates are not intrinsically captured by user utility: for example, high-utility content may either induce a high feedback rate (e.g., if a user retweets controversial content that they agree with) or induce a low feedback rate (e.g., if the content is educational and does not provoke a response). Similarly, low-utility content may either incite response (e.g., if the user disagrees with controversial content) or be ignored (leading to low feedback rates).} As a result, the approach by which the learning algorithm accounts for these differential rates of information gain can affect how often content with such properties (e.g., controversiality) is recommended. Unfortunately, the resulting impact on recommendations may inadvertently affect the overall user experience on the platform, as we describe in Examples \ref{ex:owngroup-apc} and \ref{ex:incendiary-foc}.

We study the impact of the platform's learning algorithm within the \emph{multi-armed bandits} framework with \emph{probabilistic feedback}. In this model, each piece of content corresponds to an arm with a \textit{loss}, which quantifies 
a fixed user's utility for the corresponding content and can vary over time. The platform's objective is regret minimization, and is aligned with maximizing user utility. To capture the fact that content may generate observable user data at different rates, each arm $i$ has a fixed \textit{feedback rate} $f_i$ representing the probability of the algorithm observing a sample from that arm's loss distribution in a given round. 
The platform must then determine how to account for these differential rates of information gain in its learning algorithm---a choice which can significantly impact what content users see. 

Rather than only focusing on regret, we study how often a bandit algorithm engages with individual arms, and how this depends upon the arm's feedback rate $f_i$. To quantify engagement with an arm, we introduce two measures: the \textit{arm pull count} $\APC_i$ (how often an algorithm pulls arm $i$ in $T$ rounds), and the \textit{feedback observation count} $\FOC_i$ (how often it sees feedback from arm $i$ in $T$ rounds).\footnote{While these measures are linked through $f_i$, they can lead to different user impacts, so we consider both.} To formalize how these measures vary with $f_i$ for a given algorithm, we introduce the notions of \textit{feedback monotonicity} and \textit{balance}. 
At a high level, an algorithm is positive (negative) feedback monotonic with respect to \APC (\FOC) if, when an arm's $f_i$ increases, the algorithm weakly increases (decreases) $\APC_i$ ($\FOC_i$). An algorithm satisfies \textit{balance} when such a change in $i$'s feedback rate is guaranteed to have \textit{no} effect on $\APC_i$ ($\FOC_i$).

The following examples illustrate how these types of feedback monotonicity properties can in turn affect downstream user experience on the platform. Note that these effects transcend what is typically captured in individual utilities (how much a given user likes a given piece of content), instead constituting community-level, platform-level, and society-level impacts. 
\begin{example}[Own-group content and $\APC$] 
\label{ex:owngroup-apc}
For a given user, $f_i$ may be higher for content that appears to be produced by \textit{own-group} creators, e.g. creators who are demographically or ideologically similar to the user \citep{agan2023automating}.\footnote{The empirical study of \citet{agan2023automating} is explicitly motivated by $\APC$ in the context of recommendations.} 
\APC captures how often content is shown to users. If an algorithm induced positive monotonicity in \APC, users may see own-group content disproportionately often, contributing to problems such as polarization and echo chambers.
\end{example}

\begin{example}[Incendiary content and $\FOC$.]
\label{ex:incendiary-foc}
    Observable feedback often occurs in the form of ``retweets,'' and high $f_i$ can be associated with highly controversial or incendiary content. 
    $\FOC$ captures observable engagement metrics. If an algorithm induced positive monotonicity in $\FOC$, 
    creators may be incentivized to optimize for $\FOC$ by creating more incendiary content; this would increase the incendiariness of the overall landscape of content available on the platform. Moreover, since retweets by users about incendiary content are visible to other users, positive monotonicity in $\FOC$ may also create a toxic environment on the platform and impact the overall user experience. 
\end{example}
We defer a discussion of further examples to Appendix \ref{appendix:examples}.

\subsection{Our contributions}
We initiate the study of how a bandit algorithm's choice of arms to pull correlates with the probability of observing feedback for those arms. 
We introduce the measures $\APC$ (Def. \ref{measure:pulled}) and $\FOC$ (Def. \ref{measure:observe}), 
which capture two aspects of how a bandit algorithm treats arms that can result in downstream impacts on users; feedback monotonicity and balance are the algorithmic properties we aim to analyze. We summarize our results in Table \ref{table:summary}.

Our main technical finding is that no-regret algorithms for the probabilistic feedback setting can exhibit a range of behavior with respect to $\APC$ and $\FOC$ (Table \ref{table:summary}). We illustrate this by constructing different families of no-regret algorithms with strikingly different monotonicity properties for both $\APC$ and $\FOC$, where these differences are driven by how the algorithms respond to probabilistic feedback. 

\begin{enumerate}
    \item We present three black-box transformations $(\BBDivide, \BBPull, \BBDivideAdjusted)$ which convert a generic no-regret bandit algorithm for the \textit{deterministic} feedback setting  into a bandit algorithm for the \textit{probabilistic} feedback setting (Section \ref{sec:blackbox}); each of these transformations has different consequences for $\APC$ and $\FOC$. 
\item We analyze these black-box transformations applied to concrete algorithms (UCB and AAE), and achieve both improved regret bounds and stricter monotonicity guarantees (Sections \ref{subsec:regret} and \ref{subsec:mono}). 
\item We give an algorithm which improves known regret bounds for adversarial losses, removing the dependence on the \textit{minimum} feedback probability in \cite{esposito2022learning} 
(Section \ref{subsec:exp3}). 
\end{enumerate}

Compared to regret, $\APC$ and $\FOC$ are finer-grained measures for the behavior of a bandit algorithm, so tightly analyzing how these properties change with $f_i$ also requires finer-grained control than in typical regret analyses. 
To isolate the impact of modifying feedback probabilities, we use a coupling argument to explicitly compare the algorithm's behavior on two instances that are identical except for one $f_i$. 

\begin{table*}[t]
\small
\addtolength{\tabcolsep}{-1pt}
\centering
\makebox[0pt]{
\begin{tabular}{llllllll}
\toprule
\multicolumn{2}{l}{\textbf{Algorithm}} &  \multicolumn{2}{l}{\textbf{\APC mono.}} & \multicolumn{2}{l}{\textbf{\FOC mono.}} & \multicolumn{2}{l}{\textbf{Regret upper bound}} \\
\cmidrule(lr){1-2}\cmidrule(lr){3-4}\cmidrule(lr){5-6}\cmidrule(lr){7-8}
$\BBDivide(\ALG,\fst)$ & Alg.~\ref{algo:BB-divide}  &  $\approx$ & Thm.~\ref{thm:feedbackmonotonicitybbdivide} & $+$  & Thm.~\ref{thm:feedbackmonotonicitybbdivide} & \small{$R_\ALG\left(T \fst / \ln (T) \right) \cdot \ln (T) / \fst$} & Thm.~\ref{thm:bb1-regret} \\[0.5em]
$\BBPull(\ALG)$ & Alg.~\ref{algo:BB-pull}&  $\approx/-$ & Thm.~\ref{thm:feedbackmonotonicitybbpull}& $\approx/+$ & Thm.~\ref{thm:feedbackmonotonicitybbpull} & \small{$R_\ALG(T) \cdot 1/\min_i f_i$} & Thm.~\ref{thm:bb2-regret}\\[0.5em]
$\BBDivideAdjusted(\ALG)$ & Alg.~\ref{algo:BB-divide-adjusted}& $\approx/+$ & Thm.~\ref{thm:feedbackmonotonicitybbda} & $\approx/+$ & Thm.~\ref{thm:feedbackmonotonicitybbda} & \small{$R_\ALG(T) \cdot 4 \ln (T) / \min_j f_j$} &Thm.~\ref{thm:bb3-regret}  \\[0.3em]
\hline \\[-0.8em]  

$\BBPull$(AAE) & Alg.~\ref{algo:BB-pull-se}&$-^\diamond$ & Thm.~\ref{thm:feedbackmonotonicitybbpullaae} & $\approx^\diamond$ & Thm.~\ref{thm:feedbackmonotonicitybbpullaae} & $ O \left(\ln (T) \cdot \sum_i 1/(\Delta_i f_i) \right)$ & Thm.~\ref{thm:bbpullregretaaeucb}\\[0.5em]
$\BBDivideAdjusted$(AAE) & Alg.~\ref{algo:BB-da-se} &$+^\diamond$ & Thm.~\ref{thm:feedbackmonotonicitybbdaaae} & $+$ & Thm.~\ref{thm:feedbackmonotonicitybbdaaae} & \small{$ O \left(\ln^2 (T) \cdot\sum_i 1 / (\Delta_i \min_j f_j) \right)$}  & Thm. \ref{thm:bb3-regret}  \\[0.5em] 

\hline \\[-0.8em] 
3-Phase EXP3 & Alg.~\ref{algo:EXP3-3phase} & &&&& \small{$O \left( \sqrt{T \ln (K) \sum_{i \in [K]} 1/f_i}  \right)$} & Thm.~\ref{thm:EXP3-3phase-regret} \\
\bottomrule
\end{tabular}
}
\caption{$\ALG$ is any no-regret bandit algorithm with regret $R_{\ALG}$ in the deterministic feedback setting. $\fst$ is a tunable parameter. 
AAE is active-arm elimination; UCB is the upper confidence bounds algorithm. 
In columns \APC and \FOC, $+$, $-$ indicate \textit{strict} positive, negative feedback monotonicity. $\approx$ indicates \textit{approximate} balance, differing across arms by up to a factor of $O(1/T)$. 
$\approx/+$ (resp. $\approx/-$) means that either approximate balance or positive (resp. negative) monotonicity may be achieved, depending on the underlying algorithm and problem instance. The superscript $\diamond$ indicates that the stated property holds only for suboptimal arms.
}
\label{table:summary}
\end{table*}

\subsection{Related work}
Our work relates to research on multi-armed bandits, empirical evidence for probabilistic feedback, real-world interpretation of $\FOC$ and $\APC$, and the societal impacts of recommender systems.

\xhdr{Multi-armed bandits.} 
        Our technical results build on the vast literature on multi-armed bandits (see \cite{hazan2016introduction} for a textbook treatment). Most relevant to our work is \textit{multi-armed bandits with probabilistic feedback graphs} (e.g.~\cite{esposito2022learning}). This extends the framework of multi-armed bandits with feedback graphs \citep{ACDK15}, where at each round, when an arm is pulled, the loss of all of the neighbors of that arm is observed. In the probabilistic feedback setting, the graphs are drawn from a \textit{distribution} at each time step. Recent work has studied regret guarantees for the probabilistic feedback graph setting for adversarial (e.g. \cite{esposito2022learning,ghari2022online}) and stochastic losses (e.g. \cite{li2020stochastic,cortes2020online}). We study a special case of this framework where the graph is always (a union of) self-loops and achieve an improved regret bound for adversarial losses (Theorem~\ref{thm:EXP3-3phase-regret}).

A handful of recent works have examined how the feedback observed by the bandit learner impacts the arm pull count $\APC$. For example, \citet{haupt2022risk} study how the variance of the noise in the observations of arm rewards impacts $\APC$ for $\epsilon$-Greedy in a 2-arm setting; in contrast, we vary the feedback probability that the reward is observed and study the behavior of more general algorithms and instances. Moreover, motivated by clickbait, \citet{buening2023bandits} also study how feedback probabilities 
impact $\APC$, focusing on the $K$ arms (content creators) strategically selecting feedback probabilities to optimize for $\APC$. However, \citet{buening2023bandits} focuses on designing incentive-aware platform algorithms that optimize a utility function (that can take into account both clickthrough rates and arm rewards); in contrast, we consider no-regret platform algorithms that optimize only for arm reward, and analyze their impact in terms of monotonicity properties.

Separately, 
the measure $\APC$
has been studied in recent work
that aims to achieve \textit{fairness} across arms, with a focus on ensuring that higher mean reward arms are pulled more often than lower mean reward arms \citep{joseph2016fairness}. Another notion of stronger constraints on arm pulls is \textit{replicability} \citep{esfandiari2022replicable}, which seeks to ensure that an algorithm will pull arms in the same order across identical instances with high probability. Though related, this is distinct from our definition of $\APC$ and our goals of controlling monotonicity. Their algorithms employ a similar ``block'' approach as ours, though they give explicit algorithms rather than black-box transformations for generic algorithms.

\xhdr{Empirical evidence for probabilistic feedback.}
The idea that recommendation platforms may not observe all user ``utilities'' at all times is well-studied. While the intuition that expressed preferences may not be a full picture of their true preferences underpins an entire subfield of behavioral economics, we note several works here that study the problem applied to recommendation systems through a more algorithmic lens. In particular, probabilistic feedback often occurs for reasons that cannot be fully explained by quality of the content itself, which motivates our idea that $f_i$ should be studied separately from utilities. For example, \citet{schnabel2019shaping} show that probabilistic feedback can arise from interface design choices; \citet{joachims2005accurately} uses eye-tracking to show that clickthrough (i.e. feedback) rates depend on factors like ranking position and the set of other content that is shown, while \citet{joachims2017unbiased} applies this intuition to develop recommendation algorithms that are sensitive to the impact of ranking position on feedback rates; \citet{li2020picture} show that advertisements with images induce more user engagement than advertisements with text only, and that various attributes of images (e.g. colorfulness, professional versus amateur photography, human face, image-text fit) can also affect feedback rates; and \citet{cao2021understanding} find similar results in the context of fashion social media marketing, with both media richness and trustworthiness of marketing content as factors that affect feedback rates. 

    \xhdr{Real-world interpretations of $\FOC$ and $\APC$.}
Many (though of course not all) of the commonly-discussed harms of recommendation systems and online platforms can be formalized in terms of $\FOC$ and $\APC$. For example, the setting described in \citet{wsj2021}---harm to teen girls on Instagram---harm arises due to repeated exposure ($\APC$) to particular types of content; in the setting described in \citet{nyt2019}---radicalization on Youtube---the harm is due to ``rabbit holes'' that arise due to a combination of $\APC$ and $\FOC$. In fact, though \citet{nyt2019} is a general-audience reported case study, the more rigorous evaluation of \citet{ribeiro2020auditing} also examines both $\APC$ and $\FOC$ in the context of evaluating the role of algorithms in radicalization. 
Similarly, emotional contagion experiments (e.g. \cite{ferrara2015measuring} on Twitter, \cite{kramer2014experimental} on Facebook) often find that exposure to ($\APC$) content with emotional valence (either positive or negative) also affects the emotional valence of users' downstream posts.

Of particular note is \cite{agan2023automating}, which is the most closely-related empirical work to our knowledge. This recent work is an empirical study explicitly motivated by the harms of $\APC$ in recommendation systems, and correlations that may arise due to a learning algorithm’s treatment of information; this work motivates our \cref{ex:owngroup-apc}. In particular, they model $f_i$ as related to ``own-group'' content, e.g. demographic similarity of the creator, and are concerned about algorithmic bias in the sense of over-representing content from ``own-group'' creators. They show that under this model, standard learning algorithms do in fact induce correlations between ``own-group'' content (i.e. $f_i$) and how often it is shown (i.e. $\APC$). This work can be seen as an empirical validation that our theoretical framework may be concretely applicable.

\xhdr{Societal impacts of recommender systems.} 
This research thread has broadly investigated misalignment between recommendations and user utility. One proposed source of misalignment is potential discrepancies between metrics derived from observed behavior (e.g. engagement) and user utility (e.g. \cite{EW16, MBH21, KMR22}). Another source of misalignment that has recently been studied is how recommendations can shape user preferences over time \citep{ABCZ13, CDHR22, DBLP:conf/sigecom/DeanM22}. Furthermore, approaches for bringing human values in recommender system design have been investigated \citep{SVNAH21, S22}. 
Several other societal impacts of recommender systems have been studied including the emergence of filter bubbles \citep{FGR16}, stereotyping \citep{GKJG21}, the ability of users to reach different content \citep{DRR20}, and content creator incentives induced by the recommendation algorithm \citep{BT18, BRT20, JGS22,  HKJKD}. 

\section{Model \& Preliminaries}\label{sec:model}
We model the interaction between the platform/ learner and the user as a multi-armed bandit (MAB) that happens over $T$ rounds. Each arm corresponds to a piece of content; ``pulling'' an arm corresponds to recommending that piece of content to the user. We say that an arm \emph{``returns feedback''} if we observe its loss upon pulling it. 

An \textit{instance} of our problem is specified by $\calI = \{\calA, \calF, \calL\}$, where $\calA$, $\calF$, and $\calL$ are defined as follows. Let $\calA := [K]$ denote the set of $K$ arms.
Let $\calF := [f_1, \ldots, f_K]$ be the feedback probabilities for each arm, i.e., the value $f_i \in [0,1]$ denotes the probability with which  arm $i$ returns feedback when pulled. The probabilities $\calF$ can be chosen arbitrarily by an adversary and are unknown to the learner, but remain fixed throughout the $T$ rounds. In the \textit{deterministic feedback setting}, $f_i = 1$ for all $i$; in the \textit{probabilistic feedback setting}, $f_i$ 
can be less than $1$.
Let $\calL$ denote the process by which the losses $\ell_{i,t}$ are generated. Losses can be \emph{adversarial} or \emph{stochastic}. For adversarial losses, the sequence $\{\ell_{i,t}\}_{i \in [K], t \in [T]}$ can be chosen arbitrarily, but obliviously, i.e. before the start of the algorithm. For stochastic losses, each arm $i \in [K]$ has
a loss distribution with mean $\bell_i$ and variance $1$ from which the per-round loss $\ell_{i,t}$ is sampled. For each arm $i$, we define $\Delta_i:= \bell_i - \min_j\bell_j$ to be the difference between the mean loss of arm $i$ and the mean loss of the optimal arm $\min_j\bell_j$.

For an arm $i \in [K]$, the random variable $X_{i,t}$  corresponds to whether feedback will be observed if arm $i$ is pulled at round $t$, i.e., $X_{i,t} \sim \Bern(f_i)$. 
With some abuse of notation, we let $\ell_{i_\tau, \tau}\cdot X_{i_\tau,\tau}$
represent the observed loss at time $\tau$, where
$\ell_{i_\tau, \tau}\cdot X_{i_\tau,\tau}= \perp$ denotes lack of observation when $X_{i_\tau,\tau}=0$ and $\ell_{i_\tau, \tau} \cdot X_{i_\tau,\tau}=\ell_{i_\tau, \tau}$ denotes the observed loss when $X_{i_\tau,\tau}=1.$
Let $H_t = \{(i_\tau, \ell_{i_\tau, \tau} \cdot  X_{i_\tau,\tau}, X_{i_\tau,\tau} \})\}_{\tau \in [t-1]}$ for some round $t$ denote the \emph{history} of play until round $t$,
and $\calH_t$ denote the family of all possible history trajectories until round $t$. An algorithm $\ALG: \bigcup_{t=0}^T\calH_t \to [K]$ produces a (possibly randomized) mapping from histories of play to arms to be chosen. We sometimes overload notation and write $\ALG(t)$ to denote the mapping from $H_t$ to $i_t$.

\subsection{Measuring the behavior of an algorithm on an instance}
We capture the behavior of an algorithm by the following three quantities. The first is the standard objective function in multi-armed bandits, an algorithm's \emph{(pseudo-)regret}:\footnote{Throughout, we omit ``pseudo'' from the definition below for succinctness. } 
\begin{definition}[Regret] 
The (pseudo-)regret of an algorithm $\ALG$ playing arm $i_t \in [K]$ at round $t$ is defined as: \[R_{\ALG}(T) =  \E [\sum_{t \in [T]} \ell_{i_t, t}]  - \min_{j\in [K]} \E [\sum_{t \in [T]} \ell_{j, t}] .\]
\end{definition}

We are also interested in how much an algorithm engages with individual arms. To capture this, we define the quantities $\FOC_i$ and $\APC_i$ for arms $i \in [K]$. 

\begin{definition}[Arm Pull Count ($\APC$)]
\label{measure:pulled}
Given a problem instance $\calI$, the arm pull count ($\APC$) of an arm $i$ over a run of an algorithm $\ALG$ is equal to 
\[\APC_i(\calI; \ALG) = \E\big[ \sum_{t \in [T]} \mathds 1[i_t = i]\big].\]
\end{definition}

\begin{definition}[Feedback Observation Count ($\FOC$)]
\label{measure:observe} 
Given a problem instance $\calI$, the feedback observation count ($\FOC$) of arm $i$ over a run of an algorithm $\ALG$ is equal to \[\FOC_i(\calI; \ALG)=  \E\big[\sum_{t \in [T]} \mathds 1[i_t = i] \cdot X_{i_t,t}\big].\]
\end{definition}
\noindent In all three definitions, the expectation is taken with respect to randomness in both the algorithm and the instance (i.e.,  loss distributions and feedback observations). When the instance $\calI$ and algorithm $\ALG$ are clear from context, we write $\FOC_i$ and $\APC_i$. A simple consequence of the definitions is that $\FOC_i$ and $\APC_i$ are related by a multiplicative factor of $f_i$. 
\begin{restatable}{lemma}{lemrelationship}
\label{prop:relationship}
For any arm $i$, instance $\calI$, and algorithm $\ALG$, it holds that $\FOC_i(\calI) = f_i \cdot \APC_i(\calI). $
\end{restatable}
We prove \Cref{prop:relationship} in Appendix \ref{appendix:modelproofs}; the result follows from noting that, at any time $t$, the realization of $X_{i_t,t}$ is independent of all history up to time $t$. 

\subsection{Feedback monotonicity and balance}
Using $\FOC$ and $\APC$, we formalize how an algorithm responds to feedback probabilities through \textit{feedback monotonicity} and \textit{balance}. We let $\widetilde{\calF}(i)$ denote a set of feedback probabilities in which we have modified arm $i$'s feedback rate, holding all else constant: that is, $\smash[t]{f_i \neq \tilde f_i}$, and 
$\forall j \neq i, \smash[t]{f_j = \tilde f_j}$. For an instance $\calI = \{\calA, \calF, \calL\}$, we use ${\widetilde \calI}$ to notate the instance identical to $\calI$ except for $\smash{\tilde f_i}$, that is $\smash{\widetilde \calI = \{\calA, \widetilde \calF(i), \calL\}}$. In our analysis, we let $i$ be arbitrary, and only modify feedback for one arm $i \in [K]$ at a time, and
analyze how $\APC_i$ and $\FOC_i$ would change if the feedback probabilities were $\smash{\widetilde{\calF}(i)}$ instead of 
$\calF$. We formally define \textit{monotonicity} and \textit{balance} below.
\begin{definition}[Feedback monotonicity.] \label{def:feedback monotonicity} An algorithm exhibits \textit{positive (resp. negative) feedback monotonicity} wrt measure $Q \in \{\APC,\FOC\}$ if and only if for all $\calI = \{\calA, \calF, \calL\}$, for all $i \in \calA$, and for all pairs $\tilde{f}_i, f_i \in [0,1]$ such that $\tilde{f}_i > f_i$, we have that $Q_i(\widetilde{\calI}) \geq Q_i(\calI)$ (resp. $Q_i(\widetilde{\calI}) \leq Q_i(\calI)$). 
\end{definition}

\begin{definition}[Balance] \label{def:balance}
An algorithm is \emph{balanced} with respect to a measure $Q$ if for all pairs of $\smash{\widetilde{\calI}, \calI}$, we have that
$\smash{Q_i(\calI) = Q_i(\widetilde{\calI})}$.
\end{definition}

The goal of our work is to examine the landscape of potential feedback monotonicity properties (positive, negative, balance) for each measure ($\APC$ and $\FOC$). We note that not all combinations are achievable: in particular, the measure $\FOC$ cannot satisfy \textit{balance} or \textit{negative feedback monotonicity} across all instances and all arms. 
\begin{restatable}{proposition}{propimpossibility}
\label{prop:impossibility}
Suppose that $\ALG$ has sublinear regret for stochastic losses in the probabilistic feedback setting. For any pair of feedback probabilities $\tilde{f}_i > f_i$, for sufficiently large $T$, there exists an instance $\calI$ such that $\FOC_i(\tilde{\calI}) > \FOC_i(\calI)$.
In fact, $\FOC_i(\tilde{\calI}) - \FOC_i(\calI) > \frac{9}{10} \cdot T(\tilde f_i - f_i)$.
\end{restatable}
We prove \Cref{prop:impossibility} in Appendix \ref{appendix:modelproofs}; the high-level idea is to note that if $i$ were to be the optimal arm, it must be pulled $T - o(T)$ times on each instance. 

In the remaining sections, as we analyze what feedback monotonicities are achievable, we sometimes consider relaxed versions of the precise definitions (e.g., restricting to suboptimal arms), as we will make explicit in the theorem statements. 
\section{Algorithmic Transformations and Implications for \APC and \FOC} 
\label{sec:blackbox}

In order to understand how an algorithm behaves with respect to $\APC$ and $\FOC$, we need to disentangle how it reacts to probabilistic feedback and how it incorporates feedback observations to make future decisions. To do this, we study \emph{black-box (\BB) transformations} of a generic no-regret algorithm $\ALG$ for the deterministic feedback setting into a no-regret algorithm $\BB(\ALG)$ that accounts for probabilistic feedback. We call these transformations ``black-box'' as they require only query access to $\ALG$.

We analyze three different black-box transformations, which exhibit distinct behavior with respect to regret, $\APC$, and $\FOC$ (see Table \ref{table:summary}). The first, $\BBDivide$, divides the time horizon $T$ into equally-sized intervals and repeatedly pulls the same arm within each interval (Section \ref{subsec:bb1}). The second, $\BBPull$, repeatedly pulls the same arm until feedback is observed (Section \ref{subsec:bb2}). The third,  $\BBDivideAdjusted$, pulls each arm a pre-specified number of times, depending on the feedback probability $f_i$ of that arm (Section \ref{subsec:bb3}). 

\xhdr{High-level approach.} All three transformations use the high-level idea of dividing $T$ into \textit{blocks}, where the transformed algorithm $\BB(\ALG)$ pulls the same arm for all rounds in the same block. Rounds of $\BB(\ALG)$ are indexed $t \in [T]$. We index blocks, or rounds of $\ALG$, with $\phi$, and let $\Phi$ be the total number of blocks, or calls to $\ALG$, in the evaluation of $\BB(\ALG)$ up to $T$.  Finally, $S_\phi$ denotes the set of all $t$ indices that are within block $\phi$. Then each transformation proceeds as follows. 
For each $\phi$, we notate $\smash{i_\phi^\ALG} := \ALG(\phi)$, i.e. the arm selected by $\ALG$ in its $\phi$th round. Then, $\BB(\ALG)$ pulls $\smash{i_\phi^\ALG}$ for $t \in S_\phi$ and returns an observation $\smash{\ell_{i_\phi^\ALG, \phi}}$ to $\ALG$. Each transformation implements two steps differently: first, defining $S_\phi$, and second, returning $\smash{\ell_{i_\phi^\ALG, \phi}}$ to $\ALG$.

\subsection{$\BBDivide$: Transformation for balanced \APC and positive \FOC}
\label{subsec:bb1}

The first black-box transformation that we construct, $\BBDivide$, generates algorithms that approximately balance $\APC$. $\BBDivide$, formalized in Alg.\ref{algo:BB-divide}, separates $T$ into equally sized blocks of size $B = \ceil{3 \ln T / \fst}$, where $\fst \in (0, \min_i f_i]$ is a tunable parameter for trading-off regret and monotonicity. 

In the context of the high-level approach described above, 
the set $S_{\phi}$ is taken to be the next $B$ timesteps on $\BB(\ALG)$'s time horizon, i.e.
$S_{\phi} =\left\{(\phi -1) \cdot B +1, \ldots, \phi \cdot B \right\}$, and $\ell_{i_\phi^{\ALG}, \phi}$ is taken to be a uniform-at-random draw from the set of observations $\left\{\ell_{i_t, t} : X_{i_t, t} = 1, t \in S_{\phi}\right\}$.

\begin{algorithm2e}[htbp]
\caption{\textsc{BBDivide}($\ALG, \fst$)}
\label{algo:BB-divide}
\DontPrintSemicolon
\LinesNumbered
Set the block size to $B = \lceil 3 \ln T / \fst \rceil$ and initialize round count $t = 1$. \;
 \For {blocks $\phi \in \{1, \dots, \Phi = \floor{T/B}\}$}{
    Let $i_{\phi}^{\ALG} = \ALG(\phi)$ be the arm chosen by $\ALG$ on its $\phi$th timestep.\;
    Let $S_{\phi} =\left\{(\phi -1) \cdot B +1, \ldots, \phi \cdot B \right\}.$\;
    Pull $i_{\phi}^{\ALG}$ for rounds $t \in S_\phi$, i.e. $i_t = i_\phi^\ALG, \forall t \in S_\phi$ and let $t \gets t + 1$.\;
    \uIf{$\exists t \in S_\phi$ s.t. $X_{i_t, t} = 1$ (i.e. there are observations)}{
        Return a random observation to $\ALG$, i.e. $\ell_{i_\phi^\ALG, \phi} \sim \text{Unif} \{\ell_{i_t, t}: X_{i_t, t} = 1, t \in S_\phi\} $. \;
    }
    \lElse{
    Return a loss of $1$ to $\ALG$, i.e. $\smash{\ell_{i_\phi^\ALG, \phi}} = 1$}
    }
For remaining rounds, pull a random arm.\; 
\end{algorithm2e}
First, we show the following regret bound which holds for adversarial losses as well as stochastic losses.

\begin{restatable}[Regret $\BBDivide$]{theorem}{Rbbdivide}
\label{thm:bb1-regret}
Let $\ALG$ be any algorithm for the deterministic feedback setting that achieves regret at most $R_\ALG(T)$ when losses are adversarial (resp. stochastic). Then, for $\fst \in (0,\min_i f_i]$ and adversarial (resp. stochastic) losses, 
\[
R_{\BBDivide(\ALG, \fst)}(T) \le \frac{3 \ln T}{f^{\star}} R_\ALG\left(\frac{T f^{\star}}{3 \ln T} \right).
\] 
\end{restatable}

Theorem \ref{thm:bb1-regret} indicates that  the regret of $\BBDivide(\ALG)$ exceeds that of $\ALG$ by at most a factor of $3 \ln T/f^{\star}$.
When $\ALG$ is specified, 
direct applications of Theorem \ref{thm:bb1-regret} can improve the $f^\star$ dependence (Appendix \ref{appendix:corollaries}).

$\BBDivide$ \textit{approximately balances} $\APC$ and is \textit{positive feedback monotonic} with respect to $\FOC$. 
\begin{restatable}{theorem}{Fbbdividemono}[Impact of $\BBDivide$ on \APC and \FOC]
\label{thm:feedbackmonotonicitybbdivide}
Fix an instance $\calI = \left\{\calA, \calF, \calL\right\}$. Let $\smash{\widetilde{f}_i \ge f_i}$, and let $\smash{\widetilde{\calI} = \{\calA, \widetilde{\calF}(i), \calL\}}$. For any algorithm $\ALG$ for the deterministic feedback setting and for any $f^{\star} \le \min_i f_i$, if $T$ is sufficiently large, then the algorithm $\BBDivide(\ALG, \fst)$ satisfies \[|\APC_i(\calI) - \APC_i(\widetilde{\calI})| \le 1/T \hspace{4pt}\text{and}\hspace{4pt} \FOC_i(\widetilde{\calI}) > \FOC_i(\calI).\]
\end{restatable}

These monotonicity results, together with our regret bound, suggest that $\fst$ may have opposite effects on regret and monotonicity.
By Theorem \ref{thm:bb1-regret}, a higher value of $\fst$ decreases the regret bound, and setting $\fst$ to be close to $\min_i f_i$ is optimal.\footnote{Via an estimation phase, we can estimate $\min_i f_i$ 
without asymptotically affecting the regret guarantees.} Conversely, for monotonicity, 
while Theorem \ref{thm:feedbackmonotonicitybbdivide} set $\tilde{f}_i > f_i$, the reverse statements would also hold (i.e., we could have instead set $f_i > \tilde{f}_i \ge f^{\star}$).\footnote{The lower bound on $\tilde{f}_i$ ensures that there is still an observation in each block with high probability, despite the lower feedback probability.} As such, 
a higher value of $f^{\star}$ restricts the set of feedback probabilities under which the monotonicity results apply. We give full proofs in \cref{appendix:bbdivide}.

\subsection{$\BBPull$: Transformation for negative \APC and positive \FOC}
\label{subsec:bb2}
Our second transformation $\BBPull$, formalized in Alg. \ref{algo:BB-pull}, generates algorithms with
\textit{negative} monotonicity in $\APC$. 
$\BBPull(\ALG)$ will pull $i_\phi^\ALG$ until feedback is observed for that arm, return the observation to $\ALG$. 
In terms of the structure described at the beginning of the section, 
if block $\phi$ starts at time step $t$, $S_{\phi}$ is implicitly defined as the set of time steps until there is an observation: i.e., $S_{\phi} = \left\{t' \ge t \mid X_{i_\phi^\ALG, t''} = 0  \text{  } \forall t'' < t'\right\}$. The loss passed to $\ALG$ is the observation made at the end of $S_\phi$, i.e. 
$\ell_{i_\phi^{\ALG}, \phi} := \ell_{i_\phi^{\ALG}, \max\left\{t \mid t \in S_{\phi}\right\}}$. 

\begin{algorithm2e}[htbp]
\caption{\textsc{BBpull}($\ALG$)}
\label{algo:BB-pull}
\DontPrintSemicolon
\LinesNumbered
Begin with $\phi = 1$ and $t = 1$.\\
\While {$t \leq T$}{
    Let $i_\phi^{\ALG} = \ALG(\phi)$ be the arm chosen by $\ALG$ on its $\phi$th timestep. \;
    \While {$X_{i_\phi^\ALG, t} = 0$ and $t \leq T$}{
    Pull $i_\phi^\ALG$, i.e. $i_t = i_\phi^\ALG$
    , and let $t \gets t + 1$.} 
    Return $\ell_{i_t, t}$ to $\ALG$, i.e. $\ell_{i_\phi^\ALG,\phi} = \ell_{i_t, t}$  and let $\phi \gets \phi + 1$.
    }
\end{algorithm2e}

First, we bound regret for stochastic losses.\footnote{Theorem \ref{thm:bb2-regret} requires stochastic losses, because regret analysis in Theorem \ref{thm:bb1-regret} relies on the block size being fixed (and arm-independent). 
}

\begin{restatable}[Regret $\BBPull$]{theorem}{Rbbpull}
\label{thm:bb2-regret}
Let $\ALG$ be any algorithm for the deterministic feedback setting that achieves regret at most $R_{\ALG}(T)$ for stochastic losses. Then, for stochastic losses, $\BBPull(\ALG)$ achieves regret at most 
\[
R_{\BBPull(\ALG)}(T) \le R_\ALG(T) \cdot \frac{1}{\min_i f_i}.
\] 
\end{restatable}

Theorem \ref{thm:bb2-regret} shows that applying $\BBPull$ increases regret by up to a $\nicefrac{1}{\min_i f_i}$ factor, 
improving upon the regret for $\BBDivide$ (Theorem \ref{thm:bb1-regret}) by a $\ln T$ factor. 
We next formalize the monotonicity of $\BBPull$. 
We show that although $\APC$ is negative monotonic, $\FOC$ maintains positive monotonicity, like in $\BBDivide$. 
\begin{restatable}{theorem}{Monobbpull}[Impact of $\BBPull$ on \APC and \FOC]
\label{thm:feedbackmonotonicitybbpull}
Fix an instance $\calI = \left\{\calA, \calF, \calL\right\}$ with stochastic losses. Let $\smash{\widetilde{f}_i \ge f_i}$, and let $\widetilde{\calI} = \{\calA, \widetilde{\calF}(i), \calL\}$. For any algorithm $\ALG$ for the deterministic feedback setting, the algorithm $\BBPull(\ALG)$ satisfies \[\APC_i(\calI) \ge \APC_i(\widetilde{\calI}) \hspace{4pt}\text{and}\hspace{4pt} \FOC_i(\calI) \le \FOC_i(\widetilde{\calI}).\]   
\end{restatable}
\begin{proof}[Proof sketch]
The observations provided to $\ALG$ by $\BBPull(\ALG)$ are identically distributed on $\calI$ and $\smash{\widetilde\calI}$. A coupling argument illustrates that the only source of difference is in the number of times that $\ALG$ is called by $\BBPull(\ALG)$ on $\calI$ and $\widetilde\calI$. A higher $f_i$ means that $\ALG$ can be called more times before $T$ runs out on $\smash{\widetilde\calI}$, giving positive monotonicity in $\FOC$.  
For $\APC$, higher $\smash{f_i}$ means fewer pulls per observation. Full proofs are deferred to Appendix \ref{appendix:bb2}.
\end{proof}

\subsection{$\BBDivideAdjusted$: Transformation for positive $\APC$ and positive \FOC}
\label{subsec:bb3}

The third-black box transformation generates algorithms that are \textit{positive} monotonic in $\APC$. Given an algorithm $\ALG$ for the deterministic feedback setting, $\BBDivideAdjusted$, formalized in Alg. \ref{algo:BB-divide-adjusted}, combines conceptual ingredients from $\BBDivide$ and $\BBPull$ (DA is short for $\text{DivideAdjusted}$). 
As in $\BBDivide$, block sizes are pre-specified, but are also arm-dependent, as in $\BBPull$.
To set block sizes $B_i$ for each arm, we make the additional assumption that the algorithm designer knows the feedback probabilities apriori, and set $\smash{B_i = \ceil{ \frac{3\ln T}{\fst}(1+f_i)}}$ for $\fst \in (0, \min_jf_j]$. 
In terms of the high-level approach described at the beginning of the section, the set $S_{\phi}$ is taken to be $\left\{(\phi -1) \cdot B_{i_\phi}^{\ALG} +1, \ldots, \phi \cdot B_{i_\phi}^{\ALG} \right\}$, and $\ell_{i_\phi^{\ALG}, \phi}$ is taken to be a uniform-at-random draw from the set of observations $\left\{\ell_{i_t, t} : X_{i_t, t} = 1, t \in S_{\phi}\right\}$.


\begin{restatable}{algorithm2e}{BBDA} 
\caption{\textsc{BBda}($\ALG, f^\star$)}
\label{algo:BB-divide-adjusted}
\DontPrintSemicolon
\LinesNumbered
Begin with $\phi = 1$ and $t=1$. \\
\While {$t \leq T$}{
    Let $i_\phi^{\ALG} = \ALG(\phi)$, 
    $B_\phi = \ceil{\frac{3 \ln T}{f^\star}(1 + f_{i_\phi^\ALG})}$, and $S_\phi = \{t, t + 1, \dots, \min(t + B_\phi, T)\}$.\;
    \For{$t \in S_\phi$}{
    Pull $i_{\phi}^{\ALG}$, i.e. $i_t = i_\phi^\ALG$, and let $t \gets t + 1$.
    }
    \uIf{$\exists t \in S_\phi$ s.t. $X_{i_t, t} = 1$ (i.e. there are observations)}{
        Return a random observation to $\ALG$, i.e. $\ell_{i_\phi^\ALG, \phi} \sim \text{Unif} \{\ell_{i_t, t}: X_{i_t, t} = 1, t \in S_\phi\} $. \;
    }
    \lElse{
    Return a loss of $1$ to $\ALG$, i.e. $\smash{\ell_{i_\phi^\ALG, \phi}} = 1$.}
    }
    Update $\phi \gets \phi + 1$.
\end{restatable}

First, we show the following regret bound; Theorem \ref{thm:bb3-regret} requires stochastic losses because $B_i$ is arm-dependent.

\begin{restatable}{theorem}{Regretbbda}[Regret $\BBDivideAdjusted$]
\label{thm:bb3-regret}
Let $\ALG$ be any algorithm for the deterministic feedback setting that achieves regret at most $R_{\ALG}(T)$ when the losses are stochastic. Then, for stochastic losses, for any $f^{\star} \le \min_i f_i$, the algorithm $\BBDivideAdjusted(\ALG, f^\star)$ achieves regret at most
\[
R_{\BBDivideAdjusted(\ALG)}(T) \le \frac{6 \ln T}{\fst} R_\ALG\left(\frac{T \fst} {3\ln T}\right).
\]
\end{restatable}

Since the block size explicitly increases with $f_i$, $\BBDivideAdjusted(\ALG)$ pulls an arm more frequently when its feedback probability increases. 
More formally, increasing the feedback probability of an arm (approximately) \textit{increases} the number of times it is pulled within any block where $\ALG$ selects it. 
We show $\BBDivideAdjusted$ exhibits positive monotonicity for both $\APC$ and $\FOC$. 
\begin{restatable}{theorem}{Monobbda}[Impact of $\BBDivideAdjusted$ on \APC and \FOC]
\label{thm:feedbackmonotonicitybbda}
 Fix an instance $\calI = \left\{\calA, \calF, \calL\right\}$ with stochastic losses. Let $\smash{\tilde{f}_i \ge f_i}$, and let $\smash{\widetilde{\calI} = \{\calA, \widetilde{\calF}(i), \calL\}}$. For any algorithm $\ALG$ for the deterministic feedback setting and for any $f^{\star} \le \min_i f_i$, the algorithm $\BBDivideAdjusted(\ALG, f^\star)$ satisfies  \[\APC_i(\widetilde{\calI}) \ge \APC_i(\calI) - 1/T \hspace{4pt}\text{and}\hspace{4pt}\FOC_i(\widetilde{\calI}) \ge \frac{\tilde{f}_i}{f_i} \FOC_i(\calI) - \frac{\tilde{f}_i}{T} > \FOC_i(\calI).\]   
\end{restatable}

Theorem \ref{thm:feedbackmonotonicitybbda} also follows from a coupling argument; we defer proofs to Appendix \ref{appendix:bb3}. 
\section{Finer-Grained Analyses of Monotonicity and Regret}
\label{sec:beyond-bb}

While the black-box transformations in Section \ref{sec:blackbox} provided a clean way to analyze the behavior of $\FOC$ and $\APC$, the regret bounds obtained for those transformations unfortunately scaled with the \textit{minimum} feedback probability 
$1/\min_{i \in [K]} f_i$
of any arm, and the monotonicity analysis did not differentiate between strict monotonicity and balance.
In this section, we introduce four concrete algorithms---which are variants of EXP3 \citep{auer2002nonstochastic}, UCB \citep{auer2002finite}, and AAE (Active Arm Elimination \citep{even2002pac})---that have improved monotonicity and/or regret guarantees.

In Section \ref{subsec:regret}, we show that applications of $\BBPull$ to $\AAE$ and $\UCB$ can also achieve improved regret bounds that scale with the average feedback probability $\smash{\sum_{i \in [K]} \nicefrac{1}{f_i}}$ across arms rather than 
the minimum feedback probability 
$\smash{\nicefrac{K}{(\min_{i \in [K]} f_i)}}$. Section \ref{subsec:mono} shows that more explicit analyses of $\BBPull$ and $\BBDivideAdjusted$ applied to AAE enjoy stronger monotonicity guarantees than what is implied by naive applications of Theorems \ref{thm:feedbackmonotonicitybbpull} and \ref{thm:feedbackmonotonicitybbda}. Finally, in Section \ref{subsec:exp3}, we move beyond black-box transformations and present a variant of EXP3, which also achieves regret that scales with $\smash{\sum_{i \in [K]} \nicefrac{1}{f_i}}$ in the adversarial case, but which lacks clean monotonicity properties.

\subsection{Improved regret guarantees}
\label{subsec:regret}

Consider $\BBPull$ applied to $\AAE$ and $\UCB$, formalized in Algorithms \ref{algo:BB-pull-se} and \ref{algo:BB-pull-ucb} below. First, we show that these algorithms achieve improved regret bounds compared to a naive application of Theorem \ref{thm:bb1-regret}.

\begin{algorithm2e}[htbp]
  \caption{$\BBPull(\text{AAE})$}
  \label{algo:BB-pull-se}
  \DontPrintSemicolon
  \LinesNumbered
  Maintain active set $A$; start with $A := [K]$.
  \\
  Initialize phase $s=1$ and $t = 1$. 
  \\
  \While{$t \le T$}{
      \For{arm $i \in A$}{
      Let $R_{i,s} = \emptyset$. \\
      \While{$|R_{i,s}| \le 8 \ln T \cdot 2^{2s}$ and $t \le T$}{
      \uIf{$X_{i, t} = 1$}{Append $R_{i,s} \gets R_{i,s} \cup \left\{t \right\}$.}
      Pull $i_t = i$, and increment $t \gets t + 1$.\\
      }
      Calculate the mean $\mu_s(i) := -\frac{1}{|R_{i,s}|}\sum_{t' \in R_{i,s}} \ell_{i, t'}$ of the negative of all observations.\footnotemark\\
      Set $\text{LCB}_s(i) = \mu_s(i) - 2^{-s}$ and 
      $\text{UCB}_s(i) = \mu_s(i) + 2^{-s}$.
      }
      For any arm $i \in A$ where $\exists j \in A$ such that $\text{LCB}_s(j) > \text{UCB}_s(i)$, remove $i$ from $A$.  \\
      Increment $s \gets s+1$.
  } 
  \end{algorithm2e}
  \footnotetext{The negative is introduced to convert losses into utilities.}

  \begin{algorithm2e}[htbp]
    \caption{$\BBPull(\text{UCB})$}
    \label{algo:BB-pull-ucb}
    \DontPrintSemicolon
    \LinesNumbered
    Initialize number of pulls $n_i = 0$ for all $i \in [K]$. \\
    Initialize empirical mean $\mu(i) = 0$ for all $i \in [K]$. \\
    Initialize $t = 1$.
    \\
    \While{$t \le T$}{
        \uIf{$n_i = 0$ for any arm $i \in [K]$}{Let $i_t$ be the arm with the smallest index such that $n_{i_t} = 0$.}
        \Else{
        For every arm $i \in [K]$, compute $\text{UCB}(i) = \mu(i) + \sqrt{\frac{6 \ln T}{n_i}}$. \\
        Let $i_t = \text{argmax}_{j \in [K]} \text{UCB}(j)$.}
        Pull arm $i_t$. \\
        \uIf{$X_{i, t} = 1$}{
        Update the empirical mean $\mu(i) \gets \frac{n_{i_t} \cdot \mu(i)}{n_{i_t}+1} - \frac{\ell_{i_t,t}}{n_{i_t}+1} $.\\
        Increment $n_{i_t} \gets n_{i_t} + 1$.}
        Increment $t \gets t+1$.
    } 
    \end{algorithm2e}
\begin{restatable}{theorem}{Regretbbpullaaeucb}
\label{thm:bbpullregretaaeucb}
 On any stochastic instance $\mathcal{I} = \left\{\calA, \calF, \calI\right\}$, $\BBPull(\AAE)$ (presented in Algorithm~\ref{algo:BB-pull-se}) and 
 $\BBPull(\UCB)$ (presented in Algorithm~\ref{algo:BB-pull-ucb}) have regret bound of
 $O \left(\sqrt{T \ln (T) \sum_{i \in [K]} {1}/{f_i}} \right)$ 
 and an instance-dependent regret bound of
 $O \left(\sum_{i \in [K] \mid \Delta_i > 0} \frac{\ln T }{\Delta_i f_i} \right)$.   
\end{restatable}
\begin{proof}[Proof sketch]
As in the standard analysis of AAE and UCB,
we upper bound the number of times that an arm $i$ can be pulled in terms of its reward gap $\Delta_i$. To do so, we first bound the maximum number of phases an arm $i$ is active, and then show a high probability bound on the maximum number of times $i$ can be pulled in a given phase in terms of $f_i$. We defer a full proof to \Cref{subsec:analysisbbpullaae} for $\BBPull(\AAE)$ and to \Cref{appendix:UCB} for $\BBPull(\UCB)$. 
\end{proof}

Theorem \ref{thm:bbpullregretaaeucb} converts  the dependence on the minimum feedback probability $\min_j f_j$ from Theorem \ref{thm:bb2-regret} into a finer-grained dependence on the per-arm feedback probabilities $f_j$. 
In particular, in the instance-dependent regret bounds of Theorem~\ref{thm:bbpullregretaaeucb}, the \emph{``effective''} gap $\Delta_i f_i$ can be small either if the arm is close to optimal or if the feedback probability is small. In contrast, the regret bound of $O (\sum_{i \in [K]} \frac{\ln T }{\Delta_i \min_j f_j})$ given by applying Theorem \ref{thm:bb2-regret} directly has an effective gap $\Delta_i \min_j f_j$ that can be small even if $\smash{\min_j f_j}$
is small. Similarly, the instance-independent regret bounds in Theorem~\ref{thm:bbpullregretaaeucb}, in comparison to the instant-independent regret bound of $\smash[b]{O (\sqrt{T (\ln T)\frac{K}{\min_i f_i}})}$, also replace the dependence on $K / \min_i f_i$ with $\sum_{j \in [K]} 1/f_j$. 

Interestingly, the improvement in regret bounds relies on the specifics of $\BBPull$: we do not expect it to be possible to obtain a similar improvement in regret bounds for $\BBDivide$ or $\BBDivideAdjusted$ applied to $\AAE$ or $\UCB$. Intuitively, this is because $\BBPull$ does not pull any arm more than is necessary to observe feedback, while $\BBDivide$ and $\BBDivideAdjusted$ must pull all arms (including sub-optimal ones) a prespecified number of times, by definition.

\subsection{Stricter monotonicity guarantees}
\label{subsec:mono}

When the black-box transformations $\BBPull$ and $\BBDivideAdjusted$ are applied to $\AAE$, we show stronger monotonicity properties for \textit{suboptimal} arms, i.e. any arm $i$ where $\bar{\ell}_{i} > \min_{j \in [K]} \bar{\ell}_j$.
Specifically, $\BBPull(\AAE)$ achieves \textit{strict negative} monotonicity in $\APC$ and \textit{approximate balance} in $\FOC$ (Theorem \ref{thm:feedbackmonotonicitybbpullaae}), while $\BBDivideAdjusted(\AAE)$ achieves \textit{strict positive} monotonicity in $\APC$ (Theorem \ref{thm:feedbackmonotonicitybbdaaae}).
For both of these results, we focus on suboptimal arms for technical reasons (for large $T$, $\AAE$ eventually only pulls the optimal arm, which would equalize $\APC$). Despite this restriction, we expect that the qualitative impacts of monotonicity still arise even if monotonicity holds for all arms but the optimal arm.

We start by analyzing $\BBPull(\AAE)$ (formalized in Algorithm \ref{algo:BB-pull-se}). We show that for suboptimal arms, $\BBPull(\AAE)$ is approximately balanced for $\FOC$ as long as $T$ is sufficiently large, implying (with Lemma \ref{prop:relationship}) that $\APC$ strictly decreases in $f_i$.

\begin{restatable}{theorem}{Monobbpullaae}
\label{thm:feedbackmonotonicitybbpullaae}
 Fix a stochastic instance $\calI = \left\{\calA, \calF, \calL\right\}$. Let $i$
be such that $\bar{\ell}_{i} > \min_{j \in [K]} \bar{\ell}_j$. Let $\smash{\widetilde{f}_i > f_i}$, and let $\smash{\widetilde\calI = \{\calA, \widetilde\calF(i), \calL\}}$.   For sufficiently large $T$,  $\BBPull(\AAE)$ satisfies \[|\FOC_i(\calI) - \FOC_i(\widetilde\calI)| \le \nicefrac{1}{T} \hspace{4pt}\text{and}\hspace{4pt} \APC_i(\widetilde\calI) < \APC_i(\calI).\] 
\end{restatable}

Theorem \ref{thm:feedbackmonotonicitybbpullaae} strengthens the monotonicity properties of $\BBPull$: $\BBPull(\AAE)$ satisfies \textit{strict} (rather than \textit{weak}) negative monotonicity in $\APC$, and \textit{approximate balance} (rather than \textit{weak positive} monotonicity) in $\FOC$.\footnote{At first glance, it would appear 
that this result contradicts Proposition \ref{prop:impossibility}, because we show balance is possible for $\FOC$ (for suboptimal arms). However, Proposition \ref{prop:impossibility} only shows that balance is not possible across \textit{all} arms (in particular, the optimal arm necessarily exhibits positive feedback monotonicity).} The proof of Theorem \ref{thm:feedbackmonotonicitybbpullaae}  leverages a modified version of the coupling argument from the proof of Theorem \ref{thm:feedbackmonotonicitybbpull} which incorporates that any suboptimal arm must be eliminated in $\AAE$ after sufficiently many rounds. We defer a full proof to \Cref{subsec:analysisbbpullaae}. 

To achieve strictly positive feedback monotonicity in $\APC$, we turn to $\BBDivideAdjusted(\AAE)$ (formalized in Algorithm \ref{algo:BB-da-se}). The monotonicity properties of $\BBDivideAdjusted(\AAE)$ are given in \cref{thm:feedbackmonotonicitybbpullaae}.
\begin{algorithm2e}[htbp]
  \caption{$\BBDivideAdjusted(\text{AAE}, \fst)$}
  \label{algo:BB-da-se}
  \DontPrintSemicolon
  \LinesNumbered
  Maintain active set of $A$; start with $A := [K]$.
  \\
  For arm $i \in [K]$, set $B_i = \ceil{(1 + f_i) \cdot \frac{3 \ln T}{\fst}}$. \\
  Initialize phase $s=1$ and $t =1$.
  \\
  \While{$t \le T$}{
      \For{arm $i \in A$}{
          Let $R_{i,s} = \emptyset$. \\
          \For{$\min(B_i, T-t)$ iterations}
          {
              \uIf{$X_{i, t} = 1$ and $|R_{i,s}| < 8 \ln T \cdot 2^{2s}$}{
                  Append $R_{i,s} \gets R_{i,s} \cup \left\{t\right\}$.
                  }
              Pull $i_t = i$, and increment $t \gets t + 1$.
          }
      Calculate the mean $\mu_s(i) := -\frac{1}{\min(|R_{i,s}|, 2 \ln T \cdot 2^{2s})} \sum_{t' \in R_{i,s}}\ell_{i,t'}$ of the negative of the first $8 \ln T \cdot 2^{2s}$ observations (if more than $8 \ln T \cdot 2^{2s}$ observations are made).\\
      Set $\text{LCB}_s(i) = \mu_s(i) - 2^{-s}$ and 
      $\text{UCB}_s(i) = \mu_s(i) + 2^{-s}$.
      }
      For any arm $i \in A$ where $\exists j \in A$ such that $\text{LCB}_s(j) > \text{UCB}_s(i)$, remove $i$ from $A$.  \\
      Increment $s \gets s+1$.
  } 
  \end{algorithm2e}

\begin{restatable}{theorem}{BBDAAAEAPC}
\label{thm:feedbackmonotonicitybbdaaae}
Fix a stochastic instance $\calI = \left\{\calA, \calF, \calL\right\}$. Let $i$
be such that $\bell_{i} > \min_{j \in [K]} \bell_j$. Let $\widetilde f_i > f_i$, and let $\widetilde{\calI} = \left\{\calA, \calF(i), \calL\right\}$. For any $f^{\star} \le \min_i f_i$ and sufficiently large $T$, $\BBDivideAdjusted(\AAE, f^\star)$ satisfies 
\[\APC_i(\widetilde\calI) > \APC_i(\calI) \hspace{4pt}\text{and}\hspace{4pt} \FOC_i(\widetilde\calI) > \FOC_i(\calI).\]
\end{restatable}

Theorem \ref{thm:feedbackmonotonicitybbdaaae} follows from a similar argument as in the proof of Theorem \ref{thm:feedbackmonotonicitybbpullaae}: if $i$ is sub-optimal, there it must be removed at the same phase on both $\calI$ and $\widetilde\calI$.
\footnote{As for why we begin with analyzing $\APC$ instead of $\FOC$, note that in $\BBPull(\AAE)$, the number of \textit{observations} per arm per phase was predetermined; for $\BBDivideAdjusted(\AAE)$, the number of \textit{pulls} per arms per phase is predetermined.}  We defer a full proof to \Cref{appendix:bbdaaae}.

\subsection{Improving regret for adversarial losses}
\label{subsec:exp3}
Next, we consider the adversarial setting and aim for regret that scales with $\sum_{i \in [K]}\nicefrac{1}{f_i}$,
to match the stochastic result of Theorem \ref{thm:bbpullregretaaeucb}. Like in the stochastic setting, our black-box transforms fail to achieve this improved regret dependence: the regret analysis from Theorem \ref{thm:bb1-regret} for $\BBDivide(\text{EXP3})$ results in regret that unavoidably scales with $\sqrt{\nicefrac{K}{\min_i f_i}}$,\footnote{See Cor. \ref{cor:exp3bb1}. Note this is still better than scaling with $\nicefrac{K}{\min_i f_i}$, the naive implication of Theorem \ref{thm:bb1-regret}.}
because 
$\min_i f_i$
is explicitly used for determining the block size in $\BBDivide$, and our regret guarantees in \Cref{sec:blackbox} for the other two black-box transformations are restricted to stochastic losses. Moreover, directly using standard EXP3 incurs linear regret (Proposition \ref{prop:linearregret}). 

We move beyond the black-box framework and construct 3-Phase EXP3 (\Cref{algo:EXP3-3phase}), an algorithm that achieves improved regret bounds that scale with $\sum_{i \in [K]}\nicefrac{1}{f_i}$. These regret bounds for the adversarial setting match the instance-independent regret bounds that $\BBPull(\AAE)$ and $\BBPull(\UCB)$ achieve for the stochastic setting. 3-Phase EXP3 directly modifies EXP3 to account for probabilistic feedback: 
in particular, 3-Phase EXP3 obtains both \textit{unbiased} and \textit{high-probability} estimates of $\nicefrac{1}{f_i}$, then runs a version of standard EXP3 with a reward estimator and learning rate that uses those estimates. However, despite this simple structure, we empirically show that \Cref{algo:EXP3-3phase} does not seem to permit clean monotonicity properties.

\begin{algorithm2e}[htbp]
\caption{\textsc{3-Phase EXP3}}
\label{algo:EXP3-3phase}
\DontPrintSemicolon
\SetAlgoNoLine
{\bf Phase 1:}
Set $N = \ceil{8 \log (TK)}$.\;
\For{arms $i \in [K]$}{ 
    Pull arm $i$ until a reward is observed $N$ times.\;
    Set $P^{LR}_i$ to be the total number of rounds taken by the previous step divided by $N$.
}
{\bf Phase 2:} 
\For{arms $i \in [K]$}{ 
    Pull arm $i$ until a reward is observed.\;
    Set $P^E_i$ to be the number of rounds taken by the previous step. 
}
Let $t_0$ indicate the current round (after the completion of phase 1 and 2).\;
Let $\pi_{i, t_0}=1/K$ for all $i\in [K]$.\;
{\bf Phase 3:} 
Set $\eta = \sqrt{\frac{\log K}{T \sum_{i \in [K]} P^{LR}_i}}$.\;
\For{rounds $t = t_0, \dots, T$}{
    Pull an arm $i_t$  
    with probability $\pi_{i_t, t}$. \;
    Update estimator: $\widehat{\ell}_{i,t} = \frac{\ell_{i,t} \cdot X_{i, t} }{\pi_{i,t}}P^E_i, \forall i \in [K]$. \;
    Update weights: $w_{i, t+1} = w_{i,t} \cdot \exp (-\eta \widehat{\ell}_{i,t} ), \forall i \in [K]$.\;
    Update probability distribution: $\pi_{i,t+1} = 
    \frac{w_{i,t+1}}{\sum_{j \in [K]} w_{j, t+1}}$, $\forall i \in [K]$.
    }
\end{algorithm2e}

\xhdr{Regret of 3-phase EXP3.}
We prove the following regret bound for 3-Phase EXP3. 
\begin{theorem}\label{thm:EXP3-3phase-regret}
Let $\mathcal{I} = \left\{\calA, \calF, \calI\right\}$ be an adversarial instance such that $\ell_{i,t} \in [0,1]$ for all arms $i$ and all time steps $1 \le t \le T$.
For an oblivious adversary and unknown $f_i$ values, \Cref{algo:EXP3-3phase} incurs regret \[R(T) \leq O \left( \sqrt{T \ln (K) \sum_{i \in [K]} \nicefrac{1}{f_i}} \right).\]
\end{theorem}
The intuition for \Cref{thm:EXP3-3phase-regret} follows. If $f_i$'s were known, a natural way to create an unbiased loss estimator would have been $\hat{\ell}_{i,t} = (\nicefrac{\ell_{i,t}}{\pi_{i,t}}) \cdot (\nicefrac{X_{i,t}}{f_i})$. The first two phases of the algorithm adjust the algorithm to account for \textit{unknown} $f_i$.
In particular, $P_i^E$ is a low-variance unbiased estimator of $1/f_i$ and $P_i^{LR} = \Theta(1/f_i)$ with high probability. With these estimates, we adjust the second moment analysis of EXP3 while incurring only a constant overhead in the regret. We defer the full proof to \Cref{appendix:regret3phaseexp3}.

The regret bound in Theorem \ref{thm:EXP3-3phase-regret}  outperforms the $O(\sqrt{T K \log K/ \min_i f_i})$ bound achieved by $\BBDivide$ applied to EXP3 (Corollary \ref{cor:exp3bb1}). In particular, when $\min_i f_i$ is much smaller than other values of $f_i$, the regret bound in Theorem \ref{thm:EXP3-3phase-regret} can be up to a factor of $K$ better than the regret bound in Corollary \ref{cor:exp3bb1}. Theorem \ref{thm:EXP3-3phase-regret} also matches (up to $\log K$) the instance-independent regret bound achieved for stochastic losses by Algorithm \ref{algo:BB-pull-se} (Theorem \ref{thm:bbpullregretaaeucb}). 

The regret from Theorem \ref{thm:EXP3-3phase-regret} also outperforms regret bounds from existing work on multi-armed bandits with probabilistic feedback (e.g. \cite{esposito2022learning}). In particular, the feedback structure in our setting corresponds to a simple feedback graph consisting of a union of $K$ self-loops (one for each arm) with probability $f_i$ associated with the self-loop for arm $i$. \cite{esposito2022learning} show a regret bound of $\tilde{O}(\sqrt{T K / \min_i f_i})$ (with some additional optimizations when $\min_i f_i$ is very small). Their algorithm is very similar to $\BBDivide$ applied to EXP3 and uses a similar approach of splitting the time horizon into blocks. \cite{esposito2022learning} provide an algorithm that achieves a regret bound of $\widetilde{O}(\sqrt{T \cdot \sum_{i \in [K]} 1/f_i })$, but only under the additional assumption that the full feedback graph is observed at every round. This assumption is not satisfied in our setting.
In comparison, our bound achieves a much more fine-grained dependence on the feedback probabilities $f_i$. 

\xhdr{Monotonicity of 3-Phase EXP3.}
However, it seems that the improved regret for Algorithm \ref{algo:EXP3-3phase} comes at the cost of clean monotonicity properties in $\FOC$ and $\APC$. In fact, we can construct two simple instances that exhibit significantly different feedback monotonicities with respect to $\APC$, even for a simplified version of 3-Phase EXP3 where the algorithm is given the $f_i$'s as inputs rather than estimating them in the first two phases. \Cref{fig:exp3mono} shows that the algorithm exhibits strictly positive monotonicity in one instance and strictly negative monotonicity in the other instance. 
\begin{figure}[h]
\centering
\includegraphics[width=0.45\linewidth]{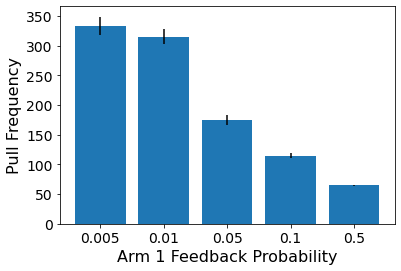} 
\includegraphics[width=0.45\linewidth]{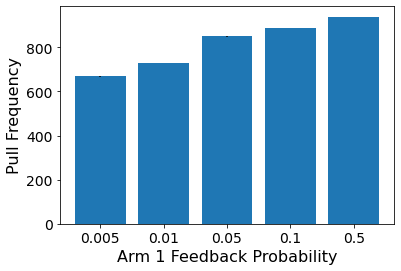}
\caption{Analysis of $\APC$ for a simplified version of 3-Phase EXP3 (Algorithm \ref{algo:EXP3-3phase}) in two instances where $K=2$ and $T=1000$. In Instance 1 (left), Arm 1  has constant loss $0.9$ and Arm 2 has constant loss $0.1$; In Instance 2 (right), Arm 1 has constant loss $0.1$ and Arm 2 has constant loss $0.9$. $\APC$ is strictly negative monotonic in Instance 1 and strictly positive in Instance 2. These differing directions of monotonicity suggest that Algorithm \ref{algo:EXP3-3phase} does not exhibit clean monotonicity guarantees. }
\label{fig:exp3mono}
\end{figure}

Further, the instances in \Cref{fig:exp3mono} differ only in their loss functions, suggesting that the monotonicity properties of 3-phase EXP3 may depend on the instance through the loss functions. Examining 3-Phase EXP3 further, we can see that such a dependency could arise due to its loss estimator directly incorporating estimates for $f_i$. In particular, unlike algorithms generated with the black-box reductions in Section \ref{sec:blackbox}, 3-phase EXP3 does not permit a clean separation between how it reacts to probabilistic feedback and how it incorporates loss observations to make future decisions. The entanglement of these decisions may make monotonicity unavoidably instance-dependent.

\section{Beyond Monotonicity: An Empirical Study of Correlations}
\label{sec:simulations}

To better understand and control the downstream impacts of the relationship between feedback and \APC/\FOC, our theoretical analysis focuses on the \textit{monotonicity} properties of these relationships, not just correlation. Monotonic dependence shifts the state of the entire system, rather than only in certain pockets of the content landscape, and so it is a stronger property to study. However, the weaker notion of \textit{correlations} may also be of concern; here, we initiate a numerical exploration of \textit{correlation} induced by bandit algorithms between $f_i$ and $\APC$/$\FOC$.

Fig. \ref{fig:correlations-combined} shows the correlation between either measure and the $f_i$'s in a \textit{single} instance. In this example, by inspection appears that $\APC_i$ is weakly negatively correlated with $f_i$ across algorithms, and $\FOC_i$ is somewhat more strongly positively correlated with $f_i$ across algorithms. 
Furthermore, these trends hold consistently across randomly generated instances.
We give experimental details below.

\begin{figure}[t!]
    \centering
    \begin{subfigure}[b]{0.32\textwidth}
      \includegraphics[width=1\textwidth]{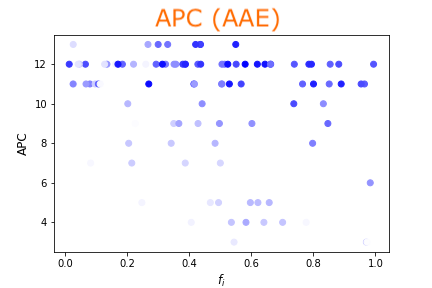}
    \end{subfigure}
    \begin{subfigure}[b]{0.32\textwidth}
      \includegraphics[width=1\textwidth]{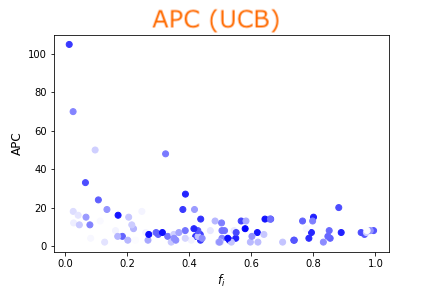}
    \end{subfigure}
      \begin{subfigure}[b]{0.32\textwidth}
      \includegraphics[width=1\textwidth]{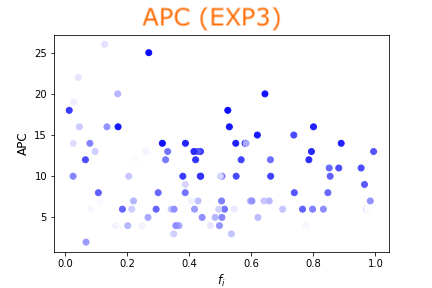}
    \end{subfigure}
    \\
      \begin{subfigure}[b]{0.32\textwidth}
      \includegraphics[width=1\textwidth]{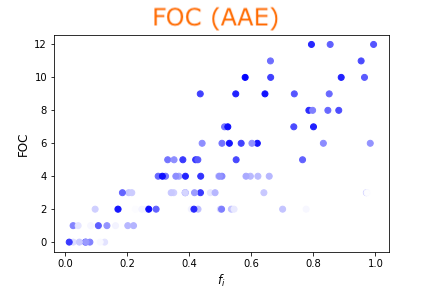}
    \end{subfigure}
    \begin{subfigure}[b]{0.32\textwidth}
      \includegraphics[width=1\textwidth]{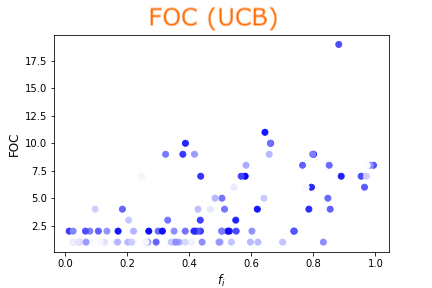}
    \end{subfigure}
      \begin{subfigure}[b]{0.32\textwidth}
      \includegraphics[width=1\textwidth]{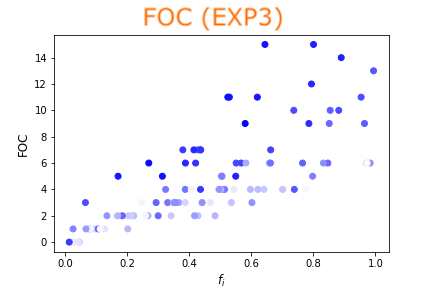}
    \end{subfigure}
    \caption{Correlations induced between $f_i$ and  $\APC_i$ (top row) as well as $\FOC_i$ (bottom row) by $\BBPull(\AAE)$ (left column), $\BBPull(\UCB)$ (middle), and 3-Phase EXP3 (right). 
    There are $K = 100$ arms and $T = 1000$ rounds. The darkness of a point indicates the corresponding arm's average utility; darker is higher.}
    \label{fig:correlations-combined}
  \end{figure}

\xhdr{Algorithms.} On each instance, we run three of the algorithms from \Cref{sec:beyond-bb}: $\BBPull(\UCB)$ (Algorithm \ref{algo:BB-pull-ucb}), $\BBPull(\AAE)$ (\Cref{algo:BB-pull-se}), and 3-Phase EXP3 (\Cref{algo:EXP3-3phase}) with the simplification that the algorithm is given the $f_i$'s as inputs (rather than estimating them in the first two phases).

\xhdr{Instance generation methods.} All of our instances have $K = 100$ arms and $T= 1000$ rounds. We first uniformly randomly generate the means of arms' utility / loss distributions (utilities for $\BBPull(\UCB)$ and $\BBPull(\AAE)$, losses for 3-Phase EXP3). These means range from 0 to 1 for $\BBPull(\UCB)$, 0 to 5 for $\BBPull(\AAE)$  \ \footnote{The larger magnitude for \AAE \ ensures that arms are actually eliminated in the time horizon $T$ we have chosen to be standard across algorithms.}, and -1 to 0 for 3-Phase EXP3. Then, for each arm, we sample realized rewards for each time step in $ [T]$ from a Gaussian distribution centered at that arm's mean with standard deviation 0.1 for $\BBPull(\UCB)$ and 3-Phase EXP3 and 0.5 for $\BBPull(\AAE)$ (commensurate with the scaled up mean). Negative utilities / positive losses are truncated to 0. Finally, we uniformly draw each arm's $f_i$ from the interval $[0,1]$.

\xhdr{Results.} For illustrative purposes, the scatter plots below show the correlation between either measure and the $f_i$'s over arms in a \textit{single} random instance. In this single example, by inspection it looks as if $\APC_i$ is weakly negatively correlated with $f_i$ across algorithms, and $\FOC_i$ is somewhat more strongly positively correlated with $f_i$ across algorithms. 

We show these trends to hold consistently across randomly generated instances: we randomly generate $100$ instances by the same procedure as above, and evaluate the Pearson correlation coefficient\footnote{The Pearson correlation coefficient is the ratio of two variables' covariance and the product of their standard deviations. It ranges from $-1$ to $1$, where positive (resp. negative) values indicate positive (resp. negative) correlation, and magnitude indicates the strength of correlation.} between $\APC_i$ and $f_i$, $\FOC_i$ and $f_i$ for each of the three algorithms on every instance. In Table \ref{table:correlations} below, we report the average Pearson correlation coefficients across these instances, which are consistent with the inspected trends in our scatter plots.
\begin{table}[h!]
\begin{center}
\begin{tabular}{ r|c c c | c c c } 
     & \multicolumn{3}{c}{APC$_i$ and $f_i$} &  \multicolumn{3}{c}{FOC$_i$ and $f_i$} \\ 
     & mean & min & max & mean & min & max\\
     \hline
    $\BBPull(\UCB)$ & $-0.33$ & $-0.51$ & $-0.16$ & $0.43$ & $0.22$ & $0.59$ \\ 
3-Phase EXP3 & $-0.23$ & $-0.41$ & $0.03$ & $0.72$  & $0.49$ & $0.86$ \\ 
$\BBPull(\AAE)$ & $-0.33$ & $-0.53$ & $-0.11$ & $0.74$ & $0.54$ & $0.88$ 
\end{tabular}
\end{center}
\caption{Correlations between $f_i$ and \APC/\FOC observed in 100 randomly generated instances.}
\label{table:correlations}
\end{table}

\section{Discussion}
\label{sec:discussion}

In this work, we illustrate how the learning algorithm can inadvertently lead to downstream impacts on users even when the objective is perfectly aligned with user welfare. In particular, we show that the ways in which the algorithm handles heterogeneous rates of user reaction across different types of content can inadvertently impact the user experience. To study this, we provide a framework to investigate how the learning algorithm's engagement with individual arms depends on the feedback rates of the arms. We analyze the monotonicity of the arm pull count $\APC$ and the feedback observation count $\FOC$ in the feedback rates across the space of no-regret algorithms. From a platform design perspective, our results highlight the importance of measuring the feedback monotonicity of a learning algorithm as well as the resulting downstream impacts on users.

To achieve some of these monotonicities, many of our algorithms require discarding information. This creates an interesting parallel to the literature on robust bandits: While it is common to discard bandit observations that are produced by adversarial or non-myopic agents (e.g.
\citet{lykouris2018stochastic, haghtalab2022learning,gupta2019better}), our work discards information not because the information is untrustworthy but because we aim to avoid undesirable downstream impact.
On the other hand, in the probabilistic feedback setting, information is already hard to come by; intuitively, we may want to do better with the information we do have access to. 
An interesting open question, therefore, is about whether it might be possible to interpolate between monotonicity properties and how efficiently information is used,
and whether such an approach can also improve regret.

Ultimately, which feedback monotonicities the platform may hope to induce  will inherently depend on which downstream effects are desirable or concerning. This may differ across areas of the content space: for example, among content that tends to generate constructive discussion, we may want \textit{positive} feedback monotonicity in \FOC, as this would elicit more beneficial user interaction.
Meanwhile, among the kinds of content described
in Examples \ref{ex:owngroup-apc} and \ref{ex:incendiary-foc}, \textit{negative} monotonicity or balance may be preferred. As the algorithms we give in this work affect monotonicity over \textit{all} content, more practical future directions may explore finer-grained control over monotonicity in different subsets of the content space.

Finally, though 
our theoretical analysis focuses on \textit{monotonicity},
in real-world settings, more general \textit{correlations} between feedback and $\APC$ or $\FOC$ may also be of concern to the platform. In \cref{sec:simulations}, we give a simulation study of correlations induced by common bandit algorithms. Combined with our rigorous monotonicity results, these simulations provide a bridge towards better understanding how probabilistic feedback can shape the impacts of a learning algorithm on users. 

\newpage
\bibliographystyle{plainnat}
\bibliography{refs}

\newpage
\appendix
\section{Additional Motivating Examples}
\label{appendix:examples}

Here, we provide additional elaboration for why varied feedback rates may cause downstream impacts on users, motivating our study of $\APC$ and $\FOC$.

\begin{example}[Clickbait and $\APC$]
\label{ex:clickbait-apc}
Observable feedback often occurs in the form of “clicks” or “likes/dislikes;” a high feedback rate is correlated with how “clickbaity” the content is. Creators may be incentivized to optimize for $\APC$, which captures objectives like view count. 
A creator can easily increase clickbaitiness (e.g. by changing content title or video thumbnail) without affecting the content's true utility.  
Thus, if the algorithm induces positive correlations between $\APC$ and feedback rate, creators may seek to make cosmetic changes without necessarily improving content quality. In the absence of positive correlation between $\APC$ and feedback rates, creators would be unable to rely on cheap strategies to generate engagement; instead, they would need to actually improve the quality of their content in order to increase the likelihood that it is shown to users.
\end{example}

More generally, even absent creator incentives, we can see that different relationships between feedback rate and $\APC$/$\FOC$ can result in qualitative differences in user experience. 

\begin{example}[Recommended topics] 
\label{example:balance}
    Certain content topics, such as political commentary or news, may naturally correspond to higher feedback rates than other topics, such as scientific or educational material. If the algorithm induces positive correlations between APC and feedback rate, then this would result in more political content shown to users; if it induces negative correlations, more educational content would be shown to users. Both have significant consequences for the overall qualitative user experience on the platform. (Of course it is possible for the platform to manually up- or down-weight content of various topics to control the content balance shown to users; however, we are interested in understanding possibly-unexpected changes that arise as a consequence of the learning algorithm itself.) Adding the possibility of creator incentives in this setting only amplifies these effects. 
\end{example}

Finally, we would like to highlight that understanding how algorithms behave with respect to APC and FOC under probabilistic feedback settings is of general interest in many applications where bandits are used to model sequential decisionmaking settings, even beyond content recommendation in online platforms. 

\begin{example}[Advertising]
In online advertising, retailers trying to place ads via a centralized platform (such as Google) can decide whether to pay the platform per-click, or per-conversion. We can think of the number of times an ad is shown as $\APC$, and the number of times an ad is clicked as $\FOC$. If the retailers choose pay-per-conversion, the resulting data provided to the platform can be viewed as having a lower feedback probability than the data that would have been provided for pay-per-click: this is because a conversion only happens a subset of the time that a click happens. Retailers may want to maximize the number of times that their ad is shown, which is captured by $\APC$. Whether an algorithm induces positive or negative correlations with $\APC$ and feedback rate could affect which of these payment models advertisers decide to select.     
\end{example}

\begin{example}[Audits and public policy]
    Because it is costly to ensure that every single person or organization complies exactly with established standards or laws, governments often instead prefer to conduct audits, where some people or organizations are selected for an audit.\footnote{See \cite{henderson2022beyond} for a more extensive discussion of bandits applied in similar contexts.} In this model, we can think of an arm as the person or organization to be audited; the pull of an arm as an audit; and feedback observation as whether the government will be able to get the ground truth “yes/no” for whether the law was violated. Why might some arms have lower or higher feedback probabilities? There may be some other reasons/features of the arm that affect feedback probability: for example, it may be harder to observe feedback for small businesses (vs bigger ones that have structured accounting departments), or non-English-speaking businesses. This reasoning also gives intuition for why it may be undesirable to pull low or high feedback arms more often (i.e. why monotonicity in $\APC$ may be problematic) --- for example, perhaps this means that in the long run, minority-owned businesses or smaller companies are audited disproportionately more than larger ones.
\end{example}
\section{Supplemental Materials for Section \ref{sec:model}}\label{appendix:modelproofs}

We prove \Cref{prop:relationship}, restated below for convenience.
\lemrelationship*
\begin{proof}[Proof of Lemma~\ref{prop:relationship}]
For an arm $i \in [K]$, recall that we have defined $X_{i,t}$ to be the random variable corresponding to whether feedback is returned at round $t$, if arm $i$ were to be pulled, i.e., $X_{i,t} \sim \texttt{Bern}(f_i)$, and $H_t$ to be the history of algorithm $\ALG$ until round $t$: $H_t = \{(i_\tau, \ell_{i_\tau, \tau}\cdot \1 \{ X_{i_\tau, \tau}\}, X_{i_\tau,\tau})\}_{\tau \in [t-1]}$. Then, for any arm $i \in [K]$, by the definition of $\FOC_i$ we have that:
\begin{align*}
  \FOC_i(\calI) &= \E_{H_t, \ell_{i_t, t}, X_{i_t,t}} \left[ \sum_{t \in [T]} \1 \left[ i_t = i\right] \cdot X_{i_t,t} \right] &\tag{\Cref{measure:observe}}\\
  &= \E_{H_t} \left[ \E_{\ell_{i_t, t}, X_{i_t,t}} \left[\sum_{t \in [T]} \1 \left[ i_t = i\right] \cdot X_{i_t,t} \Big| H_t \right] \right] &\tag{law of total expectation} \\ 
  &= f_i \cdot \E_{\ell_{i_t, t}, X_{i_t,t}} \left[ \sum_{t \in [T]} \1 \left[ i_t = i \right]\right]
  &\tag{note that $\E[X_{i,t} \mid H_t] = \E[X_{i_t,t}] = f_i$}\\
  \\
  &= f_i \cdot \APC_i(\calI)
\end{align*}
    where the third equality is due to the fact that conditioning on $H_t$, the arm drawn by $\ALG$ $i_t$ is independent of whether $X_{i_t,t}$ is $0$ or $1$.
\end{proof}

We prove \cref*{prop:impossibility}, restated below for convenience.
\propimpossibility*
\begin{proof}[Proof of Proposition \ref{prop:impossibility}]

  By  \cref{prop:relationship}, we have that 
  $\FOC_i(\widetilde\calI) - \FOC_i({\calI}) = \tilde{f}_i * \APC_i(\widetilde{\calI}) - f_i * \APC_i(\calI) $, and since $f_i < \widetilde f_i$, we can write $f_i = c \cdot \widetilde f_i$ for some $c < 1$.

  Let $i$ be the optimal arm. A no-regret algorithm will pull $i$ $T - o(T)$ times on both $\calI$ and $\widetilde{\calI}$, so that for a fixed $T$, there exists some $\alpha > 0$ where $\APC_i(\widetilde\calI) > T \cdot c^\alpha$.

  Now, we have 
  \begin{align*}
    \tilde{f}_i * \APC_i(\widetilde{\calI}) - f_i * \APC_i(\calI) &> \tilde f_i \cdot c^\alpha \cdot T - f_i \cdot \APC_i(\calI)
    \\&\geq \tilde f_i \cdot c^\alpha \cdot T - f_i \cdot T
    \\&= T \cdot \widetilde f_i \cdot (c^{\alpha} - c).
  \end{align*}
It remains to show that \begin{align*}
  T \cdot \widetilde f_i \cdot (c^{\alpha} - c) > \frac{9}{10}\cdot T (\widetilde f_i - f_i) &= \frac{9}{10} \cdot T \cdot \widetilde f_i \cdot (1-c) \\\iff
  c^{\alpha} - c &> \frac{9}{10}-\frac{9}{10}c
  \\\iff c^\alpha - \frac{c}{10} &> 9/10.
\end{align*}
Taking $T$ to be sufficiently large guarantees that $\alpha$ is sufficiently small for the above inequality to hold.

\end{proof}

\section{Supplemental Materials for Section \ref{sec:blackbox}}
\label{appendix:blackboxproofs}

In this section, we provide proofs of regret and montonicity guarantees for our black-box transformations.

\subsection{Proofs for Section \ref{subsec:bb1}: $\BBDivide$}
\label{appendix:bbdivide}

Recall that $\BBDivide$, formalized in \cref*{algo:BB-divide}, divides time horizon into equally-sized blocks of size $B = \nicefrac{3\ln T}{f^\star}$. We analyze its regret and monotonicity properties below. 

\subsubsection{Corollaries of \Cref{thm:bb3-regret}}\label{appendix:corollaries}

Finally, we present several corollaries which give the regret of $\BBDivide$ applied to standard bandit algorithms.
\begin{corollary}
\label{cor:exp3bb1}
\textcolor{black}{For fixed $\fst \in (0,\min_i f_i]$}, transformation $\BBDivide$ applied to standard EXP3 incurs the following regret: 
\[
R_{\BBDivide(\text{EXP3},\fst))}(T) \leq O\left(\sqrt{\frac{1}{\fst} \cdot TK \ln (T)\ln (K)}\right).
\]
This follows from \Cref{thm:bb1-regret}, along with the known result of~\cite{auer2002nonstochastic} that the regret for standard EXP3 is 
$
R_{\text{EXP3}}(T) \leq O(\sqrt{TK \ln K})
$.
\end{corollary}

\begin{corollary}
\label{cor:ucbbb1}
\textcolor{black}{For fixed $\fst \in (0,\min_i f_i]$}, transformation $\BBDivide$ applied to standard UCB incurs the following regret: \[
R_{\BBDivide(\text{UCB},\fst))}(T) \leq O\left(\sum_{i \in [K]} \frac{\ln^2(T)}{\Delta_i
 \cdot \fst}\right).
\]
This follows from \Cref{thm:bb1-regret}, along the known result of~\cite{auer2002finite} that the instance-dependent regret for standard UCB is $R_{UCB}(T) \leq O\left(\ln T \cdot \left(
\sum_{i \in [K]} \frac{1}{\Delta_i}
\right)\right)$, where $\Delta_i = \bell_i - \min_j\bell_j$.
\end{corollary}

\begin{corollary}
\label{cor:sebb1}
\textcolor{black}{For fixed $\fst \in (0,\min_i f_i]$}, transformation $\BBDivide$ applied to standard AAE incurs the following regret:
\[
R_{\BBDivide(\text{AAE},\fst))}(T) \leq O\left(\sum_{i \in [K]} \frac{\ln^2(T)}{\Delta_i
\cdot \fst}\right).
\]
This follows from \Cref{thm:bb1-regret}, along the known result by~\cite{even2002pac} that the instance-dependent regret for standard AAE is $R_{AAE}(T) \leq O\left(\ln T \cdot \left(
\sum_{i \in [K]} \frac{1}{\Delta_i}
\right)\right)$, where $\Delta_i = \bell_i - \min_j\bell_j$.
\end{corollary}

\subsubsection{Regret of $\BBDivide$: Proof of \cref{thm:bb1-regret}}

We prove \cref{thm:bb1-regret} and give some applications to concrete algorithms. For convenience, we restate the regret bound of $\BBDivide$ below.
\Rbbdivide*

To analyze the regret of $\BBDivide$ applied to a generic algorithm $\ALG$, we will use the following lemma, which lower bounds the likelihood of seeing a sample from the true loss distribution in every block.

\begin{lemma}
\label{lemma:clean}
Fix an $\fst \in (0, \min_if_i]$, and divide the time horizon $T$ into blocks of size $B = \frac{3 \ln T}{\fst}$ and let $\Phi = \floor{T/B}$, as in \cref{algo:BB-divide}. Suppose then that for each block $\phi \in \{1,2,\dots,\Phi$\}, we play the same arm $i_\phi$ for every round in block $\phi$. 
Let $E$ be the ``\emph{clean event}'' that at least one feedback observation occurs in \emph{each} block $\phi$, i.e., that for all blocks $\phi$, $\exists t \in S_\phi: X_{i_t,t} = 1$. Then, $\Pr[E] \ge 1 - 1 / T^2$. 
\end{lemma}
\begin{proof}

Let $E_\phi$ be the event that at least one feedback observation occurred in block $\phi$, i.e., $\exists t \in S_\phi: X_{i_t,t} = 1$. Since for any arm $i$, $\Pr[X_{i,t} = 1] = f_{i}$, then for arm $i_\phi$, we have that
\[\Pr[\neg E_\phi] = (1-f_{i_\phi})^B \leq (1-\fst)^B \leq \exp(-\fst B) = 1/T^3.
\]
Union bounding over all $\floor{T/B}$ blocks, we conclude that
\[ \Pr[\neg E] \leq \sum_{\phi \in [\Phi]} \Pr[\neg E_\phi] \leq 1/T^2. \qedhere
\]

\end{proof}

We are now ready to prove \cref{thm:bb1-regret}.
\begin{proof}[Proof of \cref{thm:bb1-regret}: Regret $\BBDivide$]

Throughout the proof we will use $\fst \in (0,\min_i f_i]$.\\ 

First, observe that Line 9 of Algorithm~\ref{algo:BB-divide} (i.e., the last $T - B \Phi$ steps of the time horizon) contributes $O(\ln T)$ regret, because $T - B\Phi < T - B \frac{T}{B} + B \leq B = \frac{3 \ln T}{f^{\star}}$. The rest of the proof thus analyzes the regret incurred in the first $B \Phi$ time steps.
    Now, we divide up these rounds into $\Phi$ blocks of size $B$, and let $E$ be the ``clean event'' that at least one feedback observation occurs in each block $\phi \in \{1, \dots, \Phi\}$. By Lemma \ref{lemma:clean}, we have that $\Pr[E] \ge 1 - \frac{1}{T^2}$. The event that $E$ does not occur contributes at most $O(1)$ to the expected regret, so we can condition on $E$ for the remainder of the analysis.\\ 

Next, fix an instance with stochastic feedback $\mathcal{I} = \{\mathcal{A},\mathcal{F}, \mathcal{L}\}$ over $T$ rounds. Now, we define corresponding instance with deterministic feedback $\mathcal{I'} = \{\mathcal{A}, (1,\dots,1), \mathcal{L}'\}$ over $\Phi = \floor{T/B}$ time steps, where $\mathcal{L}'$  denotes the process generating the following sequence of $\Phi$ losses $\ell_{i,1}',\dots,\ell_{i,\Phi}'$
for all $i \in \mathcal{A}$: For all $i \in \calA$ and $\phi \in \{1, \dots,\Phi\}$,
$\ell'_{i,\phi}\sim \mathrm{Unif}\left\{\ell_{i, s} : s \in S_\phi\right\}$, i.e., the loss is sampled uniformly from the loss functions of the original instance within block $\phi$.
    Now, we show that the (pseudo)regret of $\ALG$ on instance $\mathcal{I}'$ over $\Phi$ rounds is the same as that of $\BBDivide(\ALG,\fst)$ on instance $\mathcal{I}$ over $T$ rounds. By the definition of the regret, the regret of $\ALG$ on the instance $\calI'$ is equal to
    \begin{equation}
        \E \left[\sum_{\phi=1}^{\Phi} \ell'_{i_\phi, \phi} \right] - \min_{i} \E\left[\sum_{\phi=1}^{\Phi} \ell'_{i, \phi}\right], \label{eq:regALG}
    \end{equation}
    where the randomness of the first expectation is due to the randomness of the algorithm $\ALG$ and $\mathcal{L}'$.
    The regret of $\BBDivide(\ALG)$ on $\mathcal{I}$ is equal to: 
\begin{align*}
\E\left[\sum_{t=1}^T \ell_{i_t, t}\right] - \min_{i} \E\left[\sum_{t=1}^{T} \ell_{i, t} \right]
&= \E\left[\sum_{\phi=1}^{\Phi} \sum_{t \in S_\phi} \ell_{i_\phi, t}\right] -  \min_{i} \sum_{\phi=1}^{\Phi} \sum_{t\in S_\phi} \E[\ell_{i, t}] \\
&= B \cdot \left(\E\left[\sum_{\phi=1}^{\Phi} \underbrace{\frac{1}{B} \sum_{t\in S_\phi}\ell_{i_\phi, t}}_{(A)} \right] -  \min_{i} \sum_{\phi=1}^{\Phi} \underbrace{\frac{1}{B}  \sum_{t\in S_\phi}\E[\ell_{i, t}]}_{(B)}\right)
\end{align*}
where randomness in the expectation is due to the randomness of $\BBDivide(\ALG)$, the randomness of feedback observations, and the randomness of the loss functions. Notice that (A) is equal to the  $\mathbb{E}[\ell'_{i_\phi, \phi}]$ and (B) is equal to $\mathbb{E}[\ell'_{i, \phi}]$. Moreover, 
the observation at the end of the $\phi$th block, which is passed as the loss to $\ALG$ for its $\phi$th timestep, is also distributed according to $\text{Unif}\left\{\ell_{i_\phi, t} : t \in S_\phi\right\}$. 

We thus see that: 
\begin{align*}
\E\left[\sum_{t=1}^T \ell_{i_t, t}\right] - \min_{i} \E\left[\sum_{t=1}^{T} \ell_{i, t}\right] &= B \cdot \left(\E\left[\sum_{\phi=1}^{\Phi} \E[\ell'_{i_\phi, \phi}] \right] -  \min_{i} \sum_{\phi=1}^{\Phi} \E[\ell'_{i, \phi}] \right) \\
&= B \cdot \left(\E_{\mathcal{L}'}\left[\sum_{\phi=1}^{\Phi} \ell'_{i_\phi, \phi} \right] - \min_{i} \E_{\mathcal{L}'}\left[ \sum_{\phi = 1}^{\Phi} \ell'_{i, \phi} \right] \right).
\end{align*}

This expression corresponds exactly to $B$ times the regret of $\ALG$ on $\mathcal{I}'$ (see \eqref{eq:regALG}), and thus can be upper bounded by 
\begin{align*}
B \cdot \left(\mathbb{E}_{\mathcal{L}'}\left[\sum_{\phi=1}^{\Phi} \ell'_{i_\phi, \phi} \right] - \min_{i} \E_{\mathcal{L}'}\left[\sum_{\phi=1}^{\Phi} \ell'_{i, \phi} \right] \right)  &\le B \cdot R_{\ALG}(\Phi) \\
&\le B \cdot R_{\ALG}(T/B) \\
&\le \frac{3 \ln T}{f^{\star}} \cdot R_{\ALG}\left(\frac{T f^{\star}}{3 \ln T} \right).
\end{align*}

The $\frac{3 \ln T}{f^{\star}}$ error term from the last $T - B \Phi$ steps can be absorbed into this regret term, since $R_{\ALG}\left(\frac{T f^{\star}}{3 \ln T} \right) \ge 1$.

\end{proof}

\subsubsection{Monotonicity of $\BBDivide$: Proof of \cref{thm:feedbackmonotonicitybbdivide}}

Next, we analyze the monotonicity properties of $\BBDivide$. For convenience, we restate \cref{thm:feedbackmonotonicitybbdivide} below.

\Fbbdividemono*

The intuition for the $\APC$ statement is that since $\BBDivide$ effectively treats each block as one round of $\ALG$, equalizing the block sizes will naturally balance $\APC$. Once $\ALG$ decides to pull an arm $i$, $\BBDivide(\ALG)$ will pull it $B$ times regardless of its feedback probability. This result relies on $\fst$ being sufficiently small to ensure that there is an observation in every block. The $\FOC$ statement follows from an application of \cref{prop:relationship}. We formalize this below.

\begin{proof}[Proof of \cref*{thm:feedbackmonotonicitybbdivide}]

We first analyze $\APC$. Let $E$ be the ``clean event'' that at least one feedback observation occurs in each block $1 \le \phi \le \Phi$. By Lemma \ref{lemma:clean}, we know that $\mathbb{P}[E] \ge 1 - \frac{1}{T^2}$. Conditioning on the clean event $E$, we see that $\APC_i(\calI) = \APC_i(\tilde\calI)$ by the construction of the algorithm, since in every block where $\ALG$ selects $i$, $\BBDivide(\ALG)$ will pull $i$ exactly $B$ times.
The event that $E$ does not occur contributes at most $1/T$ to $\APC$. 

We next analyze $\FOC$. From the proof of the $\APC$ statement, we have that 
$\APC_i(\tcalI) \ge \APC_i(\calI)  - 1/T.$
Applying \Cref{prop:relationship}, which states that $\FOC_i = f_i \cdot \APC_i$, we have 
\[
\FOC_i(\tilde\calI) = \tilde f_i \cdot \APC_i(\tilde\calI)\geq \tilde f_i \cdot \APC_i(\calI) - \tilde f_i/T = \frac{\tilde f_i}{f_i}\cdot \FOC_i(\calI) - \tilde f_i/T \geq \FOC_i(\calI) - \tilde f_i/T,
\] 
where the last inequality is because $\tilde f_i \geq f_i$. 

For strict inequality, notice that it suffices to show that $\frac{\tilde{f}_i}{f_i} \cdot \FOC_i(\calI)  - \frac{\tilde{f}_i}{T} > \FOC_i(\calI)$. As long as $\ALG$ pulls $i$ at least once, this will hold for sufficiently large values of $T$. 
\end{proof}

\subsection{Proofs for Section \ref{subsec:bb2}: $\BBPull$}
\label{appendix:bb2}

In this section, we provide proofs for the regret and monotonicity of algorithms transformed by $\BBPull$. Before doing so explicitly, we first introduce a simulated version of $\BBPull$, as well as Lemmas \ref*{lem:joint-losses-identical} and \ref*{lem:simulated-indistinguishable},
which will help us compare transformed algorithms on similar instances. 

\subsubsection{Construction of a simulated version of $\BBPull(\ALG)$}
\label{subsubsec:simulated}

We first introduce \cref{algo:BB-pull-simulated}, a simulated version of $\BBPull(\ALG)$, which will be easier to analyze but behaves the same way as $\BBPull(\ALG)$. 
First, let us define the following random variables. (Recall that $\phi$ indexes losses for the time horizon of $\ALG$, $\Phi$ is the total number of times $\ALG$ is called by $\BBPull(\ALG)$, and $\Phi \leq T$ because $\ALG$ can be called at most $T$ times.)
\begin{itemize}
    \item \textit{Losses:} For each round $\phi \in [\Phi]$ of $\ALG$ and each arm $j \in [K]$, $\ell'_{j, \phi}$ is the placeholder for the loss passed to $\ALG$ if $\ALG$ were to observed the loss of arm $j$ at time $\Phi.$ More formally, 
    $\ell'_{j, \phi} := \ell_{j, t}$ the loss for arm $j$ at a time step $t$ that corresponds to the last time step in block $\phi$ of $\BBPull(\ALG)$. Since we are in the stochastic loss setting, $\ell'_{j,\phi}$ is a random variable drawn from the distribution of arm $j$ (with mean $\bar{\ell}_j$) independently across $\phi$ and $j$. We note that these losses are only observed up to timestep $\Phi$ (which is a random variable less than $T$) and only for the specific arms pulled by the algorithm.

    \item \textit{Feedback realizations:} For all $j \in [K]$ and $\phi \in [T]$, let $Q_{j,\phi} \sim \textrm{Geom}(f_j)$ for $\phi \in [T]$ be a random variable distributed according to the geometric distribution with parameter equal to the feedback probability of arm $j$. 
    This will represent the number of Bernoulli trials needed to observe a success.  
    (These random variables are also fully independent across values of $j$ and $\phi$.) 

    \item \textit{Algorithm randomness:} Let $b$ be randomness of $\ALG$ that will be used across time steps $1 \le \phi \le T$. Let $\ALG_b$ denote $\ALG$ initialized with randomness $b$. 
\end{itemize}

We are now ready to present the simulated version of $\BBPull(\ALG)$, described in Algorithm \ref{algo:BB-pull-simulated}. 

\begin{algorithm2e}[htbp]
\caption{Simulated version of $\BBPull(\ALG)$}
\label{algo:BB-pull-simulated}
\DontPrintSemicolon
\LinesNumbered
\KwIn{A sequence of positive integers $Q_{j, \phi}$ for $\phi \in [T]$ and $j \in [K]$.}
Initialize $t= 1$ and $\phi = 1$. \\
\While{$t \le T$}{
    Let $i^{\ALG}_\phi = \ALG(\phi)$ be the output of $\ALG$ at timestep $\phi$. \;
    \For {$\min(T - t, Q_{i^{\ALG}_\phi, \phi})$ iterations}{
    Pull $i_t \coloneqq i^{\ALG}_\phi$ and let 
    $t \gets t +1$.
    }
    \If{$t < T$}
    {Observe $\ell_{i_t, t}$ and return $\ell'_{i^{\ALG}_\phi, \phi} \coloneqq \ell_{i_t, t}$ to $\ALG$. \\
    Let $\phi \gets \phi + 1$.}}
\end{algorithm2e}

Note that the random variables $Q_{j_\phi, \phi}$ actually now capture the \textit{block size} of the transformed algorithm $B_\phi$ (which, for $\BBPull$, is a random variable). For clarity, we will use $Q$ rather than $B$ in the remaining analyses. 

We first argue that, given two instances $\calI, \tilde \calI$ which are identical except for $\tilde f_i \geq f_i$, the sequences of arms that \cref{algo:BB-pull-simulated} pulls are distributed identically across the instances. We formalize this in the following lemma.

\begin{lemma}
    \label{lem:joint-losses-identical}
     Let $Q_{j,\phi}$ and $ \widetilde Q_{j,\phi}$ for $j \in [K]$ and $\phi = 1, \dots, $ be an infinitely-long sequence of arbitrary positive integers. Let $\Phi^*$ be any positive integer and $T = \max\{\sum_{\phi\in [\Phi^*]}\sum_{j\in[K]}Q_{j,\phi}, \allowbreak \sum_{\phi\in [\Phi^*]}\sum_{j\in[K]} \widetilde{Q}_{j,\phi} \}$ be the time horizon.
     Let $\calI = \{\calA, 
     \calF, 
     \calL\}$ be a stochastic instance with time horizon $T$; let $\tilde f_i \geq f_i$ and $\widetilde \calI = \{\calA, \tilde\calF(i), \calL\}$. 
    Run \cref{algo:BB-pull-simulated} with parameters $\left\{Q_{j, \phi}\right\}_{j \in [K], \phi \in [T]}$ on $\calI$ and run \cref{algo:BB-pull-simulated} with parameters $\left\{\widetilde{Q}_{j, \phi}\right\}_{j \in [K], \phi \in [T]}$ on $\widetilde \calI$.
    Let 
   $i^\ALG_\phi$ and $\tilde{i}^\ALG_\phi$ denote the  arms pulled in the description of \cref{algo:BB-pull-simulated} for the two instances, respectively.
      Then, the following two vector valued random variables are identically distributed: $(i^\ALG_1, \dots, i^\ALG_{\Phi^*})$ and $(\tilde{i}^\ALG_1, \dots, \tilde{i}^\ALG_{\Phi^*})$.
\end{lemma}

The intuitive interpretation of \cref{lem:joint-losses-identical} is very natural: if we have two set of arms with identical loss distributions and run $\BBPull(\ALG)$ on them, we expect to see that the sequence of arms recommended by $\ALG$ is distributed identically across the two instances, even if we can't guarantee that the exact same arm is picked at every timestep on each instance. We provide a formal proof below. 

\begin{proof}[Proof of \cref{lem:joint-losses-identical}]
Let $\{\ell'_{j, \phi}\}_{j \in [K], \phi \in [\psi]} $ denote possible loss sequences observed on $\calI$ up to some $\psi \leq \Phi^*$ and $\{\tilde\ell'_{j, \phi}\}_{j \in [K], \phi \in [\psi]} $ denote possible loss sequences observed on $\widetilde\calI$ up to the same $\psi$. 
Let us fix the bit of randomness $b$ used for $\ALG$ on $\calI$ to be the same as the bit of randomness used for $\ALG$ on $\widetilde\calI$.
Because of the way we have set $T= \max\{\sum_{\phi\in [\Phi^*]}\sum_{j\in[K]}Q_{j,\phi}, \allowbreak \sum_{\phi\in [\Phi^*]}\sum_{j\in[K]} \widetilde{Q}_{j,\phi} \}$, we are guaranteed that blocks $\phi = 1, \dots, \psi$ will have been reached on both $\widetilde \calI$ and $\calI$.
Conditioned on $b$,
let $F_b: [0,1]^{K \times \psi} \to [K]^{\psi}$ be the mapping from all $\ell'_{j,\phi}$ for $\phi \leq \psi$,
to the sequence of arms it would have pulled correspondingly, that is, 
\begin{align*}
    F_b\left(\{\ell'_{j, \phi}\}_{j \in [K], \phi \in [\psi]}\right) = (i_1^\ALG, i_2^\ALG, \dots, i_\psi^\ALG).
\end{align*}
Note that $F_b$ does not depend on the feedback probabilities $f_i$ or the random variables $Q_{i,\phi}$, because $\ALG$ is fully oblivious to these quantities. For any $b$, $F_b$ is fully deterministic. 
Therefore, the distribution of $(i_1^\ALG, i_2^\ALG, \dots, i_\psi^\ALG)$ is fully specified by the distributions of $\{\ell'_{j, \phi}\}_{j \in [K], \phi \in [\psi]}$, and the distribution of $(\tilde i_1^\ALG, \tilde i_2^\ALG, \dots, \tilde i_\psi^\ALG)$ is fully specified by the distributions of $\{\tilde \ell'_{j, \phi}\}_{j \in [K], \phi \in [\psi]}$.

Since the loss sequences are distributed identically across instances, we have that \begin{align*}
\{\ell'_{j, \phi}\}_{j \in [K], \phi \in [\psi]} &\stackrel{d}{=}  \{\tilde\ell'_{j, \phi}\}_{j \in [K], \phi \in [\psi]}    
\\\implies F_b\left(\{\ell'_{j, \phi}\}_{j \in [K], \phi \in [\psi]}\right) &\stackrel{d}{=}  F_b\left(\{\tilde\ell'_{j, \phi}\}_{j \in [K], \phi \in [\psi]} \right)
\\\implies (i_1^\ALG, i_2^\ALG, \dots, i_\psi^\ALG) &\stackrel{d}{=} (\tilde i_1^\ALG, \tilde i_2^\ALG, \dots, \tilde i_\psi^\ALG),
\end{align*}
where $\stackrel{d}{=}$ denotes identically distributed relationship. Finally, because this holds conditionally over any arbitrary $b$, we can integrate over all possible random bits $b$ to establish the claim.
\end{proof}

To use \cref{algo:BB-pull-simulated} in our proofs, we need to argue that it makes decisions that are distributed identically to those of \cref{algo:BB-pull}. We formalize this below: 

\begin{lemma}
[Distribution of arms pulled by simulated algorithm]
\label{lem:simulated-indistinguishable}
Fix an instance $\calI$.
    Let $\{i_t^{\text{orig}}\}_{t \in [T]}$ be a sequence of random variables that represents the arms selected by \cref{algo:BB-pull} on $\calI$ over the time horizon $T$, and $\{i_t^\text{sim}\}_{t \in [T]}$ be a sequence of random variables that represents the arms selected by  \cref{algo:BB-pull-simulated} on an identical instance $\calI$. Then the sequence $\{i^\text{orig}_t\}_{t \in [T]}$ is distributed identically to $\{i^\text{sim}_t\}_{t \in [T]}$. 
\end{lemma}

The key difference between \cref{algo:BB-pull-simulated} and \cref{algo:BB-pull} is that the number of times  \cref{algo:BB-pull-simulated} pulls $i_\phi^\ALG$ is determined by the random variable $Q_{i_\phi,\phi}$, rather than by the first time feedback is observed in \cref{algo:BB-pull}. However, $Q_{i_\phi,\phi}$ is distributed identically to the number of times feedback will be observed, so the simulated version should overall produce the same distribution of outputs. We formalize this intuition below. 

\begin{proof}[Proof of \cref{lem:simulated-indistinguishable}]
This proof will proceed in three main steps. First, we argue that the sequence of arms selected by $\ALG$ for either \cref{algo:BB-pull-simulated} or \cref{algo:BB-pull} are identically distributed. Second, we relate the $X_{j,t}$ used by \cref{algo:BB-pull} to the $\phi$ timescale. Third, we show by induction that feedback observations are identically distributed on  \cref{algo:BB-pull-simulated} and \cref{algo:BB-pull}. Finally, we argue that the sequences of arms selected by \cref{algo:BB-pull-simulated} and \cref{algo:BB-pull} are identically distributed.

\paragraph{Step 1: Coupling arm pulls}$i_\phi^{\ALG,\text{orig}} = i_\phi^{\ALG, \text{sim}}$.

Fix a sequence of random variables $Q_{j, \phi} \sim \text{Geom}(f_j)$ for $ \phi \in [T]$ and $j \in [K]$ used to run \cref{algo:BB-pull-simulated}. Let $T^*  = \sum_{\phi \in [T]} \max_{j \in [K]} Q_{j,\phi}$; then fix a sequence of random variables  $X_{j,t} \sim \text{Bern}(f_j)$ for $t \in [T^*]$ and $j \in [K]$ that determine feedback observations in \cref{algo:BB-pull}.

We run \cref{algo:BB-pull} and \cref{algo:BB-pull-simulated} on identical copies of $\calI$ for $T^*$ rounds; we distinguish each copy by $\calI^{\text{orig}}$ for \cref{algo:BB-pull} and $\calI^{\text{sim}}$ for \cref{algo:BB-pull-simulated}. We set $T^*$ in this way to guarantee that timestep $\phi$ will be reached on both $\calI^{\text{orig}}$  and $\calI^{\text{sim}}$; we will handle truncation in the final step.

Recall that \cref{algo:BB-pull} and \cref{algo:BB-pull-simulated} both make calls to the same underlying $\ALG$. Let $b$ be the bit of randomness used for $\ALG$ in \cref{algo:BB-pull} and \cref{algo:BB-pull-simulated}.
Now, conditioning on $b$, let $F_b: [0,1]^{K \times \psi} \to [K]^{\psi}$ be, as defined before, the
mapping from the sequence of possible losses that $\ALG$ may have observed for any arm at any time $\phi \leq \psi$, to the sequence of arms it would have pulled corresponding to those losses, that is, 
\begin{align*}
    F_b\left(\{\ell_{j, \phi}\}_{j \in [K], \phi \in [\psi]}\right) = (i_1^\ALG, i_2^\ALG, \dots, i_\psi^\ALG).
\end{align*}
Note that $F_b$ does not depend on the feedback probabilities $f_i$ or the feedback observations $Q_{i,\phi}$ or $X_{i, t}$, because $\ALG$ is fully oblivious to these quantities. For any $b$, $F_b$ is fully deterministic. 
Furthermore, the simulated and real algorithms use $\ALG$ with the same bit of randomness, so $F_b^{\text{orig}} = F_b^{\text{sim}}$, and the arms selected by $\ALG$ for either \cref{algo:BB-pull} or \cref{algo:BB-pull-simulated} are fully specified by the distributions of the losses for each arm. Then, for any $\psi \leq T$, we have that
\begin{align*}
    \{\ell^{\text{orig}}_{j, \phi}\}_{j \in [K], \phi \in [\psi]} &\stackrel{d}{=}  \{\ell^{\text{sim}}_{j, \phi}\}_{j \in [K], \phi \in [\psi]}   &&\textit{because }\calI^{\text{orig}} = \calI^{\text{sim}} 
\\\implies F_b\left(\{\ell^{\text{orig}}_{j, \phi}\}_{j \in [K], \phi \in [\psi]}\right) &\stackrel{d}{=}  F_b\left(\{\ell^{\text{sim}}_{j, \phi}\}_{j \in [K], \phi \in [\psi]} \right) &&\textit{because }F_b^{\text{orig}} = F_b^{\text{sim}}
\\\implies \{i_\phi^{\ALG, \text{orig}}\}_{\phi \in [\psi]} &\stackrel{d}{=} \{i_\phi^{\ALG, \text{sim}}\}_{\phi \in [\psi]} &&\textit{by definition of $F_b$.}
\end{align*}
We end this step by creating a coupling between arms selected by $\ALG$ so that $i_\phi^{\ALG,\text{orig}} = i_\phi^{\ALG, \text{sim}}$ for all $\phi \in [T]$. For the rest of the analysis, we condition on this particular sequence.

\paragraph{Step 2: Coupling the block lengths.}

Define $Q'_\phi: [K]^{\phi} \times \{0, 1\}^{K \times T^*} \to \mathbb N$
to be the function that maps $\{i_{\phi'}^{\ALG, \text{orig}}\}_{\phi' \leq \phi}$, the sequence of arms selected by $\ALG$ up to $\phi$, and the sequence of Bernoullis $X_{j,t}$ used to run \cref{algo:BB-pull}, to the number of times $i_\phi^{\ALG, \text{orig}}$ needs to be pulled (on $t$ timescale) before feedback is observed.\footnote{We note that while \cref{algo:BB-pull} only takes $X_{j,t}$ variables into account for arm $j$ that was pulled at time $t$, these random variables can be defined for all arms at all time steps without changing the behavior of \cref{algo:BB-pull}.}
In some abuse of notation, we let $Q'_\phi := Q'_\phi\left(\{i_{\phi'}^{\ALG, \text{orig}}\}_{\phi' \leq \phi}, \{X_{j,t}\}_{j \in [K], t \geq 1}\right)$ be shorthand for the number of times $i_\phi^{\ALG,\text{orig}}$ must be pulled until a feedback observation. Now, $Q'_\phi$ is fully determined by the history of $\ALG$ arm pulls and the sequence of $X_{j,t}$.
(Note that $Q'_\phi$ needs to depend on the $X_{j,t}$'s as well as the \textit{history} of $\ALG$'s selections. This is because even if we know $i_\phi^\ALG$, we do not know which $t$ indices of the $X_{i_\phi^\ALG,t}$ sequence determine whether we make an observation or not. We can only relate $t$ to $\phi$ correctly if we know exactly which arms were pulled in previous $\phi' < \phi$, \textit{and} their corresponding feedback observations.)

Next, we will show that for all $\phi \in [T]$, conditioned on fixing $Q'_\psi$ and $Q_{i_\psi,\psi}$ such that $Q_{i_\psi,\psi} = Q'_\psi$  for all $\psi <\phi$, 
it holds that $Q_{i_\phi,\phi} \stackrel{d}{=} Q'_\phi$.

Recall that we have fixed a coupling between arms selected by $\ALG$ so that $i_\phi^{\ALG,\text{orig}} = i_\phi^{\ALG, \text{sim}}$ for all $\phi \in [T]$. For ease of presentation, we refer to these arms as $i_\phi$ simply.

First note that for any $\phi$, $Q_{i_\phi, \phi} \sim \text{Geom}(f_{i_\phi})$ by definition; furthermore, these are independent across all $\phi$. 
To complete our claim, it suffices to show that  
$Q'_\phi \sim \text{Geom}(f_{i_\phi})$ conditioned on fixing $Q'_\psi$ and $Q_{i_\psi,\psi}$ such that $Q_{i_\psi,\psi} = Q'_\psi$  for all $\psi <\phi$.

Recall that $Q'_\phi :=
Q'_\phi\left(\{i_{\phi'}\}_{\phi' \leq \phi},
\{X_{j,t}\}_{j \in [K], t \geq 1}\right)$ 
is the shorthand for the number of times $i_\phi$ must be pulled until a feedback observation. 
Let $t_{\phi}$ be the first time steps $t$ that belongs to block $\phi$. Note that $t_{\phi}$ is a deterministic function of the fixed variables $\{Q'_\psi\}_{\psi < \phi} = \{Q_{i_\psi, \psi}\}_{\psi < \phi}$.
Furthermore,
$X_{i_\phi,t}$s for $t\geq t_\phi$ are independent of 
$\{Q'_\psi\}_{\psi < \phi}$ 
and 
$Q'_\phi$ is a function of $i_\phi$ that only depends on 
$X_{i_\phi,t}$ for $t\geq t_\phi$.
Moreover, $X_{i_\phi,t}$ for $t\geq t_\phi$ are Bernoulli random variables that are independent of $t_\phi$.
Therefore, $Q'_\phi \sim \text{Geom}(f_{i_\phi})$ conditioned on the past. That is, for all $\phi \in [T]$, it holds that $Q_{i_\phi,\phi} \stackrel{d}{=} Q'_\phi$ conditioned on fixing $Q'_\psi$ and $Q_{i_\psi,\psi}$ such that $Q_{i_\psi,\psi} = Q'_\psi$  for all $\psi <\phi$.

We end this step by taking an adaptive coupling over the realizations of $Q'_\psi$ and $Q_j,\psi$, such that 
for all $\phi\in [T]$, $Q_{i_\phi,\phi} =Q'_\phi$. 

\paragraph{Step 3: Arms selected by \cref{algo:BB-pull} and \cref{algo:BB-pull-simulated} are identically distributed.}

Finally, we are ready to prove the main claim. Let us condition on the coupled sequence of arms pulled by $\ALG$ from Step 1, $\{i_\phi^\ALG\}_{\phi \in [T]}$, and the coupled feedback observations $\left(\{Q_{i_\phi,\phi}\}_{\phi \in [T]},\{Q'_{\phi}\}_{\phi \in [T]} \right)$ from Step 2.

For \cref{algo:BB-pull}, the random variables $X_{j,t}$ and the sequence $\{i_\phi^{\ALG,\text{orig}}\}_{\phi \in [T]}$ fully specifies the arms pulled by \ref{algo:BB-pull}, i.e. the sequence $\{i_t^{\text{orig}}\}_{t \in [T]}$. For \cref{algo:BB-pull-simulated}, the random variables $X_{j,\phi}$ and the sequence $\{i_\phi^{\ALG,\text{sim}}\}_{\phi \in [T]}$ fully specifies the arms pulled by \ref{algo:BB-pull-simulated}, i.e. the sequence $\{i_t^{\text{sim}}\}_{t \in [T]}$. 
Conditioned on this coupling $\left(\{Q_{j,\phi}\}_{\phi \in [T]}, \{X_{j,t}\}_{t \geq 1}\right)$, therefore, we have that 
\begin{align*}
    \{i_\phi^{\ALG, \text{orig}}\}_{\phi \in [\psi]} &\stackrel{d}{=} \{i_\phi^{\ALG, \text{sim}}\}_{\phi \in [\psi]}
    \\\implies \{i_t^{\text{orig}}\}_{t \in [T]} &\stackrel{d}{=} \{i_t^{\text{sim}}\}_{t \in [T]} &&\textit{from conditioning on coupling,}
\end{align*}
i.e. that the distribution of arms pulled by \cref{algo:BB-pull} and \cref{algo:BB-pull-simulated} are identically distributed, conditioned on the coupling. Truncating $\psi$ to $T$ for both algorithms also preserves the identical distribution. The main claim of the lemma follows by another application of the law of total probability. 

\end{proof}

\subsubsection{Regret of $\BBPull$: Proof of \cref{thm:bb2-regret}}

We prove \cref{thm:bb2-regret}, restated below. 
\Rbbpull*
The intuition is that in expectation, the number of times that an arm is pulled in $\BBPull(\ALG)$ before feedback is observed is at most $1 / \min_i f_i$. This means that we can upper bound the regret of $\BBPull(\ALG)$ as $1 / \min_i f_i$ times the regret of $\ALG$.

\begin{proof}[Proof of Theorem \ref{thm:bb2-regret} (Regret of $\BBPull$)]
Recall that the regret guarantees for $\BBPull$ apply only to stochastic losses.
To relate the regret of $\BBPull(\ALG)$ to the regret of $\ALG$, we consider the outputs of $\ALG$ while  $\BBPull(\ALG)$ is evaluated. Recall that $\Phi$ is the number of times that $\ALG$ is called.
Note that the simulated version of $\BBPull(\ALG)$, Algorithm~\ref{algo:BB-pull-simulated}, is run with the set of random variables $Q_{j,\phi}$ for $j\in [K]$ and $\phi \in [T]$, such that $Q_{j,\phi}\sim \mathrm{Geom}(f_j)$, independently. Here,  $Q_{i_\phi^\ALG, \phi}$ denotes the number of times of arm $i^\ALG_\phi$ is pulled until feedback is observed.
Recall that $\bell_i$ denotes the mean loss of arm $i$ and let $i^* = \arg\min_i \bell_i$ be the arm with optimal expected loss. The (pseudo-)regret of $\BBPull(\ALG)$ can be expressed as follows: 
\begin{align*}
\E\left[\sum_{t=1}^T \bell_{i_t}\right] - \min_{i} \sum_{t=1}^{T} \bell_{i} 
&= \E\left[ \sum_{\phi=1}^{\Phi} \sum_{i \in \calA} \mathds 1[i^{\ALG}_\phi = i] \cdot Q_{i,\phi} \cdot \left(\bell_{i} - \bell_{i^*}\right) \right]\\
& \leq \E\left[ \sum_{\phi=1}^{T} \sum_{i \in \calA} \mathds 1[i^{\ALG}_\phi = i] \cdot Q_{i,\phi} \cdot\left(\bell_{i} - \bell_{i^*}\right)\right] \\
&= \E\left[ \sum_{\phi=1}^{T} \sum_{i \in \calA} \mathds 1[i^{\ALG}_\phi = i] \cdot \E[Q_{i,\phi}] \cdot \left(\bell_{i} - \bell_{i^*}\right)\right] \\
&\le \frac{1}{\min_i f_i} \underbrace{\E\left[ \sum_{\phi=1}^{T} \sum_{i \in \calA} \mathds 1[i^{\ALG}_\phi = i] \cdot \left(\bell_{i} - \bell_{i^*}\right)\right]}_{(1)},
\end{align*}
where the second transition follows by noting that, as described above, Algorithm~\ref{algo:BB-pull-simulated}
is run with well-defined variables $Q_{i,\phi}\geq 0$ for all $\phi \le T$ and $\bell_i - \bell_{i^*} \geq 0$ for all $i\in \calA$, so that we can extend the summation to $\phi \in (\Phi, T]$. In the third transition, the outer expectation is over $\ALG$ and the inner expectation is over the feedback observations. And the last transition uses $\mathbb{E}[Q_{i,\phi} ] = \frac{1}{f_{i}} \le \frac{1}{\min_i f_i}$.

To relate (1) to the regret of $\ALG$, we observe that in $\BBPull(\ALG)$, the algorithm $\ALG$ also receives stochastic losses with mean $\bell_i$ when it pulls $i_\phi^{\ALG} = i$ that are identically distributed as in the original instance $\calI$. This means that (1) is exactly equal to the regret of $\ALG$ in an instance with stochastic losses over $T$ time steps. This completes the proof.
\end{proof}

\subsubsection{Monotonicity of $\BBPull$: Proof of \cref{thm:feedbackmonotonicitybbpull}}
\label{subsubsec:bbpull-mono}

Here, we formalize the coupling argument which will allow us to show positive feedback monotonicity in $\FOC$ and negative feedback monotonicity in $\APC$ for $\BBPull$ applied to an underlying algorithm $\ALG$. A very similar approach will be used to prove Theorems \ref{thm:feedbackmonotonicitybbda}, \ref{thm:feedbackmonotonicitybbpullaae}, and \ref{thm:feedbackmonotonicitybbdaaae} in the following sections, though those arguments will require a slightly more complex conditioning step. 

For reference, we restate the result below. 
\Monobbpull*

We are now ready to proceed with the main coupling argument. 
\begin{proof}[Proof of Theorem \ref{thm:feedbackmonotonicitybbpull} (Monotonicity of $\BBPull$).]
Fix an instance $\calI = \left\{\calA, \calF, \calL\right\}$ with stochastic losses. Let $\tilde{f_i} \ge f_i$, and let $\widetilde{\calI} = \left\{\calA, \calF(i), \calL\right\}$. We will denote the time horizon of the transformed algorithm on $\calI$ as $\Phi$, as before, and the time horizon of the transformed algorithm on $\widetilde \calI$ as $\widetilde \Phi$.
We will analyze $\BBPull(\ALG)$ by comparing the behavior of \cref{algo:BB-pull-simulated} on $\calI$ and on $\widetilde{\calI}$ in three steps as follows:
\begin{enumerate}
    \item We construct a probability coupling between the sequence of random variables $Q_{j,\phi}$ and $\widetilde{Q}_{j,\phi}$ for $j\in [K]$ and $\phi = 1, \dots, \infty$. This coupling ensures that 
    $Q_{i,\phi} \geq \widetilde{Q}_{i,\phi}$ for arm $i$ and $Q_{j,\phi}= \widetilde{Q}_{j,\phi}$ for all other arms $j\neq i$, for all $\phi$.\footnote{\label{footnote:infinite-seq}Constructing an infinitely long sequence is only for convenience in using \cref{lem:joint-losses-identical}; we only consume at most $T$ of these random variables in any algorithm for analysis.}

    \item We call \cref{algo:BB-pull-simulated} on $\calI$ and $\widetilde\calI$
    with $Q_{j,\phi}$ and  $\widetilde{Q}_{j,\phi}$ for $j\in [K]$ and $\phi = 1, \dots, \infty$, respectively. Using \cref{lem:joint-losses-identical}, we argue that for any $\Phi^*$, $(i^\ALG_1, \dots, i^\ALG_{\Phi^*})\stackrel{d}{=}(\tilde i^\ALG_1, \dots, \tilde i^\ALG_{\Phi^*})$; then, we couple the arm pulls on each instance so $(i^\ALG_1, \dots, i^\ALG_{\Phi^*}) = (\tilde i^\ALG_1, \dots, \tilde i^\ALG_{\Phi^*})$.
    \item By this step, random variables $Q_{j,\phi}$,   $\widetilde{Q}_{j,\phi}$,  $i^\ALG_\phi$, and $\tilde i^\ALG_\phi$ are fixed according to the above coupling. As a final step, we modify step 2 so that \cref{algo:BB-pull-simulated} terminates after $T$ rounds. In this case, $\ALG$ may be called a different number of times, $\Phi \leq \widetilde{\Phi}$, on instance $\calI$ and $\widetilde{\calI}$. We handle this by showing that this impacts the monotonicity in the claimed  direction.
\end{enumerate}

\paragraph{Step 1: Coupling realizations of feedback observations.}
Note that for $\tilde f_{i} > f_i$, the distribution of $\widetilde{Q}_{i, \phi}$ is stochastically dominated by $Q_{i, \phi}$. That is, as the feedback probability increases, we need fewer pulls to observe feedback when that arm is pulled.
Therefore, there is a joint probability distribution over $(Q_{j,\phi}, \widetilde{Q}_{j,\phi})$ such that for all $\phi$, with probability $1$ the following hold: $Q_{i,\phi} \geq \widetilde{Q}_{i,\phi}$ and for all $j\neq i$, $Q_{j,\phi}= \widetilde{Q}_{j,\phi}$.
This also gives us a coupling, that is a joint distribution, over $\left( \{Q_{j,\phi}\}_{j\in [K], \phi \in \{1, \dots, \infty\}}, \{ \widetilde{Q}_{j,\phi} \}_{j\in [K], \phi\in\{1, \dots, \infty\}}\right)$ that meets the aforementioned property. (See \cref{footnote:infinite-seq} about dealing with infinitely long sequences.)

\paragraph{Step 2: Coupling arms pulled by $\ALG$ across instances $\calI$ and $\widetilde{\calI}$.}
We next consider \Cref{algo:BB-pull-simulated} on two instances $\calI$ and $\widetilde{\calI}$ using the coupled sequence of random variables $Q_{j,\phi}$ and  $\widetilde{Q}_{j,\phi}$ for $j\in [K]$ and $\phi = 1, \dots, \infty$, respectively, as coupled in in Step 1. Conditioned on these sequences, we now apply \cref{lem:joint-losses-identical}. Note that the preconditions of this lemma are met for any $\Phi^*$, so we have that $(i^\ALG_1, \dots, i^\ALG_{\Phi^*})$ and $(\tilde i^\ALG_1, \dots, \tilde i^\ALG_{\Phi^*})$ are identically distributed. This allows us to consider a joint probability distribution over $(i^\ALG_1, \dots, i^\ALG_{\Phi^*}, \tilde i^\ALG_1, \dots, \tilde i^\ALG_{\Phi^*})$ such that 
$i^\ALG_\phi = \tilde{i}^\ALG_\phi$ for all $\phi \in [\Phi^*]$.

\begin{figure}\centering
    \includegraphics[width=0.99\textwidth]{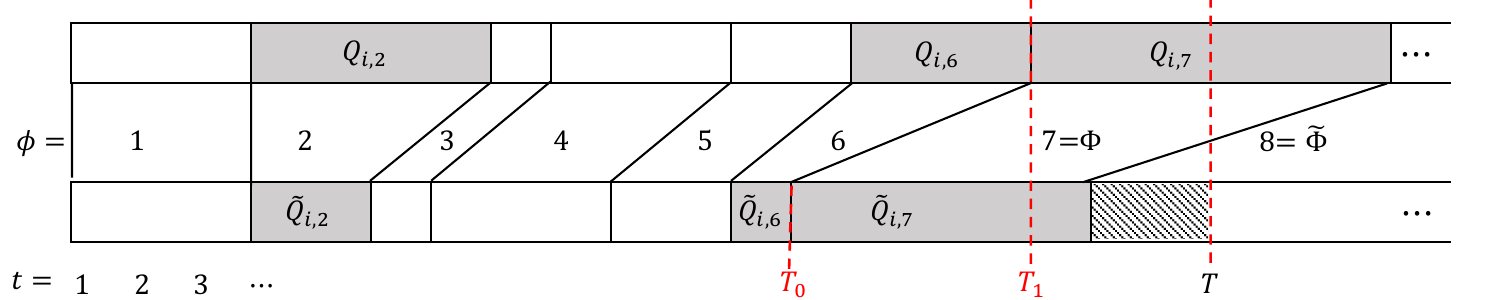}
    \caption{Timelines of $\BBPull(\ALG)$ on instances $\cal I$ (top row) and $\widetilde{\calI}$ (bottom row) are demonstrated. Each time step $t\in [T]$ maps to a block number in $\calI$ that is no more than its block number in $\widetilde{\calI}$. The total number of times $\ALG$ is called in instance $\calI$, $\Phi$, and the number of times it is called in $\widetilde{\calI}$, $\widetilde \Phi$, satisfy $\Phi\leq \widetilde{\Phi}$.}
    \label{fig:bbpull-time}
\end{figure}

\paragraph{Step 3: Handling different stopping times.}
We now have  random variables $Q_{j,\phi}$, $\widetilde{Q}_{j,\phi}$, $i^\ALG_\phi$, $\tilde i^\ALG_\phi$ are all fixed and for all $\phi = 1, \cdots, \infty$ satisfy 
$i^\ALG_\phi = \tilde{i}^\ALG_\phi$,
$Q_{i,\phi} \geq \widetilde{Q}_{i,\phi}$, and $Q_{j,\phi}= \widetilde{Q}_{j,\phi}$ for $j\neq i$.

We next consider the actual performance of \Cref{algo:BB-pull-simulated} on instances $\calI$ and $\widetilde{\calI}$ over $T$ timesteps. 
Note that this is exactly the same as history of arms played by \Cref{algo:BB-pull-simulated} on $\calI$ and $\widetilde{\calI}$, respectively, in Step 2 of the analysis, except that the algorithm now terminates at time $T$.
Therefore, the number of rounds $\ALG$ is called in each of these two instances may be different. Notice that   $\Phi$ and $\widetilde\Phi$ are deterministic variables, since the arms pulled and the number of rounds until an observation is made are all fixed.
It is not hard to see that $\Phi\leq \widetilde\Phi$. 
This is perhaps best seen by considering Figure~\ref{fig:bbpull-time}.
We note that the time horizon of $\BBPull(\ALG)$ for the two instances can be thought of as two sequence of blocks $[\Phi]$ and $[\widetilde \Phi]$.
For each $\phi \leq \min\{\Phi,\widetilde\Phi\}$, $i^\ALG_\phi =\tilde i^\ALG_\phi$. 
Therefore, the only case where $Q_{i^\ALG_\phi, \phi}\neq \widetilde{Q}_{i^\ALG_\phi, \phi}$ is when $i^\ALG_\phi = i$; these are shown by gray blocks in Figure~\ref{fig:bbpull-time}.
In this case  $Q_{i^\ALG_\phi, \phi} \geq \widetilde{Q}_{i^\ALG_\phi, \phi}$ by the coupling we designed above. 
In all other blocks, where $i^\ALG_\phi \neq i$, we have that $Q_{i^\ALG_\phi, \phi} =\widetilde{Q}_{i^\ALG_\phi, \phi}$.
Let $\phi(t)$ (resp. $\widetilde\phi(t)$) be the function that maps timesteps on the timescale indexed by $t$ to timesteps on $\ALG$'s timescale on $\calI$ (resp. $\widetilde \calI$). We can now see that every time step $t\in [T]$ maps to blocks $\phi(t)$ and $\widetilde{\phi}(t)$ in instances $\calI$ and $\widetilde{\calI}$, respectively, such that $\widetilde{\phi}(t) \geq \phi(t)$. This implies that $\Phi\leq \widetilde{\Phi}$, because $\phi(T) \leq \widetilde\phi(T)$.

\paragraph{Notation for Analyzing $\FOC$ and $\APC$.}
The remainder of the proof boils down to analyzing $\FOC$ and $\APC$ on $\calI$ and $\tilde{\calI}$. We use the coupling thus far with the property that $Q_{j,\phi}$, $\widetilde{Q}_{j,\phi}$, $i^\ALG_\phi$, $\tilde i^\ALG_\phi$ are all fixed and for all $\phi = 1, \cdots, \infty$ satisfy 
$i^\ALG_\phi = \tilde{i}^\ALG_\phi$,
$Q_{i,\phi} \geq \widetilde{Q}_{i,\phi}$, and $Q_{j,\phi}= \widetilde{Q}_{j,\phi}$ for $j\neq i$.
Figure~\ref{fig:bbpull-time} provides an intuitive proof of the desired claims.

To formalize these claims, we introduce the following additional notation.
Given a range $R\subseteq [T]$, let $\FOC_i^{R}(\calI)$  be the number of times feedback is observed on arm $i$  in timesteps in $R$ on $\widetilde \calI$, and let $\FOC_i^{R}(\widetilde \calI)$  be the number of times feedback is observed on arm $i$  in timesteps in $R$ on $\tilde \calI$. Similarly, let $\APC_i^{R}(\calI)$ be the number of times arm $i$ is pulled in timesteps in $R$ on $\calI$, and let $\APC_i^{R}(\widetilde \calI)$ be the number of times arm $i$ is pulled in timesteps in $R$ on $\widetilde \calI$. Since we have conditioned on $Q_{j,\phi}$, $\widetilde{Q}_{j,\phi}$, $i^\ALG_\phi$, $\tilde i^\ALG_\phi$, we see that at this point $\FOC_i^{R}(\calI)$, $\FOC_i^{R}(\widetilde \calI)$, $\APC_i^{R}(\calI)$, and $\APC_i^{R}(\widetilde \calI)$ are all deterministic.

Since we will analyze the last time block separately, we let $T_1 = \sum_{\phi \in [\Phi-1]} Q_{i^\ALG_\phi, \phi} \leq T$ be the time step referring to the penultimate block of $\BBPull(\ALG)$ on $\calI$. We let $T_0$ be the corresponding time step on instance $\widetilde{\calI}$ defined by $T_0 = \sum_{\phi \in [\Phi-1]} \widetilde{Q}_{i^\ALG_\phi, \phi}$ (note that the expression sums over $\phi \in [\Phi - 1]$, and not over $\phi \in [\tilde{\Phi} - 1]$). By definition, it holds that $T_0 \le T_1$.

\paragraph{Analyzing $\FOC$.} 

We first prove that $\FOC_i^{[T]}(\tilde \calI) - \FOC_i^{[T]}(\calI) \ge 0$. First, we observe that:
\[
 \FOC_i^{[T_1]}(\calI) = \sum_{\phi \in [\Phi -1]} \1[i^\ALG_\phi = i] = \FOC_i^{[T_0]}(\tilde \calI).
\]
It thus suffices to show that:
\[\FOC_i^{[T]}(\calI) - \FOC_i^{[T_1]}(\calI) \le \FOC_i^{[T]}(\tilde \calI) - \FOC_i^{[T_0]}(\tilde \calI).\]
For ease of exposition, we now consider two cases.
\begin{itemize}
\item Case 1: The $\Phi$th (last) block of $\calI$ pulls $j\neq i$, i.e., $i^\ALG_\Phi \neq i$.  In this case, we have that:
\[\FOC_i^{[T]}(\calI) - \FOC_i^{[T_1]}(\calI) = 0 \le \FOC_i^{[T]}(\calI) - \FOC_i^{[T_0]}(\tilde \calI).  \]
\item Case 2: $i$ was pulled in the $\Phi$th block of $\calI$, i.e., $i^\ALG_\Phi =i$. In this case, we have that:
\[\FOC_i^{[T]}(\calI) - \FOC_i^{[T_1]}(\calI) = \1\left[Q_{i, \Phi} \le T - T_1 \right] \le \1\left[\tilde{Q}_{i, \Phi} \le T - T_1 \right] \le \1\left[\tilde{Q}_{i, \Phi} \le T - T_0 \right] \le \FOC_i^{[T]}(\calI) - \FOC_i^{[T_0]}(\tilde \calI),  \]
as desired. 
\end{itemize}
These two cases prove that $\FOC_i^{[T]}(\tilde \calI) - \FOC_i^{[T]}(\calI) \ge 0$.

Taking an expectation over $Q_{j,\phi}$, $\widetilde{Q}_{j,\phi}$, $i^\ALG_\phi$, $\tilde i^\ALG_\phi$, we see that:
\[\FOC_i(\tilde \calI) - \FOC_i(\calI) = \E\left[\FOC_i^{[T]}(\widetilde{\calI}) - \FOC_i^{[T]}(\calI)\right] \ge 0.  \]

\paragraph{Analyzing $\APC$.}

We first prove that $\APC_i^{[T]}(\tilde \calI) - \APC_i^{[T]}(\calI) \le 0$. We claim that 
\begin{equation}
    T_1 - T_0 = \sum_{\phi \in [\Phi-1]} \mathds{1}(i^\ALG_\phi = i) \left( Q_{i, \phi} - \widetilde{Q}_{i, \phi} \right). \label{eq:bbpull-apc-T_1}
\end{equation}
This is due to the fact that, as discussed above, the only case where $Q_{i^\ALG_\phi, \phi}\neq \widetilde{Q}_{i^\ALG_\phi, \phi}$ is when $i^\ALG_\phi = i$ (these are shown by gray blocks in Figure~\ref{fig:bbpull-time}) in which case $Q_{i^\ALG_\phi, \phi} \geq \widetilde{Q}_{i^\ALG_\phi, \phi}$.
In all other cases, $Q_{i^\ALG_\phi, \phi} =\widetilde{Q}_{i^\ALG_\phi, \phi}$.
Equation~\eqref{eq:bbpull-apc-T_1} implies that 
\begin{equation}
\APC_i^{[T_1]}(\calI) - \APC_i^{[T_0]}(\widetilde{\calI}) = T_1 - T_0.    \label{eq:bbpull-apc-T1-T0}
\end{equation}
For ease of exposition, we now consider two cases.
\begin{itemize}
\item Case 1: $\Phi$th block of $\calI$ pulls $j\neq i$, i.e., $i^\ALG_\Phi \neq i$. In this case, we have that $\APC_i^{[T]}(\calI) = \APC_i^{[T_1]}(\calI)$. 
Moreover, within the last $T - T_0$ timesteps of $\widetilde{\calI}$ at least $T - T_1$ are dedicated to pulling arm $j\neq i$ in the $\Phi$th block of $\widetilde{\calI}$. Thus, 
\[
\APC_i^{[T]}(\widetilde{\calI}) \leq \APC^{[T_0]}_i(\widetilde{\calI}) + T- T_0 - (T - T_1)= \APC_i^{[T_0]}(\widetilde{\calI}) + T_1 - T_0 =  \APC_i^{[T_1]}(\calI) =\APC_i^{[T]}(\calI),
\]
where the second to last equality is by \cref{eq:bbpull-apc-T1-T0}.
\item Case 2: $i$ was pulled in the $\Phi$th block of $\calI$, i.e., $i^\ALG_\Phi =i$. In this case, we have that $\APC_i^{[T]}(\calI) = \APC_i^{[T_1]}(\calI) + T-T_1$.
Furthermore, 
\[
\APC_i^{[T]}(\widetilde{\calI}) \leq \APC^{[T_0]}_i(\widetilde{\calI}) + T- T_0 = \APC_i^{[T_1]}(\calI) + T - T_1 = \APC_i^{[T]}(\calI),
\]
where the second equation is by \cref{eq:bbpull-apc-T1-T0}.
\end{itemize}
These two cases prove that $\APC_i^{[T]}(\tilde \calI) - \APC_i^{[T]}(\calI) \le 0$.
Taking an expectation over $Q_{j,\phi}$, $\widetilde{Q}_{j,\phi}$, $i^\ALG_\phi$, $\tilde i^\ALG_\phi$, we see that:
\[\APC_i(\tilde \calI) - \APC_i(\calI) = \E\left[\APC_i^{[T]}(\widetilde{\calI}) - \APC_i^{[T]}(\calI)\right] \le 0.  \]
This completes the proof.
\end{proof}

\subsection{Proofs for Section \ref{subsec:bb3}: $\BBDivideAdjusted$}
\label{appendix:bb3}

To analyze $\BBDivideAdjusted$, we will combine the approaches of our analyses for $\BBPull$ and $\BBDivide$. For regret, we will analyze the the per-block regret; for monotonicity, we will make a coupling argument. For both we will analyze a simulated version of $\BBDivideAdjusted$, which we present in the following section.

We restate the algorithm below to clarify the dependence on the input $\fst \in (0, \min_if_i]$.
\BBDA*

\subsubsection{Constructing a simulated version of $\BBDivideAdjusted$}

As before, we construct a simulated version of  $\BBDivideAdjusted(\ALG)$. 
Again, we will define a sequence of random variables that determine how $\BBDivideAdjusted(\ALG)$ will proceed on $\calI$ and $\widetilde{\calI}$, and simulate a statistically indistinguishable version of $\BBDivideAdjusted(\ALG)$ in \Cref{algo:BB-da-simulated}. Again, we will index the time horizon with $\ALG$ with $\phi$.

\begin{itemize}
    \item \textit{Losses:}
For each round $\phi \in [\Phi]$ of $\ALG$ and each arm $j \in [K]$, $\ell'_{j, \phi}$ is the placeholder for the loss passed to $\ALG$ if $\ALG$ were to observe the loss of arm $j$ at time $\phi.$  Since we are in the stochastic loss setting, $\ell'_{j,\phi}$ is a random variable drawn from the distribution of arm $j$ (with mean $\bar{\ell}_j$) independently across $\phi$ and $j$, if at least one observation is realized in block $\phi$, and $\ell'_{j,\phi} =1$ otherwise. 
We note that these losses are only observed up to timestep $\Phi$ (which is a random variable less than $T$) and only for the specific arms pulled by the algorithm.

    \item \textit{Feedback probabilities:} for each arm $j \in [K]$ and $\phi \in [T]$, let $U_{j,\phi} \sim \text{Bern}\left(1 - (1-f_j)^{B_j}\right)$ denote the indicator variable for whether feedback will be observed in block $\phi$, where $B_j = \ceil{\frac{3\ln(T)}{\fst}(1 + f_j)}$, for $\fst \in (0, \min_if_i]$.

\end{itemize}

\begin{algorithm2e}
\caption{Simulated version of $\BBDivideAdjusted(\ALG, \fst)$}
\label{algo:BB-da-simulated}
\DontPrintSemicolon
\LinesNumbered
\KwIn{A sequence of integers in $\{0,1\}$, $U_{j,\phi}$ for $\phi \in [T]$ and $j \in [K]$; $\fst \in (0, \min_if_i]$}
Initialize $\phi = 1$. \;
For each arm $j \in [K]$, set $B_j = \ceil{\frac{3\ln(T)}{\fst}(1 + f_j)}$. \;
\While {$t \leq T$}{
    Let $i^{\ALG}_\phi = \ALG(\phi)$ be the output of $\ALG$ at timestep $\phi$. \;
    Let $S_\phi = \{t, t + 1, \dots, \min(t + B_{i^{\ALG}_\phi}, T)\}$. \;
    \For{$t' \in S_\phi$}{
    Pull $i_{\phi}^{\ALG}$, i.e. $i_{t'} = i_\phi^\ALG$, and let $t \gets t + 1$.}
    \If{$U_{i^{\ALG}_\phi, \phi} = 1$}{Observe 
    and return $\ell'_{i_\phi^\ALG, \phi} := \ell_{i_t,t}$ to $\ALG$.}
    \Else{Return $\ell'_{i_\phi^\ALG, \phi} = 1$ to $\ALG$.}
    Let $\phi \gets \phi + 1.$
    }
\end{algorithm2e}

\begin{lemma}
    \label{lem:joint-losses-identical-bbda}
     For each arm $j \in [K]$, set $B_j = \ceil{\frac{3\ln(T)}{\min_if_i}(1 + f_j)}$.
     Let $\Phi^*$ be any positive integer and $T = \Phi^* \cdot \max_jB_j$ be the time horizon.
     Let $\calI = \{\calA, 
     \calF, 
     \calL\}$ be a stochastic instance with time horizon $T$; let $\tilde f_i \geq f_i$ and $\widetilde \calI = \{\calA, \tilde\calF(i), \calL\}$. 
    Let $U_{j, \phi} = \widetilde U_{j,\phi}$ for $j \in [K]$ and $\phi \in [T]$.
    Run \cref{algo:BB-da-simulated} with parameters $\left\{U_{j, \phi}\right\}_{j \in [K], \phi \in [T]}$ on $\calI$ and run \cref{algo:BB-da-simulated} with parameters $\left\{\widetilde{U}_{j, \phi}\right\}_{j \in [K], \phi \in [T]}$ on $\widetilde \calI$.
    Let 
   $i^\ALG_\phi$ and $\tilde{i}^\ALG_\phi$ denote the  arms pulled in the description of \cref{algo:BB-da-simulated} for the two instances, respectively.
      Then, 
      the following two vector valued random variables are identically distributed: $(i^\ALG_1, \dots, i^\ALG_{\Phi^*})$ and $(\tilde{i}^\ALG_1, \dots, \tilde{i}^\ALG_{\Phi^*})$.
\end{lemma}

\begin{proof} 
Let $\{\ell'_{j, \phi}\}_{j \in [K], \phi \in [\psi]} $ denote possible loss sequences observed on $\calI$ up to some $\psi \leq \Phi^*$ and $\{\tilde\ell'_{j, \phi}\}_{j \in [K], \phi \in [\psi]} $ denote possible loss sequences observed on $\widetilde\calI$ up to the same $\psi$. 
Let us fix the bit of randomness $b$ used for $\ALG$ on $\calI$ to be the same as the bit of randomness used for $\ALG$ on $\widetilde\calI$.
Because we have set $T = \Phi^* \cdot \max_jB_j$, we are guaranteed that blocks $\phi = 1, \dots, \psi$ will have been reached on both $\tilde \calI$ and $\calI$.
Conditioned on $b$,
let $F_b: [0,1]^{K \times \psi} \to [K]^{\psi}$ be the mapping from 
all $\ell'_{j,\phi}$ for $\phi \leq \psi$,
to the sequence of arms $\ALG$ would have pulled corresponding to those losses, that is, 
\begin{align*}
    F_b\left(\{\ell'_{j, \phi}\}_{j \in [K], \phi \in [\psi]}\right) = (i_1^\ALG, i_2^\ALG, \dots, i_\psi^\ALG).
\end{align*}
Note that $F_b$ does not depend on the feedback probabilities $f_i$, because $\ALG$ is fully oblivious to these quantities. For any $b$, $F_b$ is fully deterministic. 
Therefore, the distribution of $(i_1^\ALG, i_2^\ALG, \dots, i_\psi^\ALG)$ is fully specified by the distributions of $\{\ell'_{j, \phi}\}_{j \in [K], \phi \in [\psi]}$, and the distribution of $(\tilde i_1^\ALG, \tilde i_2^\ALG, \dots, \tilde i_\psi^\ALG)$ is fully specified by the distributions of $\{\tilde \ell'_{j, \phi}\}_{j \in [K], \phi \in [\psi]}$.

In our specification of \cref{algo:BB-da-simulated}, the sequences of losses passed to $\ALG$ are determined not only by the underlying loss distributions for each arm selected $i_t$, but also by the random variables $U_{j, \phi}$ which determine whether $\ALG$ will observe $\ell_{i_t, t}$ (which is actually sampled from the distribution of the selected arm $i_t$), or a loss of $1$. Conditioning on $U_{j,\phi} = \widetilde U_{j,\phi}$ for all $j$ and $\phi$ gives us that the loss sequences are distributed identically across instances. Therefore, we have that
\begin{align*}
\{\ell'_{j, \phi}\}_{j \in [K], \phi \in [\psi]} &\stackrel{d}{=}  \{\tilde\ell'_{j, \phi}\}_{j \in [K], \phi \in [\psi]}    
\\\implies F_b\left(\{\ell'_{j, \phi}\}_{j \in [K], \phi \in [\psi]}\right) &\stackrel{d}{=}  F_b\left(\{\tilde\ell'_{j, \phi}\}_{j \in [K], \phi \in [\psi]} \right)
\\\implies (i_1^\ALG, i_2^\ALG, \dots, i_\psi^\ALG) &\stackrel{d}{=} (\tilde i_1^\ALG, \tilde i_2^\ALG, \dots, \tilde i_\psi^\ALG),
\end{align*}
where $\stackrel{d}{=}$ denotes identically distributed relationship. Finally, because this holds conditionally over any arbitrary $b$, we can integrate over all possible random bits $b$ to establish the claim.
\end{proof}

\begin{lemma}
\label{lem:sim-indistinguishable-bbda}
Fix an instance $\calI$.
    Let $\{i_t^{\text{orig}}\}_{t \in [T]}$ be a sequence of random variables that represents the arms selected by \cref{algo:BB-divide-adjusted} on $\calI$ over the time horizon $T$, and $\{i_t^\text{sim}\}_{t \in [T]}$ be a sequence of random variables that represents the arms selected by  \cref{algo:BB-da-simulated} on an identical instance $\calI$. Then the sequence $\{i^\text{orig}_t\}_{t \in [T]}$ is distributed identically to $\{i^\text{sim}_t\}_{t \in [T]}$. 
\end{lemma}

The intuition for this lemma is similar to the proof of \cref{lem:simulated-indistinguishable}; here, we argue that the likelihood that no feedback is observed at any block $\phi$ is identically distributed for both \cref{algo:BB-divide-adjusted} and \cref{algo:BB-da-simulated}, and that taking one sample from the loss distribution (as \cref{algo:BB-da-simulated} does) is the same as taking a uniform sample out of several possible observations (as \cref{algo:BB-divide-adjusted} does).

\begin{proof}

We run \cref{algo:BB-divide-adjusted} and \cref{algo:BB-da-simulated} on identical copies of $\calI$; we distinguish each copy by $\calI^{\text{orig}}$ for \cref{algo:BB-divide-adjusted} and $\calI^{\text{sim}}$ for \cref{algo:BB-da-simulated}. 
In the first step, we introduce $F_b$ which formalizes that arm selected by $\ALG$ given the random variables $\ell'_{j,\phi}$ defined earlier. We use this in the second step to show that arms selected by $\ALG$ are distributed the same across the two algorithms. In the last step, we use the fact that the block sizes are of equal lengths across the two algorithms to show that arms pulled by \cref{algo:BB-divide-adjusted} and \cref{algo:BB-da-simulated} are distributed the same.

\paragraph{Step 1: Formalize $\ALG$ arm selection.}
Recall that  \cref{algo:BB-divide-adjusted} and \cref{algo:BB-da-simulated} both make calls to the same underlying $\ALG.$ Let $\ell'_{j, \phi}$ be, as defined earlier, the placeholder for losses passed to $\ALG$, if $\ALG$ were to observe the loss of arm $j$ at time $\phi$.
    Let $b$ be the bit of randomness used for $\ALG$ in  \cref{algo:BB-divide-adjusted} and \cref{algo:BB-da-simulated}.
Now, conditioning on $b$, let $F_b: [0,1]^{K \times \psi} \to [K]^{\psi}$ be the mapping from all $\ell'_{j, \phi}$s up to time $\phi \leq \psi$, to the sequence of arms $\ALG$ would have pulled corresponding to those losses, that is, 
\begin{align*}
    F_b\left(\{\ell_{j, \phi}\}_{j \in [K], \phi \in [\psi]}\right) = (i_1^\ALG, i_2^\ALG, \dots, i_\psi^\ALG).
\end{align*}
Note that $F_b$ does not depend on the feedback probabilities $f_i$ or the feedback observations $Q_{i,\phi}$ or $X_{i, t}$, because $\ALG$ is fully oblivious to these quantities. For any $b$, $F_b$ is fully deterministic. 
Furthermore, the simulated and real algorithms use $\ALG$ with the same bit of randomness, so $F_b^{\text{orig}} = F_b^{\text{sim}}$, and the arms selected by $\ALG$ for either \cref{algo:BB-divide-adjusted} and \cref{algo:BB-da-simulated} are fully specified by the distributions of the losses for each arm. 

\paragraph{Step 2:
Arms selected by $\ALG$ are distributed the same.
}
We first establish that $\{\ell'_{j,\phi}\}_{j \in [K], \phi \in [\psi]}$ are identically distributed.

Recall that $\ell'_{j,\phi}$ are placeholders for losses of all arms $j$ and round $\phi$ of $\ALG$ (although $\ALG$ only takes into account the random variables for arms it pulls).

    By our specification of \cref{algo:BB-da-simulated}, given $j$ and $\phi$, the event that $\ell'_{j, \phi}$ is drawn from  the distribution of arm $j$ is determined by $U_{j,\phi} \sim \text{Bern}(1 - (1-f_j)^{B_j})$ and has probability  probability $1 - (1-f_j)^{B_j}$. And, with probability $(1-f_j)^{B_j}$, $\ell'_{j,\phi} = 1$.
    
    For \cref{algo:BB-divide-adjusted}, note that $j$ will be pulled exactly $B_j$ times in each block. The likelihood that at least at one of these round a loss is generated from arm $j$ is exactly $1-(1-f_j)^{B_j}$.
    Note that in this case, $\ell'_{j, \phi}$ is drawn uniformly from the realized losses, which is equivalent to being drawn from the loss of arm $j$.
    And, with probability $(1-f_j)^{B_j}$, $\ell'_{j,\phi}$ is deterministically set to $1$.
    Note that 
    the realizations of $U_{j,\phi}$ and $X_{j,t}$ are all independent across $\phi, t$, and $K$, so we have that 
    \[
    \{\ell'^{\text{  orig}}_{j,\phi}\}_{j \in [K], \phi \in [\psi]} \stackrel{d}{=} \{\ell'^{\text{  sim}} _{j,\phi}\}_{j \in [K], \phi \in [\psi]}.
    \]
    Since $F_b$ is a deterministic map, we have that 
    \begin{align*}
        F_b\left(\{\ell'^{\text{  orig}}_{j, \phi}\}_{j \in [K], \phi \in [\psi]}\right) &\stackrel{d}{=}  F_b\left(\{\ell'^{\text{  sim}}_{j, \phi}\}_{j \in [K], \phi \in [\psi]} \right)\\
        \{i_\phi^{\ALG, \text{orig}}\}_{\phi \in [\psi]} &\stackrel{d}{=} \{i_\phi^{\ALG, \text{sim}}\}_{\phi \in [\psi]} 
    \end{align*}

\paragraph{Step 3: Arms selected by \cref{algo:BB-divide-adjusted} and \cref{algo:BB-da-simulated} are identically distributed.}
Note that by the specification of each algorithm, for every $i_\phi$ selected by $\ALG$, \cref{algo:BB-divide-adjusted} and \cref{algo:BB-da-simulated} will pull $i_\phi$ exactly $B_{i_\phi}$ times. Having steps $1$ and $2$ for $\psi > T$, gives us
\begin{align*}
     \{i_\phi^{\ALG, \text{orig}}\}_{\phi \in [\psi]} &\stackrel{d}{=} \{i_\phi^{\ALG, \text{sim}}\}_{\phi \in [\psi]}
  \\
  \{i_t^{\text{orig}}\}_{t \in [T]} &\stackrel{d}{=} \{i_t^{\text{sim}}\}_{t \in [T]}.
\end{align*}
Applying the law of total expectation over possible random bits $b$ proves the claim.

\end{proof}

\subsubsection{Regret of $\BBDivideAdjusted$: Proof of \cref{thm:bb3-regret}}

First, we prove Theorem \ref{thm:bb3-regret}. The intuition is that we can bound the size of any block by $\max B_i \le \frac{6 \ln T}{\fst}$. Since $B_i$ is sufficiently large, with high probability, there will be at least one observation in each block. Since the losses are stochastic, we can upper bound the regret of $\BBDivideAdjusted(\ALG)$ as $\max_j B_j \cdot R_\ALG(T)$ as desired. 

We restate the regret result of \cref{thm:bb3-regret} below.
\Regretbbda*

The argument requires \cref{lemma:clean-bbda}, which ensures that at least one observation from the true loss distribution is made in every block (note that this is very similar to the statement and proof of \cref{lemma:clean}, except that the block size $B$ is no longer fixed).
\begin{lemma}
    \label{lemma:clean-bbda}
    Fix an $\fst \in (0, \min_if_i]$, and let $\Phi \leq T$. Divide the time horizon $T$ into blocks of size $B_\phi \geq \frac{3 \ln T}{\fst}$ for $\phi \in \Phi$. Suppose then that for each block $\phi \in \{1,2,\dots,\Phi$\}, we play the same arm $i_\phi$ a total of $B_\phi$ times, i.e. for every round in block $\phi$, as in Algorithms \ref{algo:BB-divide-adjusted} and \ref{algo:BB-da-simulated}.
    Let $E$ be the ``\emph{clean event}'' that at least one feedback observation occurs in \emph{each} block $\phi$, i.e., that for all blocks $\phi$, $\exists t \in S_\phi: X_{i_t,t} = 1$. Then, $\Pr[E] \ge 1 - 1 / T^2$. 
    \end{lemma}
    \begin{proof}
    
    Let $E_\phi$ be the event that at least one feedback observation occurred in block $\phi$, i.e., $\exists t \in S_\phi: X_{i_t,t} = 1$. Since for any arm $i$, $\Pr[X_{i,t} = 1] = f_{i}$, then for arm $i_\phi$, we have that
    \[\Pr[\neg E_\phi] = (1-f_{i_\phi})^B \leq (1-\fst)^B \leq \exp(-\fst B) \leq 1/T^3.
    \]
    Union bounding over all $\Phi \leq T$ blocks, we conclude that
    \[ \Pr[\neg E] \leq \sum_{\phi \in [\Phi]} \Pr[\neg E_\phi] \leq 1/T^2. \qedhere
    \]
    
    \end{proof}

\begin{proof}[Proof of \Cref{thm:bb3-regret} (Regret of $\BBDivideAdjusted$)]

This argument proceeds almost identically to the proof of \Cref{thm:bb1-regret}, the regret bound on $\BBDivide$, in the stochastic case. Recall that $\fst \in (0, \min_j f_j]$.
For notational convenience, let $B = \frac{3\cdot\ln(T)}{\fst}$.
First, note that it must be the case that the size of any block $B_i$ is bounded as follows, because $1 \leq 1 + f_i \leq 2$:
\[B \leq B_i \leq 2B.\]
Then, we will have at most $\floor{T/B}$ blocks, and each block will incur at most $2B$ regret.
We use \Cref{lemma:clean-bbda} to argue that we will see at least one feedback observation in each block with probability $1 - 1/T^2$; conditioned on this occurring, using the above bounds on the number of blocks and the size of each block, 
the (pseudo-)regret of $\BBDivideAdjusted(\ALG)$ can be expressed as

\begin{align*}
\E\left[\sum_{t=1}^T \bell_{i_t}\right] - \min_{i} \sum_{t=1}^{T} \bell_{i} 
&= \E\left[ \sum_{\phi=1}^{\Phi} \sum_{i \in \calA} \mathds 1[i^{\ALG}_\phi = i] \cdot B_{i_\phi} \cdot \left(\bell_{i} - \bell_{i^*}\right) \right]\\
& \leq \E\left[ \sum_{\phi=1}^{\floor{T/B}} \sum_{i \in \calA} \mathds 1[i^{\ALG}_\phi = i] \cdot B_{i_\phi} \cdot \left(\bell_{i} - \bell_{i^*}\right)\right] \\
&\leq \E\left[ \sum_{\phi=1}^{\floor{T/B}} \sum_{i \in \calA} \mathds 1[i^{\ALG}_\phi = i] \cdot \max_jB_j \cdot \left(\bell_{i} - \bell_{i^*}\right)\right] \\
&=  \frac{6 \ln(T)}{\fst} \cdot \underbrace{\E\left[ \sum_{\phi=1}^{\floor{T/B}} \sum_{i \in \calA} \mathds 1[i^{\ALG}_\phi = i] \cdot \left(\bell_{i} - \bell_{i^*}\right)\right]}_{(1)},
\end{align*}
where the second transition follows by noting that, as described above, $\Phi \leq \floor{T/B}$, and $B_j$ is well defined for all $j \in [K]$, regardless of timestep, so that we can extend the summation to $\phi \in (\Phi, \floor{T/B}]$. The last transition uses $\max_jB_j = \frac{6 \ln(T)}{\fst}$, a deterministic quantity.
To relate (1) to the regret of $\ALG$, we observe that in $\BBDivideAdjusted(\ALG)$, the algorithm $\ALG$ also receives stochastic losses with mean $\bell_i$ when it pulls $i_\phi^{\ALG} = i$ that are identically distributed as in the original instance $\calI$. This means that (1) is exactly equal to the regret of $\ALG$ in an instance with stochastic losses over $\frac{T\fst}{3\ln T}$ time steps. This completes the proof.

\end{proof}

\subsubsection{Monotonicity of $\BBDivideAdjusted$: Proof of \cref{thm:feedbackmonotonicitybbda}}

We now prove \cref{thm:feedbackmonotonicitybbda}. While the regret proof followed the regret proof for $\BBDivide$, the monotonicity proof will parallel the coupling argument we made for $\BBPull$, with two key differences. First, the size of each block is now deterministic rather than a random variable; this makes analyzing each block easier, but requires a slightly different approach to formalizing the realization of randomness because the randomness is now in the \textit{selection} of observations.
Second, higher feedback probabilities will correspond to \textit{larger} blocks by construction, which changes the direction of monotonicity in $\APC$ as desired. 

The analogous result for $\FOC$ follows directly from \cref{prop:relationship}. Intuitively, recall that in general, higher $f_i$ implies higher $\FOC_i$ for the same number of arm pulls, by definition; therefore, if $\APC_i$ is positive monotonic, $\FOC_i$ must be as well. 

For reference, we restate \cref{thm:feedbackmonotonicitybbda} below. 
\Monobbda*

\begin{proof}[Proof of Theorem \ref{thm:feedbackmonotonicitybbda} (Monotonicity of $\BBDivideAdjusted$)]

Again, we fix an instance $\calI = \left\{\calA, \calF, \calL\right\}$ with stochastic losses. Let $\tilde{f_i} \ge f_i$, and let $\tilde{\calI} = \left\{\calA, \calF(i), \calL\right\}$. The four-step argument proceeds as follows:

\begin{enumerate}
    \item We first condition on the event that $U_{j,\phi} = \widetilde U_{j,\phi} = 1$ for all $j \in [K], \phi \in [T]$. 
    \item We call \cref{algo:BB-da-simulated} on $\calI$ and $\tilde\calI$, passing in $U_{j,\phi}$ and $\widetilde U_{j,\phi}$, respectively. We use \cref{lem:joint-losses-identical-bbda} to argue that for any $\Phi^*$, $(i_1^{\ALG}, \dots, i_{\Phi^*}^\ALG) \stackrel{d}{=} (\tilde i_1^{\ALG}, \dots, \tilde i_{\Phi^*}^\ALG)$, then couple the arm pulls on each instance so $(i_1^{\ALG}, \dots, i_{\Phi^*}^\ALG) = (\tilde i_1^{\ALG}, \dots, \tilde i_{\Phi^*}^\ALG)$.
    \item By this step, $i_\phi^\ALG$ and $\tilde i_\phi^\ALG$ are fixed up to $\Phi^*$ by the above coupling. Now, we truncate the run \cref{algo:BB-da-simulated} to $T$ rounds on each instance. In this case, $\ALG$ may be called a different number of times, $\Phi \geq \widetilde{\Phi}$, on instance $\calI$ and $\widetilde\calI$. This impacts the monotonicity of $\APC$ in the claimed direction. 
    \item Finally, we handle the conditioning from Step 1, using \cref{lemma:clean} to argue that the event that an observation is not observed in at least one block $\phi$ contributes at most $1/T$ to $\APC_i(\calI)$. 
 \end{enumerate}

 \paragraph{Step 1: Condition on feedback observations.} First, let $E$ be the event that $U_{j,\phi} = \widetilde U_{j,\phi} = 1$ for all $j \in [K], \phi \in [T]$. By \cref{lemma:clean}, $\Pr[E] \geq 1 - 1/T^2$. Then, for any $\phi > T$, we let $U_{j,\phi}$ and $\tilde U_{j,\phi}$ take on arbitrary values in $\{0,1\}$.  We condition on $E$ for Steps 2-4. 

 \paragraph{Step 2: Run \cref{algo:BB-da-simulated} and couple arms pulled by $\ALG$ across $\calI$ and $\widetilde\calI$.}
We next consider \cref{algo:BB-da-simulated} on $\calI$ and $\tilde\calI$ using the sequences $U_{j,\phi}$ and $\widetilde U_{j,\phi}$ for $j \in [K]$ and $\phi = 1, \dots,\infty$, respectively. We can now apply \cref{lem:joint-losses-identical-bbda}, letting $\Phi^* = T$, so that $(i^\ALG_1, \dots, i^\ALG_{\Phi^*})$ and $(\tilde i^\ALG_1, \dots, \tilde i^\ALG_{\Phi^*})$ are identically distributed. This allows us to consider a joint probability distribution over \\ $(i^\ALG_1, \dots, i^\ALG_{\Phi^*}, \tilde i^\ALG_1, \dots, \tilde i^\ALG_{\Phi^*})$ such that 
$i^\ALG_\phi = \tilde{i}^\ALG_\phi$ for all $\phi \in [\Phi^*]$.

\paragraph{Step 3: Handle stopping times.} 
We condition on the coupling thus far with the property that $U_{j,\phi}$, $\widetilde{U}_{j,\phi}$, $i^\ALG_\phi$, $\tilde i^\ALG_\phi$ are all fixed and satisfy 
$i^\ALG_\phi = \tilde{i}^\ALG_\phi$ and $U_{j,\phi} = \tilde{U}_{j,\phi} = 1$. 

This step can be thought of as the inverse of Step 3 of the proof of \cref{thm:feedbackmonotonicitybbpull}. As in that step, $\Phi$ and $\tilde\Phi$ are deterministic. This time, however, we now have that $\Phi \geq \widetilde\Phi$; see \cref{fig:bbda-time} for an illustration.
Intuitively, on $\widetilde\calI$, the block sizes when $i$ is pulled are \textit{larger} than on $\calI$, so $\BBDivideAdjusted(\ALG)$ moves through the $\phi$-indexed timescale more slowly on $\widetilde\calI$.

Let $B_\phi := B_{i_\phi^\ALG}$ denote the size of block $\phi$ on $\calI$ and $\widetilde B_\phi:=B_{\widetilde i_\phi^\ALG} $ denote the size of block $\phi$ on $\tilde\calI$. For each $\phi \leq \min(\Phi,\widetilde\Phi)$, we know that $i_\phi^\ALG = \widetilde i_\phi^\ALG$. 
Therefore, the only case where $B_\phi \neq \widetilde B_\phi$ is when $i_\phi^{\ALG} = i$; these are illustrated by gray blocks in \cref{fig:bbda-time}, in which case $B_\phi \leq \widetilde B_\phi$, by definition. Let $\phi(t)$ (resp. $\widetilde\phi(t)$) be the function that maps timesteps on the timescale indexed by $t$ to timesteps on $\ALG$'s timescale on $\calI$ (resp. $\widetilde \calI$).
Every time step $t\in [T]$ maps to blocks $\phi(t)$ and $\widetilde{\phi}(t)$ in instances $\calI$ and $\widetilde{\calI}$, respectively, such that $\widetilde{\phi}(t) \leq \phi(t)$.
This implies that $\Phi\leq \widetilde{\Phi}$. 

\paragraph{Notation for Analyzing $\APC$.}
We are now ready to analyze $\APC$ on $\calI$ and $\widetilde{\calI}$. 
To formalize our analysis, we introduce the following additional notation (following the proof of \cref{thm:feedbackmonotonicitybbpull}).
Given a range $R\subseteq [T]$, let $\APC_i^{R}(\calI)$ be the number of times arm $i$ is pulled in timesteps in $R$ on $\calI$, and let $\APC_i^{R}(\widetilde \calI)$ be the number of times arm $i$ is pulled in timesteps in $R$ on $\widetilde \calI$. Since we have conditioned on $U_{j,\phi}$, $\widetilde{U}_{j,\phi}$, $i^\ALG_\phi$, $\tilde i^\ALG_\phi$, we see that at this point $\APC_i^{R}(\calI)$ and $\APC_i^{R}(\widetilde \calI)$ are both deterministic.

Since we will analyze the last time block separately, we let $T_1 = \sum_{\phi \in [\widetilde\Phi-1]}B_\phi \leq T$ be the time step referring to the end of the penultimate block of $\BBDivideAdjusted(\ALG)$ on $\widetilde\calI$. Let $T_0$ be the analogous time on $\calI$, so that $T_0 = \sum_{\phi \in [\widetilde\Phi-1]} \widetilde B_\phi \leq T_1$. 

\paragraph{Analyzing $\APC$.} We first prove that $\APC_i^{[T]}(\calI) \leq \APC_i^{[T]}(\widetilde\calI)$. Because the only case where $B_\phi \neq \widetilde B_\phi$ is when $i^\ALG_\phi = i$, we have that 
\begin{align*}
    T_1 - T_0 &= \sum_{\phi \in [\widetilde \Phi - 1]} \mathds 1 (i_\phi^\ALG = i)\cdot \left(\widetilde B_{i_\phi} - B_{i_\phi}\right).
    \\&= \APC_i^{[T_1]}(\widetilde\calI) - \APC_i^{[T_0]}(\calI).
\end{align*}
Now, consider two cases. 
\begin{itemize}
    \item Case 1: $i$ was pulled in the $\widetilde\Phi$th block of $\widetilde \calI$, i.e. $i_{\widetilde\Phi}^\ALG = i$. Then, 
    \begin{align*}
        \APC_i^{[T]}(\calI) &\leq \APC_i^{[T_0]}(\calI) + T - T_0
        \\&= \APC_i^{[T_1]}(\widetilde\calI) + T - T_1 
        \\&= \APC_i^{[T]}(\widetilde\calI).
    \end{align*}
    \item Case 2: Some other arm was pulled in the $\widetilde\Phi$th block of $\widetilde\calI$, i.e. $i_{\widetilde\Phi}^\ALG \neq i$. Then, we know that $\APC_i(\widetilde\calI) = \APC_i^{[T_1]}(\widetilde\calI)$. Moreover, within the last $T - T_0$ timesteps of $\calI$, at least $T - T_1$ of them are dedicated to pulling arm $j \neq i$ in the $\widetilde\Phi$th block of $\widetilde\calI$. Then, 
    \begin{align*}
        \APC_i^{[T]}(\calI) &\leq \APC_i^{[T_0]}(\calI) + T - T_0 - (T - T_1)
        \\&= \APC_i^{[T_0]}(\calI) + T_1 - T_0 
        \\&= \APC_i^{[T_1]}(\widetilde\calI) = \APC_i^{[T]}(\widetilde\calI).
    \end{align*}
\end{itemize}
Combining these two cases gives us that
\begin{align*}
    \APC_i^{[T]}(\calI) \leq \APC_i^{[T]}(\widetilde\calI).
\end{align*}

We can apply the law of total expectation over the sequences $U_{j,\phi}$, $\widetilde{U}_{j,\phi}$, $i^\ALG_\phi$, $\tilde i^\ALG_\phi$. Let $\APC_i(\calI \mid E)$ notate the metric $\APC_i$ on instance $\calI$ conditioned on the clean event $E$. We see that:
\begin{align*}
    \APC_i(\widetilde \calI \mid E) - \APC_i(\calI \mid E)  = \E\left[\APC_i^{[T]}(\widetilde\calI) - \APC_i^{[T]}(\calI) \mid E \right] \geq 0.
\end{align*}
This means that:
\begin{align*}
    \APC_i(\widetilde \calI \mid E) \geq \APC_i(\calI \mid E).
\end{align*}

\begin{figure}\centering
    \includegraphics[width=0.99\textwidth]{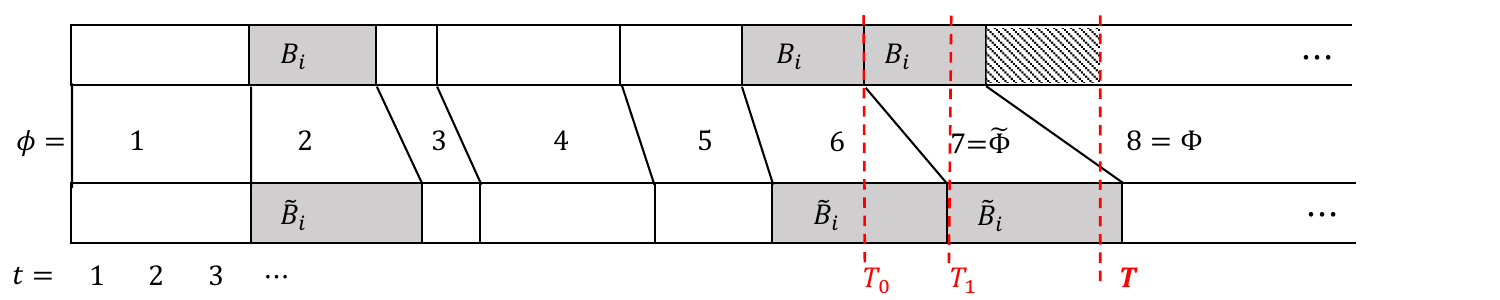}
    \caption{Timelines of $\BBDivideAdjusted(\ALG)$ on instances $\cal I$ (top row) and $\widetilde{\calI}$ (bottom row) are demonstrated. Each time step $t\in [T]$ maps to a block number in $\calI$ that is no \textit{less} than its block number in $\widetilde{\calI}$. The total number of times $\ALG$ is called in instance $\calI$, $\Phi$, and the number of times it is called in $\widetilde{\calI}$, $\widetilde \Phi$, satisfy $\Phi\geq \widetilde{\Phi}$. 
    Note that this is similar to \cref{fig:bbpull-time}, except that the direction of monotonicity has switched and that the size of $B_i$ and $\widetilde B_i$ is deterministic in each instance.    }
    \label{fig:bbda-time}
\end{figure}

\paragraph{Step 4: Handle conditioning on feedback observations.} Finally, recall that up to this point, we are still conditioning on $E$ from Step 1, i.e. that we see feedback in every block on each instance. By  \cref{lemma:clean}, $\Pr[\neg E] \leq 1/T^2$.
In the worst case, we pull $i$ for every $t \in [T]$ on $\calI$, which gives $\APC_i(\calI \mid \neg E ) \leq T$. To relate this to $\APC_i(\calI)$ overall, we can see that 
\begin{align*}
    \APC_i(\calI) &= \APC_i(\calI \mid E) \cdot \Pr[E] + \APC_i(\calI \mid \neg E) \cdot \Pr[\neg E]
    \\&\leq \APC_i(\calI \mid E) + T \cdot 1/T^2
    \\\implies \APC_i(\calI) - 1/T &\leq \APC_i(\calI \mid E).
\end{align*}
Combining this with the result from Step 3, we have that
\begin{align*}
    \APC_i(\widetilde \calI) &\geq 
    \APC_i(\widetilde \calI \mid E) 
    \geq \APC_i(\calI \mid E) \geq \APC_i(\calI) - 1/T.
\end{align*}

\textbf{Analyzing $\FOC$.}
Applying \cref{prop:relationship} gives us 
$
        \FOC_i(\tilde{\calI}) \cdot f_i \geq \FOC_i(\calI)\cdot \tilde{f_i} - \tilde{f_i}/T
$, and the result follows from dividing both sides by $f_i$.
\end{proof}
\section{Supplemental Materials for Sections \ref{subsec:regret} and \ref{subsec:mono}}

In this section, we analyze $\BBPull(\AAE)$ (\Cref{subsec:analysisbbpullaae}), $\BBPull(\UCB)$ (\Cref{appendix:UCB}), and $\BBDivideAdjusted(\AAE)$ (\Cref{appendix:bbdaaae}).

\subsection{Analysis of $\BBPull$ applied to AAE}\label{subsec:analysisbbpullaae}

We prove monotonicity properties and regret bounds for $\BBPull(\AAE)$ (Algorithm \ref{algo:BB-pull-se}), where $\AAE$ denotes the standard Active Arm Elimination algorithm.

\subsubsection{A simulated version of $\BBPull(\AAE)$}
\label{subsubsec:bbpull-aae-sim}

We consider the simulated version of $\BBPull(\AAE)$ given by Algorithm \ref{algo:BB-pull-simulated} applied to $\AAE$. For convenience, we explicitly state this algorithm below (Algorithm \ref{algo:BB-pull-simulated-aae}).

Let us define the same random variables as those used in Algorithm \ref{algo:BB-pull-simulated}, restated for convenience.  (Recall that $\phi$ indexes losses for the time horizon of $\ALG$, $\Phi$ is the total number of times $\ALG$ is called by $\BBPull(\ALG)$, and $\Phi \leq T$ because $\ALG$ can be called at most $T$ times.)

\begin{itemize}
    \item \textit{Losses:} For each round $\phi \in [\Phi]$ of $\ALG=\AAE$ and each arm $j \in [K]$, 
    let $\ell'_{j, \phi} := \ell_{j, t}$ be the loss for arm $j$ at a time step $t$ that corresponds to the last time step in block $\phi$ of $\BBPull(\AAE)$. Since we are in the stochastic loss setting, $\ell'_{j,\phi}$ is a random variable drawn from the distribution of arm $j$ (with mean $\bar{\ell}_j$) independently across $\phi$ and $j$. 
    \item \textit{Feedback realizations:} For all $j \in [K]$ and $\phi \in [T]$, let $Q_{j,\phi} \sim \textrm{Geom}(f_j)$ for $\phi \in [T]$ be a random variable distributed according to the geometric distribution with parameter equal to the feedback probability of arm $j$. 
    (These random variables are also fully independent across values of $j$ and $\phi$.) 
\end{itemize}

We are now ready to present Algorithm \ref{algo:BB-pull-simulated-aae}. For ease of analysis, we make the slight modification from Algorithm \ref{algo:BB-pull-se} that we convert the set $R_{i, s}$ which keeps track of time steps in the time horizon of $\BBPull(\AAE)$ to the set $U_{i,s}$ which keeps track of time steps in the time horizon of $\ALG=\AAE$. The behavior of the algorithm remains unchanged under this change.

\begin{algorithm2e}[htbp]
\caption{Simulated version of $\BBPull(\AAE)$ (\cref{algo:BB-pull-simulated} applied to $\AAE$)}
\label{algo:BB-pull-simulated-aae}
\DontPrintSemicolon
\LinesNumbered
Maintain active set $A$; start with $A := [K]$.
\\
Initialize phase $s=1$, $t =1$, and $\phi = 1$.
\\
\While{$t \le T$}{
    \For{arm $i \in A$}{
    Let $U_{i,s} = \emptyset$. \\
    \While{$|U_{i,s}| \le 8 \ln T \cdot 2^{2s}$ and $t \le T$}{
        Start phase $s$.\\
    \For{$\min(Q_{i, \phi}, T-t)$ iterations}{Pull $i_t = i$ and let $t \gets t + 1$.}
    Observe $\ell'_{i,\phi} := \ell_{i, t}$, append $U_{i,s} \cup \left\{\phi \right\}$, and let $\phi \gets \phi + 1$.
    }
    Calculate the mean $\mu_s(i) := -\frac{1}{|U_{i,s}|}\sum_{\phi' \in U_{i,s}} \ell'_{i, \phi'}$ of the negative of all observations.\\
    Set $\text{LCB}_s(i) = \mu_s(i) - 2^{-s}$ and 
    $\text{UCB}_s(i) = \mu_s(i) + 2^{-s}$.
    }
    For any arm $i \in A$ where $\exists j \in A$ such that $\text{LCB}_s(j) > \text{UCB}_s(i)$, remove $i$ from $A$.  \\
    Increment $s \gets s+1$.
}
\end{algorithm2e}

Since Algorithm \ref{algo:BB-pull-simulated-aae} is exactly \cref{algo:BB-pull-simulated} applied to $\AAE$, we can apply \cref{lem:simulated-indistinguishable} to see that the sequence of arms $\{i_t^{\text{orig}}\}_{t \in [T]}$ pulled by Algorithm \ref{algo:BB-pull-se} is distributed identically to the sequence of arms pulled by $\{i_t^{\text{sim}}\}_{t \in [T]}$ pulled by Algorithm \ref{algo:BB-pull-simulated-aae}.
It thus suffices to analyze Algorithm \ref{algo:BB-pull-simulated-aae} for the remainder of the analysis.

\subsubsection{Lemmas for the analysis of $\BBPull(\AAE)$}
\label{subsubsec:aae-lemmas}

We now show  intermediate results that build on the standard analysis of Active Arm Elimination \citep{even2002pac}. 

We use the following notation in these results.
\begin{enumerate}
    \item Let $S$ be a random variable denoting the maximum value of the variable $s$ reached in Algorithm \ref{algo:BB-pull-simulated-aae} on $\calI$. (That is, $S$ denotes the number of phases that Algorithm \ref{algo:BB-pull-simulated-aae} \textit{begins}.) Note that $S \le T$ with probability 1.
    \item  Let $E_\textrm{loss}$ be the ``clean'' event that at each phase $1 \le s \le S - 1$, for every arm $i \in [K]$, it holds that $\text{LCB}_s(i) \le \bell_i \le \text{UCB}_s(i)$. 
    \item Let the random variable $L_{i,s}$ be equal to the time step $t$ where phase $s$ begins for arm $i$ (i.e. the value of the variable $t$ at line 5 when $U_{i,s}$ is initialized) if that is reached, and otherwise let $L_{i,s}$ be equal to $T+1$.
    \item For each arm $i$, let $E_i^{F}$ be the event that at each phase $1 \le s \le T$, at least one of the following two conditions holds: (1) $L_{i,s} = T+1$, or (2):
    \[\sum_{\phi' \in U_s(i)} Q_{i,\phi'} \le \frac{16 \cdot 2^{2s} \ln T}{f_i}.\]
\end{enumerate} 

First, we show that the clean events occur with high probability.
\begin{lemma}[Correct confidence bounds]
\label{lemma:cf}
    Consider  Algorithm \ref{algo:BB-pull-simulated-aae} evaluated on any given instance $\calI = \left\{\calA, \calF, \calL \right\}$. Let the event $E_\textrm{loss}$ be defined as above.  Then, $\Pr[E_\textrm{loss}] \geq 1 - 2 T^{-3} K$. 
\end{lemma}
\begin{proof}
For each potential phase $1 \le s \le T$ and arm $i$, let $E^{i,s}_\textrm{loss}$ be the event that either $s \ge S$ or $\text{LCB}_s(i) \le \bell_i \le \text{UCB}_s(i)$. Condition on the event that $L_{i,s} \le T$. For ease of analysis, let us also assume that we draw additional loss values, $\ell'_{i, \phi}$ for $T \le \phi \le T + 8 \ln T \cdot 2^{2s}$ i.i.d. from the loss distribution of arm $i$. 

Run the algorithm for $T + 8 \ln T \cdot 2^2s$ time steps rather than $T$ time steps, which ensures that line 11 for $i$ and $s$ is reached and the confidence bounds $\text{LCB}_s(i)$ and $\UCB_s(i)$ are well-defined. We show that $\mathbb{P}[E^{i,s}_\textrm{loss} \mid L_{i,s} \le T] \geq 1 - 2T^{-4}$:
\begin{align*}
\mathbb{P}[E^{i,s}_\textrm{loss}] &\ge  \mathbb{P}[E^{i,s}_\textrm{loss} \mid L_{i,s} \le T] \cdot \mathbb{P}[L_{i,s} \le T]  + \mathbb{P}[L_{i,s} > T] \\ 
&\ge \mathbb{P}[\text{LCB}_s(i) \le \bell_i \le \text{UCB}_s(i)] \cdot \mathbb{P}[L_{i,s} \le T] + \mathbb{P}[L_{i,s} > T] \\ 
&\ge \mathbb{P}[\text{LCB}_s(i) \le \bell_i \le \text{UCB}_s(i)] \\
&= \mathbb{P}\left[\left|\bell_i -\frac{1}{|U_{i,s}|}\sum_{\phi' \in U_{i,s}} \ell'_{i, \phi} \right| \le 2^{-s} \right].
\end{align*}

Recall that we are working with stochastic losses, so $\bell_i -\frac{1}{|U_{i,s}|}\sum_{\phi' \in U_{i,s}} \ell'_{i, \phi'}$ is distributed as an average of $|U_{i,s}| = 8 \ln T \cdot 2^{2s}$ subgaussian random variables with variance $1$. 
Using a Chernoff bound, we have that: 
\[\mathbb{P}\left[\left|\bell_i -\frac{1}{|U_{i,s}|}\sum_{\phi' \in U_{i,s}} \ell'_{i, \phi} \right| > 2^{-s} \right] \le 2 e^{-\frac{8 \ln T \cdot 2^{2s}}{2^{2s+1}}} = 2 T^{-4}. \]

Finally, we apply a union bound to bound $\Pr[E_\textrm{loss}]$. There are $S \leq T$ potential phases and $K$ arms, so there are $KT$ events to union bound over. We see that:
\[
 \Pr[E_\textrm{loss}] \ge \sum_{s=1}^T \sum_{i=1}^K \mathbb{P}[E^{i,s}_\textrm{loss}] \ge  1 - 2 T^{-4} T K = 1 - 2 T^{-3} K.  \]
\end{proof}

\begin{lemma}
\label{lem:se-pulls-ub}
Consider Algorithm \ref{algo:BB-pull-simulated-aae} evaluated on any given instance $\calI = \left\{\calA, \calF, \calL \right\}$  with time horizon $T$. Suppose that the event $E_\textrm{loss}$ holds. Then, the optimal arm $\ist = \arg\min_j \bell_j$ is never removed from $A$. Moreover, at every phase $1 \le s \le S-1$, if $i \in A$ at the end of phase $s$ (i.e. after 13 in Algorithm \ref{algo:BB-pull-simulated-aae}), then
\[\bell_i - \min_j \bell_j \le 4 \cdot 2^{-s}. \]
\end{lemma}
\begin{proof}
Let us condition on $E_{\text{loss}}$, which means that $\text{LCB}_s(i) \le \bell_{i} \le \text{UCB}_s(i)$ for every arm $i$ and every phase $1 \le s \le S-1$.
For the optimal arm $\ist = \arg\min_j \bell_j$ it holds for every phase $s$ that: 
\[\text{UCB}_s(\ist) \ge -\bell_{\ist} = -\min_j \bell_j = \max_j (-\bell_j) \ge \max_j \text{LCB}_s(j), \]
so the optimal arm will never be removed from $A$, as desired.

If arm $i$ is in the active arm set $A$ at the end of phase $s$ (i.e. after line 13), then 
\[\text{UCB}_s(i) \ge \text{LCB}_s(\ist) \ge -\bell_{\ist} - 2 \cdot 2^{-s}.\]
This means that 
\[-\bell_i \ge \text{LCB}_s(i) \ge \text{UCB}_s(i) - 2 \cdot 2^{-s} \ge -\bell_{\ist} - 4 \cdot 2^{-s} = - \min_j \bell_j - 4 \cdot 2^{-s}.   \]
Rearranging, we obtain that:
\[\bell_i - \min_j \bell_j \le 4 \cdot 2^{-s}. \]
as desired. 
.
\end{proof}

\begin{lemma}
\label{lemma:ef}
Consider  Algorithm \ref{algo:BB-pull-simulated-aae} evaluated on any given instance $\calI = \left\{\calA, \calF, \calL \right\}$. For each arm $i$, let $E_i^F$ be defined as above. Then, $\Pr[E_i^{F}] \geq 1 - T^{-4}$.
\end{lemma}
\begin{proof}
For each arm $i$ and each phase $1 \le s \le T$, let $E_{i,s}^F$ be the event that 
\[\sum_{\phi' \in U_s(i)} Q_{i,\phi'} \le \frac{8 \cdot 2^{2s} \ln T}{f_i}.\] 
We lower bound the probability $\mathbb{P}\left[E_{i,s}^F\right]$. We analyze  $\mathbb{P}\left[\sum_{\phi' \in U_s(i)} Q_{i,\phi'}  \le \frac{16 \cdot 2^{2s} \ln T}{f_i}\right]$ as follows. Let $m =\frac{16 \cdot 2^{2s} \ln T}{f_i}$. By definition, the probability that $\sum_{\phi' \in U_s(i)} Q_{i,\phi'}  > m$ is equal to the probability that fewer than $8 \cdot 2^{2s}\cdot \ln T$ successes are observed after $m$ i.i.d. Bernouilli trials with parameter $f_i$. This probability can be analyzed with a Chernoff bound. In particular, let $Z_j \sim \text{Bern}(f_i)$ for $1 \le j \le m$ be a sequence of $m$ i.i.d. random variables. Using the multiplicative Chernoff bound, we see that:
\begin{align*}
  \mathbb{P}[\sum_{\phi' \in U_s(i)} Q_{i,\phi'}  > \frac{16 \cdot 2^{2s} \ln T}{f_i}] &= \Pr\left[\sum_{j=1}^{m} Z_j < 8 \cdot 2^{2s} \cdot \ln T
  \right]  
  \\&\leq \Pr\left[
  \sum_{j=1}^m Z_j < m \cdot f_i \cdot 0.5
    \right]
    \\ 
    &\leq \exp\left(-m \cdot f_i \cdot \frac{1}{8}\right) \\
    &= \exp \left(- \frac{1}{f_i} \cdot 16 \cdot 2^{2s} \cdot \ln T \cdot f_i \cdot \frac{1}{8}\right) 
    \\&= T^{-2^{2s + 1}}
    \\&\leq T^{-5}.
\end{align*}

This implies that $\mathbb{P}[E_{i,s}^F] \ge 1 - T^{-5}$.
Union bounding over the $T$ values of $s$, we obtain that $\Pr[E_i^{F}] \geq 1 - T^{-4}$. 

\end{proof}

\subsubsection{Regret of $\BBPull(\AAE)$: Proof of \cref{thm:bbpullregretaaeucb}}
\label{appendix:regretAE}

Here, we prove the regret bound for $\BBPull(\AAE)$.
For convenience, we restate \cref{thm:bbpullregretaaeucb} below. 
\Regretbbpullaaeucb*

We will prove the statement of \cref{thm:bbpullregretaaeucb} only for $\BBPull(\AAE)$. 
In the proof of the regret bounds, we will use the following lemma.
\begin{lemma}
\label{lemma:regretboundslemma}
Let $\Delta_1, \Delta_2, \ldots, \Delta_K \ge 0$ be a sequence of nonnegative numbers. Let $N_1, \ldots, N_K \ge 0$ be a sequence of nonnegative numbers such that for some $C > 0$, it holds that $N_i \le \frac{C \cdot \ln T}{\Delta_i^2 f_i}$ for all $1 \le i \le K$.  Then the following two bounds hold:
\[\sum_{1 \le i \le K \mid \Delta_i > 0} \Delta_i \cdot N_i \le \sum_{1 \le i \le K \mid \Delta_i > 0} \frac{C \ln T}{\Delta_i f_i} \]
and 
\[\sum_{1 \le i \le K \mid \Delta_i > 0} \Delta_i \cdot N_i \le \sqrt{C T\ln(T)\sum_{j=1}^K \frac{1}{f_j}}. \]
\end{lemma}
\begin{proof}
The first bound follows from:
\[\sum_{1 \le i \le K \mid \Delta_i > 0} \Delta_i \cdot N_i \le \sum_{1 \le i \le K \mid \Delta_i > 0} \Delta_i \cdot \frac{C \ln T}{\Delta_i^2 f_i}  \Delta_i =  \sum_{1 \le i \le K \mid \Delta_i > 0} \frac{C \ln T}{\Delta_i f_i}\]
For the second bound, first we rearrange the upper bound on $N_i$ into: 
\[\Delta_i \le \sqrt{\frac{C \ln T}{N_i f_i}}.\]
Now, we see that
\begin{align*}
  \sum_{1 \le i \le K \mid \Delta_i > 0} N_i \Delta_i  &\le \sum_{1 \le i \le K \mid \Delta_i > 0}\sqrt{\frac{C N_i \ln (T)}{f_i}} \\
  &=  \sum_{1 \le i \le K \mid \Delta_i > 0} \frac{1}{f_i} \sqrt{N_i \ln(T)f_i} \\
  &\le \sum_{1 \le i \le K} \frac{1}{f_i} \sqrt{N_i \ln(T)f_i} \\
  &=  \left(\sum_{j=1}^K \frac{1}{f_j}\right) \sum_{i=1}^K \frac{\frac{1}{f_i}}{ \left(\sum_{j=1}^K \frac{1}{f_j}\right)} \sqrt{C N_i \ln(T)f_i} \\
  &\le_{(1)} \left(\sum_{j=1}^K \frac{1}{f_j}\right) \sqrt{C \sum_{i=1}^K \frac{\frac{1}{f_i}}{ \left(\sum_{j=1}^K \frac{1}{f_j}\right)} N_i \ln(T)f_i} \\
  &= \sqrt{\sum_{j=1}^K \frac{1}{f_j}} \sqrt{C \sum_{i=1}^K N_i \ln(T)} \\
  &\le_{(2)}  \sqrt{\sum_{j=1}^K \frac{1}{f_j}} \sqrt{C T\ln(T)} 
\end{align*}
where (1) follows from Jensen's inequality and (2) follows from the fact that $\sum_{i=1}^K N_i = T$.    
\end{proof}

We are now ready to prove Theorem \ref{thm:bbpullregretaaeucb} for $\BBPull(\AAE)$. 
\begin{proof}[Proof of Theorem \ref{thm:bbpullregretaaeucb} for $\BBPull(\AAE)$]

By \cref{lem:simulated-indistinguishable}, the sequence of arms $\{i_t^{\text{orig}}\}_{t \in [T]}$ pulled by Algorithm \ref{algo:BB-pull-se} is distributed identically to the sequence of arms pulled by $\{i_t^{\text{sim}}\}_{t \in [T]}$ pulled by Algorithm \ref{algo:BB-pull-simulated-aae}. Define the event $E$ to be $E:= E_{\text{loss}} \cap E_F^1 \ldots E_F^K$ where the events are defined as in Lemma \ref{lemma:cf} and Lemma \ref{lemma:ef}.
Union bounding, $E$ occurs with probability at least $1 - 2 T^{-3} K - K T^{-4}$. When $T$ is sufficiently large, $\mathbb{P}[E] \ge 1 - T^{-2}$, so the event that $E$ does not occur contributes negligibly to the regret. Let us condition on $E$ for the remainder of the analysis.

For each arm $i$, let $\Delta_i= \bell_i - \min_j \bell_j$ be the suboptimality gap. Let $N_i$ be the number of time steps where arm $i$ is pulled over the course of Algorithm \ref{algo:BB-pull-simulated-aae}.
The regret is equal to:
\[\sum_{1 \le i \le K \mid \Delta_i > 0} \Delta_i \cdot N_i.\]

We first show that if $\Delta_i > 0$, then arm $i$ is pulled at most $O\left( \frac{\log T}{\Delta^2_i f_i }\right)$ times. By Lemma \ref{lem:se-pulls-ub}, arm $i$ must be eliminated after phase $\ceil{\log_2\left( 4/\Delta_i\right)}$. For phases $1 \le s \le \ceil{\log_2\left( 4/\Delta_i\right)}$, recall that we have defined the random variable $L_{i,s}$ to be equal to the time step $t$ where phase $s$ begins (i.e. the value of the variable $t$ at line 5 when $U_{i,s}$ is initialized) if that is reached, and otherwise let $L_{i,s}$ be equal to $T+1$. 
This means that arm $i$ is pulled at most:
\begin{align*}
N_i &\le \sum_{s=1}^{\ceil{\log\left( 4/\Delta_i\right)}} \min\left(\sum_{\phi' \in U_s(i)} Q_{i,\phi}, T-(L_{i,s} - 1)\right) \\
&\le \sum_{s=1}^{\ceil{\log\left( 4/\Delta_i\right)}} \sum_{\phi' \in U_s(i)} Q_{i,\phi} \\
&\le_{(1)} \sum_{s=1}^{\ceil{\log\left( 4/\Delta_i\right)}} \frac{16 \cdot 2^{2s} \ln T}{f_i} \\
&\le  16 \cdot \frac{\ln T}{f_i} \sum_{s=1}^{\ceil{\log\left( 4/\Delta_i\right)}} 2^{2s} \\
&\le \frac{C \cdot \ln T}{\Delta_i^2 f_i}
\end{align*}
for some universal constant $C > 0$, where (1) follows from the event $E_i^F$ holding.

The instance-dependent and instance-independent regret bounds now both follow from Lemma \ref{lemma:regretboundslemma}.

\end{proof}

\subsubsection{Monotonicity of $\BBPull(\AAE)$: Proof of \cref{thm:feedbackmonotonicitybbpullaae}}
\label{appendix:monotonicityAAE}

We prove Theorem \ref{thm:feedbackmonotonicitybbpullaae}, restated below. 
\Monobbpullaae*

The intuition is that we can leverage the structure of $\AAE$ to refine the analysis in Theorem \ref{thm:feedbackmonotonicitybbpull}. In particular, in the proof of Theorem \ref{thm:feedbackmonotonicitybbpull}, the difference in feedback observations came from the fact that the time horizons $\Phi$ and $\tilde{\Phi}$ were different (that is, the number of calls to $\ALG$ differed for the two instances). In contrast, in the proof of Theorem \ref{thm:feedbackmonotonicitybbpullaae}, we take advantage of a key structural property of $\AAE$: we can upper bound the number of phases until arm $i$ is guaranteed to be eliminated. By assuming that $T$ is sufficiently large, we can guarantee that the algorithm will reach this phase on both instances and thus the arm will be eliminated. We formalize this using a coupling argument similar to the proof of Theorem \ref{thm:feedbackmonotonicitybbpull}, but that leverages the structure of $\AAE$. 

\begin{proof}[Proof of \Cref{thm:feedbackmonotonicitybbpullaae} (Monotonicity for $\BBPull(\AAE)$)]

Like in Theorem \ref{thm:feedbackmonotonicitybbpull}, we will analyze $\BBPull(\ALG)$ by comparing the behavior of \cref{algo:BB-pull-simulated-aae} on $\calI$ and $\widetilde\calI$ in three steps; the main modification is in Step 3 below, where we condition on clean events specific to $\AAE$. 
\begin{enumerate}
    \item We construct a probability coupling between the sequence of random variables $Q_{j,\phi}$ and $\widetilde{Q}_{j,\phi}$ for $j\in [K]$ and $\phi = 1, \dots, \infty$. This coupling ensures that 
    $Q_{i,\phi} \geq \widetilde{Q}_{i,\phi}$ for arm $i$ and $Q_{j,\phi}= \widetilde{Q}_{j,\phi}$ for all other arms $j\neq i$, for all $\phi$. 
    \item We call \cref{algo:BB-pull-simulated} on $\calI$ and $\widetilde\calI$
    with $Q_{j,\phi}$ and  $\widetilde{Q}_{j,\phi}$ for $j\in [K]$ and $\phi = 1, \dots, \infty$, respectively. We use \cref{lem:joint-losses-identical} to argue that for any $\Phi^*$, $(i^\ALG_1, \dots, i^\ALG_{\Phi^*})$ is identically distributed to $(\tilde i^\ALG_1, \dots, \tilde i^\ALG_{\Phi^*})$; then, we couple the arm pulls on each instance so $(i^\ALG_1, \dots, i^\ALG_{\Phi^*}) = (\tilde i^\ALG_1, \dots, \tilde i^\ALG_{\Phi^*})$.
    \item By this step, random variables $Q_{j,\phi}$, $\widetilde{Q}_{j,\phi}$,  $i^\ALG_\phi$ and $\tilde i^\ALG_\phi$ are coupled as described above. Let $E$ be the event that  $E_{\text{loss}} \cap E^F_1  \cap \dots  \cap E^F_{i-1} \dots E^F_{i+1} \cap \dots \cap E^F_K$ holds (these events are defined in \Cref{subsubsec:aae-lemmas}). We condition on the event $E$ and analyze $\FOC$. We then analyze $\APC$. 
\end{enumerate}

\paragraph{Step 1: Coupling realizations of feedback observations.} We couple the distributions over the feedback observations in the same way as in the proof of Theorem \ref{thm:feedbackmonotonicitybbpull}. Note that for $\tilde f_{i} > f_i$, the distribution of $\widetilde{Q}_{i, \phi}$ is stochastically dominated by $Q_{i, \phi}$. 
Therefore, there is a joint probability distribution over $(Q_{j,\phi}, \widetilde{Q}_{j,\phi})$ such that for all $\phi$, with probability $1$ the following holds: $Q_{i,\phi} \geq \widetilde{Q}_{i,\phi}$ and for all $j\neq i$, $Q_{j,\phi}= \widetilde{Q}_{j,\phi}$.
This also gives us a coupling, that is a joint distribution, over $\left( \{Q_{j,\phi}\}_{j\in [K], \phi \in \{1, \dots, \infty\}}, \{ \widetilde{Q}_{j,\phi} \}_{j\in [K], \phi\in\{1, \dots, \infty\}}\right)$ that meets the aforementioned property.

\paragraph{Step 2: Coupling arms pulled by $\ALG$ across instances $\calI$ and $\widetilde{\calI}$.} We couple the arms in the same way as in the proof of Theorem \ref{thm:feedbackmonotonicitybbpull}. We condition on the sequences $Q_{j,\phi}$ and  $\widetilde{Q}_{j,\phi}$ for $j\in [K]$ and $\phi = 1, \dots, \infty$, respectively, as coupled in in Step 1, and we apply \cref{lem:joint-losses-identical}. As before, the preconditions of this lemma are met for any $\Phi^*$, so we have that $(i^\ALG_1, \dots, i^\ALG_{\Phi^*})$ and $(\tilde i^\ALG_1, \dots, \tilde i^\ALG_{\Phi^*})$ are identically distributed. This allows us to consider a joint probability distribution over $(i^\ALG_1, \dots, i^\ALG_{\Phi^*}, \tilde i^\ALG_1, \dots, \tilde i^\ALG_{\Phi^*})$ such that 
$i^\ALG_\phi = \tilde{i}^\ALG_\phi$ for all $\phi \in [\Phi^*]$.

\paragraph{Step 3: Condition on $E$.} 
 We use the coupling thus far with the property that $Q_{j,\phi}$, $\widetilde{Q}_{j,\phi}$, $i^\ALG_\phi$, $\tilde i^\ALG_\phi$ are all fixed and for all $\phi = 1, \cdots, \infty$ satisfy 
$i^\ALG_\phi = \tilde{i}^\ALG_\phi$,
$Q_{i,\phi} \geq \widetilde{Q}_{i,\phi}$, and $Q_{j,\phi}= \widetilde{Q}_{j,\phi}$ for $j\neq i$. Moreover, we condition on the event $E = E_{\text{loss}} \cap E^F_1  \cap \dots  \cap \dots \cap E^F_K$ holds on $\calI$ (these events are defined in \Cref{subsubsec:aae-lemmas}).

\paragraph{Notation for Analyzing $\FOC$.} To formalize these claims, we introduce the following additional notation. Let $\FOC_i^{[T]}(\calI)$  be the number of times feedback is observed on arm $i$  in timesteps in $R$ on $\widetilde \calI$, and let $\FOC_i^{[T]}(\widetilde \calI)$  be the number of times feedback is observed on arm $i$  in timesteps on $\tilde \calI$. Since we have conditioned on $Q_{j,\phi}$, $\widetilde{Q}_{j,\phi}$, $i^\ALG_\phi$, $\tilde i^\ALG_\phi$, we see that at this point $\FOC_i^{[T]}(\calI)$ and $\FOC_i^{[T]}(\widetilde \calI)$ are both deterministic.

\paragraph{Analyzing $\FOC$.} 
We first show that arm $i$ will be eliminated on both instances before the end of the time horizon. By Lemma \ref{lem:se-pulls-ub}, from phase $s \ge s' := 3 - \log(\Delta_i)$ onwards, the arm $i$ is guaranteed to not be pulled. Using events $E_j^F$, we see that phase $s' = 3-\log(\Delta_i)$ must be reached on $\calI$ within the following number of time steps: 
\begin{align*}
  \sum_{s=1}^{s'-1} \sum_{i' \in [K]} \sum_{\phi' \in U_s(i')} Q_{i', \phi}  &\le \sum_{s=1}^{s'-1} \sum_{i' \in [K]} \frac{16 \cdot 2^{2s} \ln T}{f_{i'}} \\
  &= (16 \ln T) \left(\sum_{i' \in [K]} \frac{1}{f_{i'}} \right) \sum_{s=1}^{s'-1} 4^s \\
  &= \frac{16 \cdot (4^{s'} - 1) \ln T}{3} \left(\sum_{i' \in [K]} \frac{1}{f_{i'}} \right) \\
  &= \frac{16 \cdot 4^{3 - \log(\Delta_i)} \ln T}{3} \left(\sum_{i' \in [K]} \frac{1}{f_{i'}} \right),
\end{align*}
which grows logarithmically in $T$. Thus, for sufficiently large $T$, $\frac{16 \cdot 4^{3 - \log(\Delta_i)} \ln T}{3} \left(\sum_{i' \in [K]} \frac{1}{f_{i'}} \right) \le T$, which means that phase $s'$ will be reached on $\calI$ within a time horizon of $T$. For $\tilde \calI$, we use the fact that $\tilde{Q}_{j, \phi} \le Q_{j, \phi}$ in our coupling and moreover $i^\ALG_\phi = \tilde{i}^\ALG_\phi$ for all $\phi \in [\Phi^*]$, so phase $s'$ will be reached on $\tilde{\cal I}$ as well within a time horizon of $T$.

We are now ready to analyze $\FOC$. We observe that: 
\begin{align*}
\FOC_i^{[T]}(\calI) &= 
 \sum_{\phi=1}^{\Phi} \mathds 1[i_\phi = i] \cdot \mathds 1 \left[\sum_{\phi'=1}^{\phi}  Q_{i_{\phi'}, \phi'} \le T \right] \\
 &= 
 \sum_{s'=1}^{3 - \log(\Delta_i)} \sum_{\phi \text{ in phase } s} \mathds 1[i_{\phi} = i] \\
 &= \sum_{\phi=1}^{\tilde{\Phi}} \mathds 1[i_\phi = i] \cdot \mathds 1 \left[\sum_{\phi'=1}^{\phi}  \tilde{Q}_{i_{\phi'}, \phi'} \le T \right] \\
 &= \FOC_i^{[T]}(\tilde \calI).   
\end{align*}

Applying the law of total expectation, taking an expectation over the coupled random variables $Q_{j,\phi}$, $\widetilde{Q}_{j,\phi}$, $i^\ALG_\phi$, $\tilde i^\ALG_\phi$, we see that, conditioned on $E$, 
\[\FOC_i(\tilde \calI) - \FOC_i(\calI) = \E\left[\FOC_i^{[T]}(\widetilde{\calI}) - \FOC_i^{[T]}(\calI)\right] \ge 0.  \]
The event that $E$ does not hold contributes negligibly (i.e., at most $1/T$) to both $\FOC(\mathcal I)$ and $\FOC(\tilde{\mathcal I})$.
Taking expectations and including the possibility of $1/T$ error from the event $E$ not holding, we obtain that:
\[|\FOC_i(\mathcal{I}) -  \FOC_i(\tilde{\mathcal{I}})| \le 1/T,\]
as desired.

\paragraph{Analyzing $\APC$.}
For $\APC$, the above result implies that 
\[\FOC_i(\widetilde\calI) < \FOC_i(\calI) + 1/T. \]
Applying \cref{prop:relationship}, we have that:
\[ \APC_i( \widetilde\calI) < \APC_i(\calI) \frac{f_i}{\widetilde{f}_i} + \frac{1}{T \widetilde{f}_i}. \] 
Recall that $\widetilde f_i > f_i$, so the RHS above is less than $\APC_i(\calI)$ as long as $T > \frac{1}{\APC_i(\widetilde\calI)(\widetilde f_i - f_i)}$.
By the definition of \cref{algo:BB-pull-simulated}, every arm must be pulled at least once. Thus, for sufficiently large $T$, we see that \[\APC_i(\widetilde\calI) < \APC_i(\calI)\] as desired.

\end{proof}

\subsection{Analysis of $\BBPull(\UCB)$: Proof of \cref{thm:bbpullregretaaeucb}}
\label{appendix:UCB}

Here, we prove the regret bound of \cref{thm:bbpullregretaaeucb} for $\BBPull(\UCB)$.
For convenience, we restate \cref{thm:bbpullregretaaeucb} below. 
\Regretbbpullaaeucb*

We consider the simulated version of $\BBPull(\UCB)$ given by Algorithm \ref{algo:BB-pull-simulated} applied to $\UCB$. For convenience, we explicitly state this algorithm below (Algorithm \ref{algo:BB-pull-simulated-ucb}).

Let us define the same random variables as those used in Algorithm \ref{algo:BB-pull-simulated}, restated for convenience.  (Recall that $\phi$ indexes losses for the time horizon of $\ALG$, $\Phi$ is the total number of times $\ALG$ is called by $\BBPull(\ALG)$, and $\Phi \leq T$ because $\ALG$ can be called at most $T$ times.)

\begin{itemize}
    \item \textit{Losses:} For each round $\phi \in [\Phi]$ of $\ALG=\UCB$ and each arm $j \in [K]$, 
    let $\ell'_{j, \phi} := \ell_{j, t}$ be the loss for arm $j$ at a time step $t$ that corresponds to the last time step in block $\phi$ of $\BBPull(\UCB)$. Since we are in the stochastic loss setting, $\ell'_{j,\phi}$ is a random variable drawn from the distribution of arm $j$ (with mean $\bar{\ell}_j$) independently across $\phi$ and $j$. 
    \item \textit{Feedback realizations:} For all $j \in [K]$ and $\phi \in [T]$, let $Q_{j,\phi} \sim \textrm{Geom}(f_j)$ for $\phi \in [T]$ be a random variable distributed according to the geometric distribution with parameter equal to the feedback probability of arm $j$. 
    (These random variables are also fully independent across values of $j$ and $\phi$.) 
\end{itemize}

We are now ready to present Algorithm \ref{algo:BB-pull-simulated-ucb}. For ease of analysis, we define the lower confidence bounds $\text{LCB}$ within the algorithm, even though the algorithm does not ever use these quantities. 

\begin{algorithm2e}[htbp]
\caption{Simulated version of $\BBPull(\UCB)$}
\label{algo:BB-pull-simulated-ucb}
\DontPrintSemicolon
\LinesNumbered
Initialize number of pulls $n_i = 0$ for all $i \in [K]$. \\
Initialize empirical mean $\mu(i) = 0$ for all $i \in [K]$. \\
Initialize $t = 1$ and $\phi = 1$.
\\
\While{$t \le T$}{
    Initialize $i_\phi = 0$. \\
    \uIf{$n_i = 0$ for any arm $i \in [K]$}{Let $i_\phi$ be the arm with the smallest index such that $n_{i_{\phi}} = 0$.}
    \Else{For every arm $i \in [K]$, compute $\text{UCB}(i) = \mu(i) + \sqrt{\frac{6 \ln T}{n_i}}$ and $\text{LCB}(i) = \mu(i) - \sqrt{\frac{6 \ln T}{n_i}}$. \\
    Let $i_\phi = \text{argmax}_{j \in [K]}\text{UCB}(j)$.}
    \For{$\min(Q_{i_\phi, \phi}, T-t)$ iterations}{Pull $i_\phi = i$ and let $t \gets t + 1$.}
    Observe $\ell'_{i_\phi,\phi} := \ell_{i_\phi, t}$. \\
    Update the empirical mean $\mu(i) \gets \frac{n_{i_\phi} \cdot \mu(i)}{n_{i_\phi}+1} - \frac{\ell'_{i_\phi,\phi}}{n_{i_\phi}+1} $.\\
    Increment $n_{i_\phi} \gets n_{i_\phi} + 1$. \\
    Increment $\phi \gets \phi + 1$. 
}
\end{algorithm2e}

Since Algorithm \ref{algo:BB-pull-simulated-ucb} is exactly \cref{algo:BB-pull-simulated} applied to $\AAE$, we can apply \cref{lem:simulated-indistinguishable} to see that the sequence of arms $\{i_t^{\text{orig}}\}_{t \in [T]}$ pulled by Algorithm \ref{algo:BB-pull-ucb} is distributed identically to the sequence of arms pulled by $\{i_t^{\text{sim}}\}_{t \in [T]}$ pulled by Algorithm \ref{algo:BB-pull-simulated-ucb}.

\subsubsection{Lemmas for the analysis of $\BBPull(\UCB)$}

We now show the following intermediate results that build on the standard analysis of UCB \citep{auer2002finite}. 
First, we see immediately that for $1 \le \phi \le K$, the \textit{if} statement on line 6 of Algorithm \ref{algo:BB-pull-simulated-ucb} will be met, so $i_{\phi=1} = 1$, $i_{\phi=2} = 2, \ldots, i_{\phi=K} = K$. We will handle the regret from these rounds ($\phi = 1 \dots K$) separately.

We define the following two clean events. 
\begin{enumerate}
    \item First, recall that $\Phi$ is the maximum value of $\phi$ realized by Algorithm \ref{algo:BB-pull-simulated-ucb}. Let $E_{\text{UCB,loss}}$ be the ``clean'' event that at each round $K + 1 \le \phi \le \Phi$, for every arm $i \in [K]$, it holds that $\text{LCB}(i) \le \bar{\ell}_i \le \text{UCB}(i)$ at line 9 for round $\phi$.
    \item Let the random variable $L_\phi$ be equal to the time step $t$ where round $\phi$ begins (i.e. the value of the variable $t$ at line 5 when $i_\phi$ is initialized.) if that is reached, and otherwise let $L_\phi$ be equal to $T+1$.
    For each arm $i$ and any value $M_i \ge 0$    
    let $E_{i, M_i}^{F, \text{UCB}}$ be the event that 
    \[\sum_{\phi=1}^{\Phi} \min(Q_{i, \phi}, T- (L_\phi-1)) \cdot \mathds 1[i_\phi = i] \le 
    \frac{6 \cdot M_i}{f_i}.\]
\end{enumerate}

\begin{lemma}
\label{claim:cleanevent}
    Consider  Algorithm \ref{algo:BB-pull-simulated-aae} evaluated on any given instance $\calI = \left\{\calA, \calF, \calL \right\}$. Let the event $E_\text{UCB, loss}$ be the ``clean'' event defined above. Then, $\Pr[E_{\text{UCB,loss}}] \geq 1 - 2 T^{-3} K$. 
\end{lemma}
\begin{proof}
Consider arm $i \in [K]$ and potential number of arm pulls $1 \le n \le T$. Let $E^{i,n}_{\textrm{UCB, loss}}$ be the event that either $n_i = n$ is not reached by the algorithm or $\text{LCB}(i) \le \bar{\ell}_i \le \text{UCB}(i)$ at line 9 when $n_i = n$. Let $\tilde{\mu}_n(i)$ be the empirical mean of $n$ i.i.d. samples from the loss distribution for arm $i$. Following the standard analysis of UCB confidence sets, we see: 
\[
  \mathbb{P}[E^{i,n}_{\textrm{UCB, loss}}] \ge \mathbb{P}\left[|\tilde{\mu}_n(i) - \bar{\ell_i}| \le \sqrt{\frac{6 \ln T}{n}} \right] \ge 1 - 2 e^{\frac{6 n \ln T}{2 n}} = 1 - 2T^{-3}.
\]
We union bound over $1 \le n \le T$ and $i \in [K]$ to obtain $\Pr[E_{\text{UCB,loss}}] \geq 1 - 2 T^{-3} K$. 
\end{proof}

\begin{lemma}
\label{claim:enoughsamples}
Consider  Algorithm \ref{algo:BB-pull-simulated-aae} evaluated on any given instance $\calI = \left\{\calA, \calF, \calL \right\}$, and consider $i \in \calA$. Let $M_i$ be such that $\mathbb{P}[\sum_{\phi=1}^{\Phi} \mathds 1[i_\phi = i] \ge M_i] \le 2T^{-3}K$ and $M_i \ge 6 \ln T$. Let $E_{i, M_i}^{F, \text{UCB}}$ be defined as above. Then it holds that $\mathbb{P}[E_{i, M_i}^{F, \text{UCB}}] \ge 1 - T^{-4} - 2T^{-3}K$.
\end{lemma}
\begin{proof}
For the sake of this proof, let's assume that we realize $2T$ random variables $Q_{i, \phi}$ for $1 \le \phi \le 2T$ instead of $T$ random variables.

For $1 \le n \le T$ such that $\sum_{\phi=1}^{\Phi} \mathds 1[i_\phi = i] \ge n$, let $\Phi_{n,i}$ be equal to minimum value $\phi' \ge 1$ such that $\sum_{\phi=1}^{\phi'} 1[i_\phi = i] = n$ (that is, the time step $\phi$ at which that arm $i$ is pulled by $\ALG=\UCB$ for the $n$th time). For $1 \le n \le T$ such that $\sum_{\phi=1}^{\Phi} \mathds 1[i_\phi = i] < n$ (i.e. the arm is pulled by $\ALG = \UCB$ less than $n$ times), for technical convenience, let $\Phi_{n,i} = T + n \le 2T$. Observe that:
\[\sum_{\phi=1}^{\Phi} \min(Q_{i, \phi}, T- (L_\phi-1)) \cdot \mathds 1[i_\phi = i] \le \sum_{\phi=1}^{\Phi} Q_{i, \phi} \cdot \mathds 1[i_\phi = i] = \sum_{n \ge 1 \text{ s.t. } \sum_{\phi=1}^{\Phi} \mathds 1[i_\phi = i] \ge n} Q_{i, \Phi_{n,i}}.\] Moreover, by the definition of $M_i$, we see that
\[\mathbb{P}\left[\sum_{n \ge 1 \text{ s.t. } \sum_{\phi=1}^{\Phi} \mathds 1[i_\phi = i] \ge n} Q_{i, \Phi_{n,i}} \le \sum_{n=1}^{M_i} Q_{i, \Phi_{n,i}}\right] \ge 1 - T^{-2}. \] We thus focus on bounding 
\[\mathbb{P}\left[\sum_{n=1}^{M_i} Q_{i, \Phi_{n,i}} > \frac{6  M_i}{f_i}  \right].\]
It is easy to see that $Y := \sum_{n=1}^{N_i} Q_{i, \Phi_{n,i}}$ is distributed as the number of Bernoulli trials with parameter $f_i$ needed to observe $M_i$ successes. We can analyze the probability $\mathbb{P}[Y > \frac{6  M_i}{f_i}]$ as follows. By definition, this is equal to the probability that fewer than $M_i$ successes are observed after $\frac{6 M_i}{f_i}$ Bernoulli trials with parameter $f_i$. If we let $Z_j$ denote i.i.d. Bernoullis with parameter $f_i$, this probability can be analyzed by a multiplicative Chernoff bound:
\[\mathbb{P}\left[\sum_{n=1}^{M_i} Q_{i, \Phi_{n,i}} > \frac{6 M_i}{f_i}  \right] =  \mathbb{P}[Y > \frac{6 M_i}{f_i}] = \mathbb{P}\left[\sum_{j=1}^{6 M_i/f_i} Z_j \le M_i\right] \le T^{-4},\]
where we use that $M_i \ge 6 \ln T$. 

Union bounding, we obtain that $\mathbb{P}[E_{i, m_i}^{F, \text{UCB}}] \ge 1 - T^{-4} - T^{-2}$. 
\end{proof}

\begin{lemma}
\label{lemma:boundpullsBBPullUCB}
Consider  Algorithm \ref{algo:BB-pull-simulated-aae} evaluated on any given instance $\calI = \left\{\calA, \calF, \calL \right\}$. If the event $E_\textrm{UCB, loss}$ holds, then $\sum_{\phi=1}^{\Phi} \mathds 1[i_\phi = i] \le \frac{6 \ln T}{\Delta_i^2}$.
\end{lemma}
\begin{proof}This follows from the standard analysis of UCB. Let us condition on $E_\textrm{UCB, loss}$, and let $i^*$ be the arm with optimal mean loss. If $i_\phi = i$ and $\sum_{\phi'=1}^{\phi - 1} \mathds 1 [i_{\phi'} = i] = n$, then it must hold that $-\bar{\ell}_i + \sqrt{\frac{6 \ln T}{n}} = \text{UCB}(i) \ge \text{UCB}(i^*) \ge -\bar{\ell}_{i^*}$. Solving for $n$, we obtain that 
\[n \le \frac{6 \ln T}{\Delta_i^2}. \]
\end{proof}

Now, we are ready to prove Theorem \ref{thm:bbpullregretaaeucb}. 
\begin{proof}[Proof of Theorem \ref{thm:bbpullregretaaeucb} for $\BBPull(\UCB)$] 
By \cref{lem:simulated-indistinguishable}, the sequence of arms $\{i_t^{\text{orig}}\}_{t \in [T]}$ pulled by Algorithm \ref{algo:BB-pull-se} is distributed identically to the sequence of arms pulled by $\{i_t^{\text{sim}}\}_{t \in [T]}$ pulled by Algorithm \ref{algo:BB-pull-simulated-aae}. Let $M_i = \frac{6 \ln T}{\Delta_i^2}$ for $i \in [K]$ and let the event $E$ be defined to be $E_{\text{UCB, loss}} \cap E^{\text{F, UCB}}_{1, M_1} \cap \ldots E^{\text{F, UCB}}_{1, M_K}$.
We apply \Cref{claim:cleanevent}, \Cref{claim:enoughsamples}, and \Cref{lemma:boundpullsBBPullUCB} to see that $E$ occurs with probability at least $1 - 2 T^{-3} K - 2 T^{-3} K^2 - KT^{-4}$. When $T$ is sufficiently large, $\mathbb{P}[E] \ge 1 - T^{-2}$, so the event that $E$ does not occur contributes negligibly to the regret. Let us condition on $E$ for the remainder of the analysis.

For each arm $i$, let $\Delta_i= \bell_i - \min_j \bell_j$ be the suboptimality gap. Let $M_i$ be the number of time steps where arm $i$ is pulled over the course of the algorithm.
The regret is equal to:
\[\sum_{1 \le i \le K \mid \Delta_i > 0} \Delta_i \cdot M_i.\]
We first observe by event $E^{\text{F, UCB}}_{i, M_i}$ that if $\Delta_i > 0$, then arm $i$ is pulled at most $\frac{36 \ln T}{\Delta_i^2 f_i}$ times. The instance-independent and instance-dependent regret bounds now follow from Lemma \ref{lemma:regretboundslemma}.

\end{proof}

\subsection{Analysis of $\BBDivideAdjusted(\AAE)$}
\label{appendix:bbdaaae}

We prove the monotonicity properties of $\BBDivideAdjusted(\AAE)$.As before, we also construct a simulated version of \cref{algo:BB-da-se}. We formalize this simulated version in \cref{algo:BB-da-se-sim}. 

Let us define the following random variables.  (Recall that $\phi$ indexes losses for the time horizon of $\ALG$, $\Phi$ is the total number of times $\ALG$ is called by $\BBPull(\ALG)$, and $\Phi \leq T$ because $\ALG$ can be called at most $T$ times.)

\begin{itemize}
    \item \textit{Losses:} 
    For each round $\phi \in [T]$ of $\ALG=\AAE$ and each arm $j \in [K]$, 
    let $\ell'_{j, \phi}$ denote a stochastic loss sampled from the distribution of arm $j$ (with mean $\bell_j$). These random variables are fully independent across values of $j$ and $\phi$. Note that unlike in \cref{algo:BB-pull-simulated-aae} or \ref{algo:BB-pull-simulated-ucb}, it is not guaranteed that $\ell^\AAE_{j,\phi}$ observed by $\AAE$ is $\ell'_{j,\phi}$, because with a fixed block size, there will always be some likelihood that no feedback is observed.
    \item \textit{Feedback probabilities:} Let $U_{j, \phi} \sim \text{Bern}(1 - (1-f_j)^{B_j})$ for $j \in [K]$ and $\phi \in [T]$ denote the indicator variable for whether feedback will be observed in block $\phi$, where $B_j = \ceil{\frac{3\ln T}{\min_if_i}(1 + f_j)}$. (These random variables are also fully independent across values of $j$ and $\phi$.) 
\end{itemize}

Again, note that \cref{algo:BB-da-se-sim} is a direct application of \cref{algo:BB-da-simulated} to $\AAE$ (lines 7-11 in \cref{algo:BB-da-se-sim} reflect \cref{algo:BB-da-simulated}, while the rest are for $\ALG = \AAE$). This allows us use \cref{lem:sim-indistinguishable-bbda} directly to argue that the arms selected by \cref{algo:BB-da-se-sim} are distributed identically to those selected by \cref{algo:BB-da-se}. 

\begin{algorithm2e}[htbp]
\caption{Simulated version of $\BBDivideAdjusted(\text{AAE})$}
\label{algo:BB-da-se-sim}
\DontPrintSemicolon
\LinesNumbered
For arm $i \in [K]$, set $B_i = \ceil{(1 + f_i) \cdot \frac{3 \ln T}{\fst}}$. \\
Initialize $t = 1$, $\phi = 1$ and phase $s=1$.
Maintain active set $A$; start with $A := [K]$. 
\\
\While{$t \le T$}{
    Start phase $s$. \\
    \For{arm $j \in A$}{
    \For{$2^{2s+1} \cdot \ln T$ iterations}{
    \For{$\min(B_j, T-t)$ iterations}{
    Pull $i_t = j$ and let $t \gets t+1$. \\
    \lIf{$t = T$}{\Return.}
    }
    \lIf{$U_{j, \phi} = 1$}{observe $\ell^{\AAE}_{j, \phi} := \ell'_{j, \phi}$ and let $\phi \gets \phi + 1$.}
    \lElse{observe $\ell^{\AAE}_{j, \phi} := 1$ and let $\phi \gets \phi + 1$.}
   
    }
    Let $\psi_s(j) := \{\phi - 8 \cdot 2^{2s} \cdot \ln T, \dots, \phi\}$ be the set of $\phi$ timesteps in which arm $j$ was pulled for phase $s$. \\
    Compute empirical mean $\mu_s(j) = \frac{1}{8 \cdot 2^{2s}\cdot \ln T} \sum_{\phi \in \psi_s(j)}\ell^\AAE_{j, \phi}$.\\
    Set $\text{LCB}_s(i) = \mu_s(i) - 2^{-s}$ and 
    $\text{UCB}_s(i) = \mu_s(i) + 2^{-s}$.
    }
    For any arm $i \in A$ where $\exists j \in A$ such that $\text{LCB}_s(j) > \UCB_s(i)$, remove $i$ from $A$.\\
    Increment $s \gets s + 1$. 
} 
\end{algorithm2e}

For convenience, we restate Theorem \ref{thm:feedbackmonotonicitybbdaaae} below. 

\BBDAAAEAPC*
The proof of \cref{thm:feedbackmonotonicitybbdaaae} follows from adjusting the ideas in the proof of \cref{thm:feedbackmonotonicitybbpullaae} and \cref{thm:feedbackmonotonicitybbda}.
The high-level intuition is that for a sufficiently large $T$, we must reach a phase in both instances where $i$ is eliminated, if $i$ is a suboptimal arm. If $\ALG$ takes the same number of phases $s^*$ to eliminate $i$ in both instances (which a similar coupling argument as above will ensure), then by definition of the block sizes in each algorithm, $\APC_i(\widetilde{\calI}) = s^* \cdot \frac{3\ln T}{\fst} \cdot (1 + \widetilde{f_i})$ and $\APC_i(\calI) = s^* \cdot \frac{3\ln T}{\fst} \cdot (1 + f_i)$. 

In these results, we use the following notation. (Items 1 and 2 are analogous notation to in the analysis of $\BBPull(\AAE)$ from \cref{subsubsec:aae-lemmas}, restated below for convenience.)
\begin{enumerate}
    \item Let $S$ be a random variable denoting the maximum value of the variable $s$ reached in Algorithm \ref{algo:BB-da-se-sim} on $\widetilde\calI$. (That is, $S$ denotes the number of phases that Algorithm \ref{algo:BB-da-se-sim} \textit{begins}.) Note that $S \le T$ with probability 1.
    \item  Let $E_\textrm{loss}$ be the ``clean'' event that at each phase $1 \le s \le S - 1$, for every arm $i \in [K]$, it holds that $\text{LCB}_s(i) \le \bell_i \le \text{UCB}_s(i)$. 
    \item Let $\widetilde{A}$ denote the active set on $\widetilde\calI$, and $A$ denote the active set on $\calI$.
    \item Let $E$ be the ``clean'' event that $U_{j,\phi} = \widetilde U_{j,\phi} = 1$ for all $j \in [K], \phi \in [T]$.
\end{enumerate}

Again, we begin by arguing that ``clean events'' occur with high probability.
\begin{lemma}
\label{lemma:cf-bbda}
    Consider  Algorithm \ref{algo:BB-da-se-sim} evaluated on any given instance $\calI = \left\{\calA, \calF, \calL \right\}$.
    Condition on the event $E$ defined above, i.e. that $U_{j,\phi} = \widetilde U_{j,\phi} = 1$ for all $j \in [K], \phi \in [T]$.
    Let the event $E_\textrm{loss}$ be the defined as above.  Then, $\Pr[E_\textrm{loss}] \geq 1 - 2 T^{-3} K$. 
\end{lemma}
\begin{lemma}
\label{lem:se-pulls-ub-bbda}
Consider  Algorithm \ref{algo:BB-da-se-sim} evaluated on any given instance $\calI = \left\{\calA, \calF, \calL \right\}$  with time horizon $T$. Suppose that the events $E_\textrm{loss}$ and $E$ both hold.
Then, the optimal arm $\ist = \arg\min_j \bell_j$ is never removed from $A$. Moreover, at every phase $1 \le s \le S-1$, if $i \in A$ at the end of phase $s$ (i.e. after 13 in Algorithm \ref{algo:BB-pull-simulated-aae}), then
\[\bell_i - \min_j \bell_j \le 4 \cdot 2^{-s}. \]
\end{lemma}
The proof of \cref{lemma:cf-bbda} is identical to the proof of \cref{lemma:cf}, and the proof of \cref{lem:se-pulls-ub-bbda} is identical to the proof of \cref{lem:se-pulls-ub}, because the analysis is specific to $\AAE$, rather than the black-box transformations; we accordingly omit them here.

\begin{proof}
[Proof of \cref{thm:feedbackmonotonicitybbdaaae} (Monotonicity of $\BBDivideAdjusted(\AAE)$)]
As before, we condition on a series of clean events, construct a coupling, then analyze the phase at which arm $i$ must be eliminated. 

\textbf{Step 1: Condition on feedback observations.} 
This step is identical to Step 1 in the proof of Theorem \ref{thm:feedbackmonotonicitybbda}. Let $E$ be the event that $U_{j,\phi} = \widetilde U_{j,\phi} = 1$ for all $j \in [K], \phi \in [T]$. By \cref{lemma:clean}, $\Pr[E] \geq 1 - 1/T^2$. Then, for any $\phi > T$, we let $U_{j,\phi}$ and $\widetilde U_{j,\phi}$ take on arbitrary values in $\{0,1\}$.
We condition on $E$ for the following steps.

\textbf{Step 2: Couple arms pulled by $\ALG$ across instances $\calI$ and $\widetilde\calI$.} We couple arms in the same way as in the proof of \cref{thm:feedbackmonotonicitybbdivide}. We can apply \cref{lem:joint-losses-identical-bbda}, letting $\Phi^* = T$, so that $(i^\ALG_1, \dots, i^\ALG_{\Phi^*})$ and $(\tilde i^\ALG_1, \dots, \tilde i^\ALG_{\Phi^*})$ are identically distributed. This allows us to consider a joint probability distribution over \\ $(i^\ALG_1, \dots, i^\ALG_{\Phi^*}, \tilde i^\ALG_1, \dots, \tilde i^\ALG_{\Phi^*})$ such that 
$i^\ALG_\phi = \tilde{i}^\ALG_\phi$ for all $\phi \in [\Phi^*]$.

\textbf{Step 3: Condition on $E_{\text{loss}}$.} 
This step is similar to Step 2 in the proof of \Cref{thm:feedbackmonotonicitybbpullaae}. We will condition on $E_{\text{loss}}$, i.e. that confidence bounds are correct: $\text{LCB}_s(i) \leq \bell_i \leq\UCB_s(i)$, for every phase $s$ and every arm $i \in [K]$.
By \cref{lemma:cf-bbda}, we have that $\Pr[E_{\text{loss}}] \geq 1 - 2T^3K$. 

\textbf{Step 4: Run \cref{algo:BB-da-se-sim} and analyze $\APC$. } This step is similar to Step 3 in the proof of \Cref{thm:feedbackmonotonicitybbpullaae}. However the interpretation of the number of rounds $\phi$ in which any arm is selected by $\AAE$ is different. While in \cref{algo:BB-pull-simulated-aae}, the number of rounds $\phi$ was equivalent to the number of feedback observations for that arm, the number of rounds $\phi$ specifies the \textit{number of blocks} in which that arm is pulled by \cref{algo:BB-da-se-sim}. For each arm $i$, within each block where the arm $i$ is selected, the arm will be pulled exactly $\ceil{(1 + f_i)\cdot \frac{3\ln T}{\fst}}$ times on $\calI$, and exactly $\ceil{(1 + \widetilde f_i)\cdot \frac{3\ln T}{\fst}}$ times on $\widetilde \calI$, by the definition of \cref{algo:BB-da-se-sim}. 

We first claim that arm $i$ will be eliminated before the end of the time horizon is reached on both instances. Because we have conditioned on $E_{\text{loss}}$, we can apply \cref{lem:se-pulls-ub-bbda} to argue that from phase $s \geq s' := 3-\log(\Delta_i)$ onwards, arm $i$ is guaranteed not to be pulled. Furthermore, to our coupling, in each phase $s$, arm $i$ is in the active set $A$ for \cref{algo:BB-da-se-sim} on $\calI$ if and only if it is also in the active set $\widetilde A$ for \cref{algo:BB-da-se-sim} on $\widetilde \calI$. To show that arm $i$ is eliminated, we next count the number of times that an arm is pulled in a given phase $s$. For any phase $s$ on $\calI$, the total number of ($t$-indexed) rounds \textit{within} that phase is at most $2 \cdot K \cdot \frac{3\ln T}{\fst} \cdot 2^{2s+1} \ln T = 2^{2s + 2} \cdot \frac{3 K (\ln T)^2}{\fst}$. (To see this, note that $A$ contains at most $K$ arms, each of which have a block size of at most  multiplied by the maximum block size per arm of $2 \cdot \frac{3\ln T}{\fst}$, multiplied by $2^{2s+3} \ln T$ pulls per arm per block within phase $s$.)
Let $t_s$ be the total number of rounds elapsed by the end of phase $s$. Because each previous phase takes 1/4 as many $t$-indexed rounds as the current phase, we can see that for any $s$, \[t_s \leq \frac{4}{3}\cdot 2^{2s + 4} \cdot \frac{3 K(\ln T)^2}{\fst} = 2^{2s} \cdot \frac{ K (\ln T)^2}{\fst}.\] Then, for $s' = 3 - \log(\Delta_i)$, we have 
\[
t_{s'} \leq 2^{2(3 - \log(\Delta_i))} \cdot \frac{K (\ln T)^2}{\fst} = \frac{64}{\Delta_i^2}\cdot \frac{ K (\ln T)^2}{\fst}.
\]
We will have $t_{s'} \leq T$ as long as $\Delta_i \geq \frac{8\sqrt{K}\ln T}{\sqrt{T\fst}}$ (which holds for any $\fst$ and $\Delta_i$ as $T \to \infty$). Note that due to our coupling, this analysis holds for $\widetilde \calI$ as well. 
Altogether, this proves the number of blocks in which arm $i$ will be eliminated before the end of the time horizon on both instances. 

We are now ready to analyze $\APC$ on $\calI$ and $\widetilde{\calI}$. To formalize the rest of our analysis, we introduce the following additional notation. Let $\APC_i^{[T]}(\calI)$ be the number of times arm $i$ is pulled in timesteps on $\calI$, and let $\APC_i^{[T]}(\widetilde \calI)$ be the number of times arm $i$ is pulled in timesteps on $\widetilde \calI$. Since we have conditioned on $U_{j,\phi}$, $\widetilde{U}_{j,\phi}$, $i^\ALG_\phi$, $\tilde i^\ALG_\phi$, we see that at this point $\APC_i^{[T]}(\calI)$ and $\APC_i^{[T]}(\widetilde \calI)$ are both deterministic. 

We show that $\APC^{[T]}_i(\widetilde{\mathcal{I}}) > \APC^{[T]}_i(\mathcal{I})$. We observe that the number of rounds $\phi$ in which arm $i$ is pulled on $\calI$ is equal to the number of rounds $\phi$ in which arm $i$ is pulled on $\widetilde \calI$ (this is using the fact that arm $i$ will be eliminated before the end of the time horizon on both instances and using the property of the coupling that $\tilde{i}_\phi^{\ALG} = i_\phi^{\ALG}$ for all rounds $\phi$).  
Using the equality in the number of blocks in which arm $i$ is pulled on each instance, we see that:
\[\frac{\APC^{[T]}_i(\widetilde{\mathcal{I}})}{\APC^{[T]}_i(\mathcal{I})} = \frac{\tilde{B}_i}{B_i} =  \frac{\ceil{(1 + \tilde{f}_i) \cdot \frac{3 \ln T}{f^*}}}{\ceil{(1 + f_i) \cdot \frac{3 \ln T}{f^*}}}, \]
which is strictly greater than $1$ as long as $T$ is sufficiently large. This implies that $\APC^{[T]}_i(\widetilde{\mathcal{I}}) > \APC^{[T]}_i(\mathcal{I})$ as desired.

We can apply the law of total expectation over the sequences $U_{j,\phi}$, $\widetilde{U}_{j,\phi}$, $i^\ALG_\phi$, $\tilde i^\ALG_\phi$. Let $\APC_i(\calI \mid E, E_{\text{loss}})$ notate the metric $\APC_i$ on instance $\calI$ conditioned on the clean events $E$ and $E_{\text{loss}}$. We see that:
\begin{align*}
    \APC_i(\widetilde \calI \mid E, E_{\text{loss}}) - \APC_i(\calI \mid E, E_{\text{loss}})  = \E\left[\APC_i^{[T]}(\widetilde\calI) - \APC_i^{[T]}(\calI) \mid E, E_{\text{loss}} \right] > 0.
\end{align*}
This means that:
\begin{align*}
    \APC_i(\widetilde \calI \mid E, E_{\text{loss}}) > \APC_i(\calI \mid E, E_{\text{loss}}).
\end{align*}

\textbf{Step 4: Handle the $\APC$ contributions of the conditioning steps.} Finally, we handle the possibility that the events $E$ and $E_{\text{loss}}$ do not hold, i.e., that we do not see feedback in every block on each instance, and the possibility that our confidence bounds are not good. 
Because we first conditioned on $E$ and then conditioned on $E_{\text{loss}}$, we will remove the conditioning in the reverse order: \begin{align*}
    \left|\APC_i(\widetilde \calI | E) -  \frac{1 + \widetilde f_i}{1 + f_i} \cdot \APC_i(\calI | E) \right| &\leq \frac{1}{T^2}
    \\\implies 
    \left|\APC_i(\widetilde \calI) -  \frac{1 + \widetilde f_i}{1 + f_i} \cdot \APC_i(\calI) \right| &\leq \frac{1}{T}
\end{align*}
Combining the above result with Lemma \ref{prop:relationship} implies that $\FOC$ must be strictly increasing in $f_i$.
\qedhere

\end{proof}

\section{Supplemental Materials for \cref{subsec:exp3}}
\label{appendix:exp3}

\subsection{Linear Regret of Standard EXP3}
\label{subsec:exp3-linear}

We first illustrate how the standard EXP3 algorithm may achieve linear regret in the probabilistic feedback setting.

\begin{proposition}[Regret of Standard EXP3]
\label{prop:linearregret}
Standard EXP3 obtains regret $\Omega(T)$ when arms have $f_i \neq 1, \forall i \in [K]$.
\end{proposition}
\begin{proof}
We work with utilities here instead of losses because the intuition is clearer. To obtain the result for losses, one can use the standard transformation that loss $=1-$ utilities. 

Consider two instances:

\textit{Instance 1.} Let there be two arms. Arm 1 has reward distribution $u_1$ with expectation $\E[u_1] = 1$ and $f_1 = 1/4$. Arm 2 has reward distribution $u_2$ with expectation $\E[u_2] = 1/2$ and $f_2 = 1$. 

\textit{Instance 2.} Let there be two arms. Arm $1'$ has reward distribution such that with probability $3/4$, $u_{1'}= 0$, and with probability $1/4$, $u_{1'} \sim u_1$ (that is, sample the deterministic value 0 with probability 3/4, and sample the reward distribution $u_1$ with probability 1/4). Arm 2 has the reward distribution $u_{2'}$. Let $f_{1'} = f_{2'} = 1$.

Fix an infinite tape of independent draws from $u_1$ (call it $p_{u_1}$) and fix an infinite tape of draws from $u_2$ (call it $p_{u_2}$). Fix an infinite tape of draws from a Bernoulli distribution with rate $1/4$ (call it $p_{f_1}$) and for all $\pi \in [0,1]$, fix an infinite tape $p_{\mathcal{B}(\pi)}$ of random draws from a Bernoulli distribution with rate $\pi$.

We will define the trajectory of EXP3 run on either instance according to these sequences, and show that fixing this sequence of draws, EXP3 must pull the same arms and maintain the same values of $w_{i,t}$ and $\pi_{i,t}$ across all rounds for corresponding arms across the instances. We call the algorithm running on the respective instances World and World'.

\textsc{Base case:} Let $t = 0$. Then, both algorithms have initialized the weights to 1, so trivially, the weights are the same. Moreover, by these weights, $\pi_1 = \pi_{1'} = \pi_{2} = \pi_{2'} = 1/2$. Then, we use the first bit in the tape $p_{\mathcal{B}(1/2)}$ to determine which arm to pull in $t=1$. If this value is 0, then pull arm 1 in World and World'; else, pull arm 2 in both Worlds.

\textsc{Inductive case:} Suppose the algorithms have pulled the exact same arms up to time $t-1$, and have maintained the same weights and probabilities so that $w_{1,t} = w_{1',t}$ and $w_{2,t} = w_{2',t}$ and $\pi_{1,t} = \pi_{1',t}$ and $\pi_{2,t} = \pi_{2',t}$. Now, use the next unused bit in the tape $p_{\mathcal{B}(\pi_{1,t})}$ to determine which arm to pull. If this value is 0, pull arm 1 in both Instance 1 and Instance 2; else, pull arm 2 in both Instances. This realizes the correct probabilities in both instances. Now, if arm 2 is pulled, let the reward be the next available draw from the tape $p_{u_2}$; this realizes the correct reward distribution in both instances. If arm 1 is pulled, first take the next unused bit in $p_{f_1}$. If it is 0, in both worlds set the observed utility to 0, making the estimator of the utility $\hat{u}{1,t} = \hat{u}{1',t} = 0$. If the bit is 1, then draw the observed utility by taking the next unused bit from $p_{u_1}$, and use it to compute the utility estimator in both worlds, so that again, $\hat{u}{1,t} = \hat{u}{1',t}$. This realizes the correct distribution of the estimator in both worlds.

Because the estimators are equal, and by the induction hypothesis $w_{1,t} = w_{1',t}$ and $w_{2,t} = w_{2',t}$ and $\pi_{1,t} = \pi_{1',t}$ and $\pi_{2,t} = \pi_{2',t}$, we have that $w_{1,t+1} = w_{1',t+1}$ and $w_{2,t+1} = w_{2',t+1}$ and $\pi_{1,t+1} = \pi_{1',t+1}$ and $\pi_{2,t+1} = \pi_{2',t+1}$.

EXP3 is guaranteed to get sublinear regret when the $f_i$'s are uniformly 1; thus, it must get sublinear regret in instance 2, and thus must pull arm 1 a subconstant number of times. It directly follows that EXP3 must then also pull arm 1 a sublinear (in T) number of times in instance 2, meaning that it pulls arm 2 (the suboptimal arm in instance 2) a linear (in $T$) number of times, thus incurring linear regret in instance 1.  
\end{proof}

\subsection{Regret of 3-Phase EXP3 (Algorithm \ref{algo:EXP3-3phase})}\label{appendix:regret3phaseexp3}

We first discuss the regret bound provided in \ref{thm:EXP3-3phase-regret} and its implications in the context of related work; we prove this result in the remainder of the section. 

\subsubsection{Lemmas for Proof of Theorem \ref{thm:EXP3-3phase-regret}}\label{appendix:proofregretexp3}

We first provide several useful lemmas for formalizing the proof of \cref{thm:EXP3-3phase-regret}.

We start by proving that the estimates built on Phase 2 of the algorithm are close to the true $f_i$'s. As is customary, we prove our results for the pseudo-regret, which coincides with the expected regret for the case of an oblivious adversary.

\begin{lemma}
\label{lem:expectation}
For all $i \in [K]$, the estimate $P^E_i$ obtained in Phase 2 of \Cref{algo:EXP3-3phase} satisfies $\mathbb{E}[P^E_i] = 1/f_i$ and $\mathbb{E}[(P^E_i)^2] \leq 2/f_i^2$. 
\end{lemma}
\begin{proof}
To see that $\mathbb{E}[P^E_i] = 1/f_i$, note that $P^E_{i}$ is distributed as a geometric distribution with parameter $f_i$. To see that $\mathbb{E}[(P^E_i)^2] \le 2/f_i^2$, note that
\[\mathbb{E}\left[\left(P^E_i\right)^2\right] = \mathbb{E}\left[P^E_i\right]^2 + \Var\left(P^E_i\right) \le \frac{1}{f_i^2} + \frac{1}{f_i^2} = \frac{2}{f_i^2}. \]
\end{proof}

\begin{lemma}
\label{lem:1hf-samples}
The estimates $P^{LR}_i$ obtained in Phase 1 of \Cref{algo:EXP3-3phase} satisfy the tail bound:
\[\Pr\left[\forall i \in [K], \frac{1}{2f_i} \le P^{LR}_i \le \frac{2}{f_i}\right] \ge 1 - \frac{2}{T}.\]
\end{lemma}
\begin{proof}
Since we can union over all $i \in [K]$, it suffices to show that the following tail bound for each $i \in [K]$:
\[ \Pr\left[P^{LR}_i > \frac{2}{f_i}\right] \le \frac{1}{TK} \text{ and } \Pr\left[P^{LR}_i< \frac{1}{2f_i}\right] \le \frac{1}{TK}.\]

First, we show the upper tail bound. Since $N \cdot P^{LR}_i$ is a random variable counting the number of trials until $N$ observations are made, we can rewrite $\Pr[P^{LR}_i > \frac{2}{f_i}]$ as a tail bound on a binomial random variable. More specifically, note that $\Pr[P^{LR}_i > \frac{2}{f_i}]$ is equal to the probability that less than $N$ observations appear after $\frac{2N}{f_i}$ trials which is equal to $\Pr[Y < N]$, where $Y \sim \text{Bin}(2N/f_i, f_i)$. We can now apply a multiplicative Chernoff bound to obtain a bound on $\Pr[Y_u < N]$. Let $Z_1, \ldots, Z_{2N/f_i}$ be a sequence of Bernoulli random variables with probability $f_i$, then we see that:
\begin{align*}
\Pr\left[P^{LR}_i > \frac{2}{f_i}\right] &= \Pr[Y_u < N] = \Pr\left[\sum_{j=1}^{2N/f_i} Z_j < N \right] \\
&= \Pr\left[\sum_{j=1}^{2N/f_i} Z_j < 0.5 \cdot \mathbb{E}\left[\sum_{j=1}^{2N/f_i} Z_j \right] \right] \le e^{-\frac{2N}{8}} = e^{-N/4},
\end{align*}
by applying a multiplicative Chernoff bound. We thus obtain a tail bound of at most $1 / (TK)$ with our setting of $N = 8 \log (TK)$. 

Next, let's show the lower tail bound. As before, since $N \cdot P^{LR}_i$ is a random variable counting the number of trials until $N$ observations are made, we can rewrite $\Pr[P^{LR}_i < \frac{1}{2f_i}]$ as a tail bound on a binomial random variable. More specifically, note that $\Pr[P^{LR}_i < \frac{1}{2f_i}]$  is equal to the probability that at least $N$ observations appear after $\frac{N}{2f_i} - 1$ trials which is equal to $\Pr[Y_l \ge N]$, where $Y' \sim \text{Bin}(0.5N/f_i - 1, f_i)$. We can now apply a multiplicative Chernoff bound to obtain a bound on $\Pr[Y' \ge N]$. Let $Z_1, \ldots, Z_{0.5N/f_i - 1}$ be a sequence of Bernoulli random variables with probability $f_i$, then we see that:
\begin{align*}
\Pr\left[P^{LR}_i < \frac{0.5}{f_i}\right] &= \Pr[Y_l \ge N] = \Pr\left[\sum_{j=1}^{0.5N/f_i - 1} Z_j \ge N \right] \\ 
&= \Pr\left[\sum_{j=1}^{0.5N/f_i - 1} Z_j > 2 \cdot \mathbb{E}\left[\sum_{j=1}^{0.5N/f_i - 1} Z_j \right] \right] \le e^{-\frac{(0.5N - f_i)}{3}} = e^{-N/8},
\end{align*}
by applying a multiplicative Chernoff bound. Our setting of $N = 8 \log(TK)$ thus ensures a tail bound of at most $1/(TK)$. 

\end{proof}

Next, we analyze the regret of Phase 3 conditional on the event that the estimates are close to the $f_i$'s. 

\begin{lemma}\label{lem:phase3-conditional}
Conditional on the estimates $\{P^{LR}_i\}_{i \in [K]}$ being close to $\{1/f_i\}_{i \in [K]}$ as in \Cref{lem:1hf-samples}, the regret incurred in Phase 3 is \[ \sqrt{\frac{2 T \log K}{\sum_{i \in [K]} \frac{1}{f_i}}}\]
\end{lemma}

The proof of Lemma \ref{lem:phase3-conditional} builds on the standard analysis of the loss estimator that EXP3 maintains (e.g. \cite{hazan2016introduction}), which we state and reprove here for completeness.
\begin{lemma}[EXP3 bound on estimated rewards]\label{lem:EXP3-estimated}
Let $\ist = \arg \min_{i \in [K]} \sum_{t \in [T]} \ell_{i,t}$ be the optimal arm in hindsight. Then, for the loss estimator $\widehat{\ell}_{i,t}$ that EXP3 maintains it holds that:
\[
- \sum_{t \in [T]} \widehat{\ell}_{\ist, t} \leq - \sum_{t \in [T]} \sum_{i \in [K]} \pi_{i,t} \cdot \widehat{\ell}_{i,t} + \eta \sum_{t \in [T]} \sum_{i \in [K]} \pi_{i,t} \cdot \widehat{\ell}_{i,t}^2 + \frac{\log K}{\eta}
\]
\end{lemma}
\begin{proof}[Proof of Lemma \ref{lem:EXP3-estimated}]
Let $W_t = \sum_{i \in [K]} w_{i,t}$ be the sum of weights of all arms for round $t$. This serves as our potential function. Our goal is to upper and lower bound quantity $W_T$.
For the lower bound:
\begin{align*}
    W_T &= \sum_{i \in [K]} w_{i, T} &\tag{by definition}
    \\&\geq w_{\ist, T}  &\tag{$w_{i, t} \geq 0, \forall i, t$}
    \\&= \exp
    \left(-\eta \left(\sum_{t \in [T]}\widehat{\ell}_{\ist, t}\right) \right)
    \numberthis{\label{eq:Q-lower-bound}}
\end{align*}
For the upper bound:
\begin{align*}
    W_{T} &= 
    W_{T-1}\sum_{i \in [K]} \pi_{i,t} \cdot \exp \left( -\eta \widehat{\ell}_{i,t} \right) &\tag{by definition of the update step}
    \\
    &\leq W_{T-1} \sum_{i \in [K]} \pi_{i,t} \cdot \left( 1 - \eta \widehat{\ell}_{i,t} + \eta^2 \widehat{\ell}_{i,t}^2\right) &\tag{$e^{-x} \leq 1 - x + x^2$ for $x \geq 0$} \\
    &= W_{T-1} \left( 1 - \eta \sum_{i \in [K]} \pi_{i,t} \cdot \widehat{\ell}_{i,t} + \eta^2 \sum_{i \in [K]} \pi_{i,t} \cdot \widehat{\ell}_{i,t}^2 \right)&\tag{$\sum_{i \in [K]} \pi_{i,t} = 1$} \\
    &\leq W_{T-1} \exp \left(- \eta \sum_{i \in [K]} \pi_{i,t} \cdot \widehat{\ell}_{i,t} + \eta^2 \sum_{i \in [K]} \pi_{i,t} \cdot \widehat{\ell}_{i,t}^2 \right) &\tag{$1 + x \leq e^x$ for all $x$} \\ 
    &= W_0 \exp \left(- \eta \sum_{t \in [T]} \sum_{i \in [K]} \pi_{i,t} \cdot \widehat{\ell}_{i,t} + \eta^2 \sum_{t \in [T]}\sum_{i \in [K]} \pi_{i,t} \cdot \widehat{\ell}_{i,t}^2 \right) &\tag{telescoping for $W_{t}, \text{for } t \in [T-1]$}\\
    &=K\exp \left(- \eta \sum_{t \in [T]} \sum_{i \in [K]} \pi_{i,t} \cdot \widehat{\ell}_{i,t} + \eta^2 \sum_{t \in [T]}\sum_{i \in [K]} \pi_{i,t} \cdot \widehat{\ell}_{i,t}^2 \right)  &
    \numberthis{\label{eq:Q-upper-bound}}
\end{align*}
where the last inequality comes from the fact that $W_0 = K$. Combining Equations~\eqref{eq:Q-lower-bound} and \eqref{eq:Q-upper-bound}, taking the log on both sides, then dividing both sides by $\eta$ we get the result.
\end{proof}

Now we prove Lemma \ref{lem:phase3-conditional}.
\begin{proof}[Proof of Lemma \ref{lem:phase3-conditional}]
    We first analyze the first and the second moments of the estimator $\hu_{i,t}$. Let $H_{Ph1}$ encompass the randomness of Phase 1; $H_{Ph2}$ encompass the randomness of Phase 2; $H_{t-1}$ encompass the randomness of the algorithm in Phase 3 up to time $t-1$; and $H_{Alg}$ encompass the randomness of the algorithm at time $t$.  

    For the first moment, we have: 
    \begin{align*}
        \E \left[ \widehat{\ell}_{i,t} \mid H_{Ph1} \right] &= \E_{H_{Ph2}} \left[ \E_{H_{t-1}} \left[ \E_{H_{Alg}} \left[ \widehat{\ell}_{i,t} | H_{t-1}, H_{Ph2}, H_{Ph1}\right] | H_{Ph2}, H_{Ph1}\right] \mid H_{Ph1} \right] \\
        &=_{(A)} \E_{H_{Ph2}} \left[\ell_{i,t} \cdot f_i \cdot P^E_i | H_{Ph1} \right] \\
        &= \ell_{i,t} \cdot f_i \cdot \E_{H_{Ph2}} \left[P^E_i | H_{Ph1} \right] \\
        &=_{(B)} \ell_{i,t} \numberthis{\label{eq:first-mom}},
    \end{align*}
    where (A) follows from the fact that $\E_{H_{t-1}} \left[ \E_{H_{Alg}} \left[ \widehat{\ell}_{i,t} | H_{t-1}, H_{Ph2}, H_{Ph1}\right] | H_{Ph2}, H_{Ph1}\right] = \ell_{i,t} f_i P^E_i$ and (B) follows from \Cref{lem:expectation}. 

    For the second moment, we have: 
    \begin{align*}
        \E \left[ \pi_{i,t} \widehat{\ell}_{i,t}^2 \mid H_{Ph1} \right] &= \E_{H_{Ph2}} \left[ \E_{H_{t-1}} \left[ \E_{H_{Alg}} \left[\pi_{i,t} \widehat{\ell}_{i,t}^2 | H_{t-1}, H_{Ph2}, H_{Ph1}\right] \mid H_{Ph2}, H_{Ph1} \right]   \Big| H_{Ph1} \right] \\
        &= \E_{H_{Ph2}} \left[ \E_{H_{t-1}} \left[ \pi_{i,t} f_i \cdot \frac{\ell_{i,t}^2}{\pi_{i,t}} \cdot (P^E_i)^2 \mid H_{Ph2}, H_{Ph1} \right]   \Big| H_{Ph1} \right] \\
        &= \E_{H_{Ph2}} \left[ f_i \cdot \ell_{i,t}^2 \cdot {(P^E_i)^2}  \Big| H_{Ph1} \right] \\
        &= f_i \cdot \ell_{i,t}^2 \cdot \E_{H_{Ph2}} \left[ (P^E_i)^2  \Big | H_{Ph1} \right] \\
         &= f_i \cdot \ell_{i,t}^2\cdot \E_{H_{Ph2}} \left[ (P^E_i)^2 \right] \\
        &\leq \frac{2 \ell_{i,t}^2}{f_i} \numberthis{\label{eq:sec-mom}},
    \end{align*}
    where the last inequality is due to the fact that $\E [ (P^E_i)^2] \leq 2/f_i^2$  (see \Cref{lem:expectation}). 

    Taking expectations on both sides of \Cref{lem:EXP3-estimated} and substituting \Cref{eq:first-mom} and \Cref{eq:sec-mom} we get: 
    \begin{align*}
       \mathbb{E}\left[\sum_{t \in [T]} \ell_{\ist,t} - \sum_{t \in [T]} \ell_{i_t, t} \Big| H_{Ph1}\right] &\leq \mathbb{E}\left[\eta \sum_{t \in [T]} \sum_{i \in [K]} \frac{2 \ell_{i,t}^2}{f_i} + \frac{\log K}{\eta} \Big| H_{Ph1}\right] \\
        &\leq_{(A)} \mathbb{E}\left[2 \eta T \sum_{i \in [K]} \frac{1}{f_i} + \frac{\log K}{\eta} \Big| H_{Ph1}\right] \\
        &=  \mathbb{E}\left[2\sqrt{\frac{\log K}{T \sum_{i \in [K]} P^{LR}_i}} T \sum_{i \in [K]} \frac{1}{f_i}  + \frac{\log K}{\sqrt{\frac{\log K}{T \sum_{i \in [K]} P^{LR}_i}}} \Big| H_{Ph1}\right] \\
        &\le_{(B)} 2 \sqrt{\frac{\log K}{T \sum_{i \in [K]} 0.5 (1/f_i)}} T \sum_{i \in [K]} \frac{1}{f_i}  + \sqrt{ T \left(\sum_{i \in [K]} \frac{2}{f_i}\right)\log(K)}\\
        &\le 2\sqrt{2 T \left(\sum_{i \in [K]} \frac{1}{f_i}\right)\log(K)}  + \sqrt{2 T \left(\sum_{i \in [K]} \frac{1}{f_i}\right)\log(K)}\\
        &= 4 \sqrt{2}\sqrt{T \log(K) \sum_{i \in [K]}\frac{1}{f_i}},
    \end{align*}
where (A) follows from the fact that $\ell_{i,t} \le 1$ and (B) follows from the fact that we conditioned on \Cref{lem:1hf-samples}.
\end{proof}

\subsubsection{Proof of \cref{thm:EXP3-3phase-regret}}

We are now ready to prove \Cref{thm:EXP3-3phase-regret}.
\begin{proof}[Proof of \Cref{thm:EXP3-3phase-regret}]
The regret \Cref{algo:EXP3-3phase} can be decomposed to the regret of the three phases of the algorithm:
\begin{align*}
    R(T) &= R_{\Phase 1}(T) + R_{\Phase 2}(T) + R_{\Phase 3}(T) \\ 
         &\leq \E \left[(N+1) \sum_{i \in [K]} P_i \right] + R_{\Phase 3}(T) &\tag{expected number of rounds to obtain feedback} \\
         &= \sum_{i \in [K]} \frac{8 \log(TK)}{f_i} +  R_{\Phase 3}(T).
\end{align*}

We next analyze term $R_{\Phase 3}(T)$. From \Cref{lem:1hf-samples}, with probability at least $1-\delta$ the estimates $1/P^{LR}_i$ are close to $f_i$. But with probability at most $\delta$, the estimates are far away and the regret that we pick up in these rounds is at most $1$. Putting everything together, we have: 
\begin{align*}
    R_{\Phase 3}(T) &\leq 4 \sqrt{2} (1 - \delta) \sqrt{T \log(K) \sum_{i \in [K]}\frac{1}{f_i}} + \delta T \\
    &\leq 4\sqrt{2} \sqrt{T \log(K) \sum_{i \in [K]}\frac{1}{f_i}}  + 1 &\tag{\Cref{lem:phase3-conditional}}
\end{align*}

This proves a regret bound of:
\[\calO\left(\sqrt{T \log(K) \sum_{i \in [K]}\frac{1}{f_i}}  + \sum_{i \in [K]} \frac{\log (T)}{f_i} + \sum_{i \in [K]} \frac{\log (K)}{f_i}\right).\]
In general $T \geq K$, so in order to derive the bound in the theorem statement, we need to argue that $\sum_{i \in [K]} \log T / f_i$ is order smaller than $\sqrt{T \sum_{i \in [K]} 1/ f_i}$. Note that this is the case when $T \geq \sqrt{\sum_{i \in [K]} 1/f_i}$, which is true for large enough time horizons.
\end{proof}

\end{document}